\documentclass{ws-m3as}

\usepackage{pifont}
\usepackage{amsmath}
\usepackage{amssymb}
\usepackage{setspace}
\usepackage{enumerate}
\usepackage[unicode=true]
 {hyperref}

\makeatletter

\renewcommand{\P}{\mathbb{P}}
\newcommand{\p}{\partial}

\newcommand{\dd}{\mathrm{d}}
\DeclareMathOperator{\R}{{\mathbb R}}

\newtheorem{prop}{Proposition}[section]  
\newtheorem{defn}[prop]{Definition}
\newtheorem{thm}[prop]{Theorem}
\newtheorem{cor}[prop]{Corollary}
\newtheorem{lem}[prop]{Lemma}
  
\usepackage{hyperref}
\usepackage{dsfont}
\usepackage{ifpdf} 
\usepackage{color}		
\ifpdf 
\date{}
 \IfFileExists{lmodern.sty}{\usepackage{lmodern}}{}

\fi 

\usepackage[figure]{hypcap}

\usepackage[rightcaption]{sidecap}

\newcommand{\TabBesBeg}[1][1.0]{%
 \let\MyTable\table
 \let\MyEndtable\endtable
 \renewenvironment{table}[1]{\begin{SCtable}[#1]##1}{\end{SCtable}}}

\newcommand{\TabBesEnd}{%
 \let\table\MyTable
 \let\endtable\MyEndtable}

\newcommand{\FigBesBeg}[1][1.0]{%
 \let\MyFigure\figure
 \let\MyEndfigure\endfigure
 \renewenvironment{figure}[1]{\begin{SCfigure}[#1]##1}{\end{SCfigure}}}

\newcommand{\FigBesEnd}{%
 \let\figure\MyFigure
 \let\endfigure\MyEndfigure}

\AtBeginDocument{
  \def\labelitemi{\Pisymbol{psy}{183}}
}

\makeatother
\newcommand{\f}{\frac}

\newcommand{\proba}[1]{\ensuremath{\mathbb{P}\left(#1\right)}}
\newcommand{\esp}[1]{\ensuremath{\mathbb{E}\left[#1\right]}}
\newcommand{\1}{\ensuremath{\mathds{1}}}

\renewcommand{\R}{\ensuremath{\mathbb{R}}}
\newcommand{\N}{\ensuremath{\mathbb{N}}}
\newcommand{\Z}{\ensuremath{\mathbb{Z}}}

\definecolor{orange}{cmyk}{0,0.5,1,0.3}

\definecolor{turquoise}{rgb}{0,0.4,0.4}

\begin{document}


\title{{Microscopic approach of a time elapsed neural model}}

\author{Julien Chevallier}

\address{Laboratoire J. A. Dieudonn\'e, UMR CNRS 6621, Universit\'e de Nice Sophia-Antipolis, Parc Valrose \\
06108 Nice Cedex 2, France\\
julien.chevallier@unice.fr}

\author{Mar\'ia Jos\'e C\'aceres}

\address{Departamento de Matem\'atica Aplicada , Universidad de Granada, Campus de Fuentenueva\\
E-18071 Granada, Spain\\
caceresg@ugr.es}

\author{Marie Doumic}


\address{UPMC University of Paris 6, JL Lions Lab., 4 place Jussieu \\ 75005 Paris, France}

\author{Patricia Reynaud-Bouret}

\address{Laboratoire J. A. Dieudonn\'e, UMR CNRS 6621, Universit\'e de Nice Sophia-Antipolis, Parc Valrose \\
06108 Nice Cedex 2, France\\
patricia.reynaud-bouret@unice.fr}


\maketitle


\begin{abstract}

The spike trains are the main components of the information processing in 
the brain. To model spike trains several point processes have been 
investigated in the literature. And more macroscopic approaches have also been
studied, using partial differential equation models. 
The main aim of the present article is to build a bridge between several point 
processes models (Poisson, Wold, Hawkes) that have been proved to 
statistically fit real spike trains data and age-structured partial
differential equations as introduced by Pakdaman, Perthame and Salort.
\end{abstract}



\section*{Introduction}
In Neuroscience, the action potentials (spikes) are the main
components of the real-time information processing in the
brain. Indeed, thanks to the synaptic integration, the membrane voltage 
of a neuron 
depends on the action potentials emitted by some others, 
whereas if this membrane potential is sufficiently high, there is production 
of action potentials. 

To access those phenomena, schematically, one can proceed in two ways:
extracellularly record in vivo several neurons, at a same time, and 
have access to simultaneous spike trains (only the list of events corresponding 
to action potentials) 
or intracellularly record the whole membrane voltage of only one neuron at a time, 
being blind to the nearby neurons. 

Many people focus on spike trains. Those data are fundamentally random 
and can be modelled easily by time point processes, i.e. random countable 
sets of points on $\R_+$. Several point processes models have been investigated
in the literature, each of them reproducing different features of the neuronal
reality. The easiest model is the homogeneous Poisson process, which can only 
reproduce a constant firing rate for the neuron, but which, in particular, fails 
to reproduce refractory periods\footnote{Biologically, a neuron cannot produce 
two spikes too closely in time.}. It is 
commonly admitted that this model is too poor to be realistic. 
Indeed, in such a model, two points or spikes can be arbitrary close as soon as
their overall frequency is respected in average.  
Another more realistic model is the renewal process \cite{Pipa2013}, where the 
occurrence of a point or spike depends on the previous occurrence. 
More precisely, 
the distribution of delays between spikes (also called inter-spike intervals, ISI) 
is given and a distribution, which provides small weights to small delays,
is able to mimic refractory periods. 
A deeper statistical analysis has shown that Wold processes is showing 
good results, with respect to goodness-of-fit test on real data sets \cite{Pouzat09}. 
Wold processes are point processes for which the next occurrence of a spike depends 
on the previous occurrence but also on the previous ISI.
From another point of view, the fact that spike trains are usually non stationary 
can be easily modelled by inhomogeneous Poisson processes \cite{ventura_kernel}.
 All those models do not reflect one of the main features of spike trains, 
which is the synaptic integration and there has been various attempts to catch 
such phenomenon. One of the main model  is the Hawkes model, which has been 
introduced in \cite{chornoboy} and which has been recently shown to fit several 
stationary data \cite{RBRGTM}. Several studies have been done in similar directions 
(see  for instance \cite{BBS}).  More recently a vast interest has been shown to 
generalized linear models \cite{Pillow}, with which one can infer functional 
connectivity and which are just an exponential variant of Hawkes models.


There has {also} been several models of the full membrane voltage
such as Hodgkin-Huxley models. It is possible to fit some of those probabilistic
stochastic differential 
equations (SDE) 
on real voltage data \cite{JahnDitlevsen} and to use them to estimate meaningful
physiological parameters  \cite{DitSamson}. 
However, the lack of simultaneous data (voltages of different 
neurons  at the same time) prevent these models to be used as statistical 
models that can be fitted on network data, to estimate network parameters. 
A  simple SDE model taking synaptic integration into account is the well-known 
Integrate-and-Fire (IF) model. Several variations 
 have been proposed 
to describe several features of real neural networks such as oscillations 
\cite{B,BH}. 
{In particular, there exists hybrid IF models including inhomogeneous voltage driven Poisson process \cite{JBHD} that are able to mimic real membrane potential data. However}
up to our knowledge and unlike point processes models, no statistical test have been applied to 
show that {any of the previous variations of the IF model} fit real network data.

Both, SDE and point processes, approaches are microscopic descriptions, where random 
noise explains the intrinsic variability. Many authors have argued that there must
be some more macroscopic approach describing huge neural networks as a whole, using PDE 
formalism~\cite{CBGW,sirovich}. Some authors have already been able to perform link 
between PDE approaches as the macroscopic system and SDE approach (in particular IF models)
as the microscopic model~\cite{RBW,omurtag,mg}. 
Another macroscopic point of view on spike trains is proposed by Pakdaman, Perthame and 
Salort in a series of articles~\cite{PPS1,PPS2,PPS3}. It uses a nonlinear age-structured 
equation to describe the spikes density.
Adopting a population view, they aim at studying relaxation to equilibrium or 
spontaneous periodic oscillations. Their model is justified by a qualitative, heuristic 
approach. As many other models, their model shows several qualitative features such as 
oscillations that make it quite plausible for real networks, but once again there is no 
statistical proof of it, up to our knowledge.

In this context, the main purpose of the present article is to build a 
bridge between several point processes models that  have been proved to statistically fit real spike trains data 
and age structured PDE of the type of Pakdaman, Perthame and Salort. 
The point processes are the microscopic models, the PDE being their {meso-macroscopic} 
counterpart. In this sense, it extends PDE approaches for IF models to models 
that statistically fit true spike trains data. 
In the first section, we introduce Pakdaman, Perthame and Salort PDE (PPS) via its 
heuristic informal and microscopic description, which is based on IF models. 
Then, in Section \ref{secPP}, we develop the different point process models, 
quite informally, to draw the main heuristic correspondences between both approaches. 
In particular, we introduce the conditional intensity of a point process and 
a fundamental construction, called Ogata's thinning~\cite{ogatathin}, which allows a microscopic 
understanding of the dynamics of a point  process. 
Thanks to Ogata's thinning, in Section \ref{PDEpoint}, we have been able to 
rigorously derive a microscopic random weak version of (PPS) and to propose its expectation
deterministic counterpart. 
An independent and identically distributed (i.i.d) population version is also available. 
Several examples of applications are discussed 
in Section \ref{sec:ex}. 
To facilitate reading, technical results and proofs are included in 
two appendices.
The present work is clearly just a first to link point processes and PDE: there are much 
more open questions than answered ones and 
 this is discussed in the final conclusion. 
However, we think that this can be fundamental to acquire a deeper understanding 
of spike train models, their advantages as well as their limitations.

\section{Synaptic integration and (PPS) equation}
\label{SecPPS}

Based on the intuition that every neuron in the network
should behave in the same way, Pakdaman, Perthame and Salort  
proposed in \cite{PPS1} 
a deterministic PDE 
denoted (PPS)
in the sequel. The origin of this PDE is the 
classical 
 (IF) model. 
In this section we describe the link between the (IF) microscopic model
and the mesoscopic (PPS) model, 
the main aim being to show  thereafter the relation between (PPS) model and other
natural microscopic models for spike trains: point processes.

\subsection{Integrate-and-fire}

Integrate-and-fire  models describe the 
time evolution of the membrane 
potential, $V(t)$, by means of ordinary differential equations as follows
\begin{equation}\label{Equation-IF}
C_{m}\frac{\dd tV}{\dd t}=-g_L (V-V_L)+I(t),
\end{equation}
where $C_m$ is the capacitance of the membrane, $g_L$ is the leak 
conductance and $V_L$ is the leak reversal potential.
If $V(t)$ exceeds
a certain threshold $\theta$, the neuron fires / emits an action potential (spike) and $V(t)$ is reset to $V_L$.  The {\it synaptic current} $I(t)$ takes into account the fact that other presynaptic neurons fire and excite the neuron of interest, whose potential is given by $V(t)$.

As stated in \cite{PPS1}, the origin of (PPS) equation comes from 
\cite{PPCV}, where the explicit solution of 
 a classical IF model as \eqref{Equation-IF}
has been  discussed.
To be more precise the  membrane voltage of one neuron 
at time $t$ is described by:
\begin{equation}\label{IF}
V(t)=V_r+(V_L-V_r)e^{-(t-T)/\tau_m}+\int_{T}^t h(t-u) N_{input}(du),
\end{equation}
where $V_r$ is the resting potential 
satisfying  $V_L<V_r<\theta$,
$T$ is the last spike emitted by  the considered neuron,
$\tau_m$ is the time constant of the system (normally 
$\tau_m=g_L/C_m$),
$h$ is the excitatory post synaptic potential (EPSP) and 
$N_{input}$ is the sum of Dirac masses at each spike of the presynaptic 
neurons.
Since after firing, $V(t)$ is reset to $V_L<V_r$, there is a refractory period when the neuron is less excitable than at rest.
%
%
The constant time  $\tau_m$ indicates
whether the next spike can occur more or less rapidly.
The other main quantity,  $\int_{T}^t h(t-u) N_{input}(du)$,  is the {\it synaptic integration term}.

In \cite{PPCV}, they consider a whole random network of such IF neurons
 and look at the behavior of this model,
 where the only randomness is in the network. 
In many other studies
\cite{B,BH,CCP,CP,mg,sirovich,omurtag} 
IF models 
as \eqref{Equation-IF} are considered to finally obtain
 other systems of partial differential equations (different
to (PPS))
describing neural networks behavior.
In these studies, each presynaptic neuron is assumed to fire as an independent
 Poisson process and via  a diffusion approximation, the synaptic current is then approximated  by a continuous in 
time stochastic process of Ornstein-Uhlenbeck.


\subsection{The (PPS) equation}

The deterministic PDE  proposed by Pakdaman, Perthame and Salort, 
whose origin is also the mi\-cros\-co\-pic IF model \eqref{IF},
is the following:

$
\left(\mathrm{PPS}\right)
\quad 
\textrm{\begin{tabular}{l}  $\displaystyle \begin{cases}
\frac{\partial n\left(s,t\right)}{\partial t}+
\frac{\partial n\left(s,t\right)}{\partial s}+
p\left(s,X\left(t\right)\right)n\left(s,t\right)=0\\
m\left(t\right):=n\left(0,t\right)=
\int_{0}^{+\infty}p\left(s,X\left(t\right)\right)n\left(s,t\right)\dd s.
\end{cases}$
\end{tabular}}
$

\noindent
In this equation, $n(s,t)$ represents a probability density of neurons 
at time $t$ having discharged at time $t-s$. Therefore, 
$s$ represents the time elapsed since the last discharge. 
The fact that the equation is an elapsed time structured equation is 
natural, because the IF model \eqref{IF} clearly only depends on the time 
since the last spike. 
More informally, the variable $s$ represents the "age" of the neuron. 

The first equation of the system (PPS) 
represents a pure transport process and means that as time goes by, neurons 
of age $s$ and past given by $X(t)$ are either aging linearly or reset to age $0$
with rate $p\left(s,X\left(t\right)\right)$.

The second equation of (PPS) describes the fact
that 
when neurons spike, the 
age
(the elapsed time) returns to 0.
Therefore, $n(0,t)$ depicts the density of neurons undergoing a discharge at 
time $t$ and it is denoted by $m(t)$.
As a consequence of this boundary condition, for $n$ at $s=0$,
 the following conservation law is obtained: 
\begin{eqnarray*}
\int_0^\infty n\left(s,t\right)\dd s
= \int_0^\infty n\left(s,0\right)\dd s
\end{eqnarray*}
This means  that if $n\left(\cdot,0\right)$ is
a probabilistic density then $n\left(\cdot,t\right)$ 
can be interpreted as  a density at each time $t$. 
Denoting by $\dd t$ the Lebesgue measure and since $m(t)$ is the density of firing neurons at time $t$ in (PPS), $m(t) \dd t$ can also be interpreted as the limit of  $N_{input}(\dd t)$ in \eqref{IF} when the population of neurons becomes continuous.

The system (PPS) is nonlinear since  the rate $p\left(s,X(t)\right)$
depends
on 
$n(0,t)$ 
by means of the  quantity  $X(t)$:
\begin{equation}\label{X}
X(t)=\int\limits_0^t h(u) m(t-u)\dd u=\int\limits_0^t h(u) n(0,t-u) \dd u.
\end{equation}
The quantity  $X(t)$ represents the interactions between neurons.
It "takes into account the averaged propagation time for the ionic 
pulse in this network"~\cite{PPS1}. 
More precisely with respect to the IF models \eqref{IF}, 
this is the synaptic integration term, once the population becomes continuous. 
The only difference is that in \eqref{IF} the memory is cancelled once the 
last spike has occurred and this is not the case here.
However informally, both quantities have  the same interpretation.
Note nevertheless, that in \cite{PPS1}, the function $h$ can be much more
 general  than the $h$ of the IF models which clearly corresponds to EPSP. 
From now on and in the rest of the paper, $h$ is just a general non negative 
function without forcing the connection with EPSP.

The larger $p\left(s,X(t)\right)$  the more likely neurons of age $s$ 
and past $X(t)$
fire. Most of the time (but it is not a requisite), $p$ is assumed to be less than 1 and is interpreted as the  probability that neurons of age $s$ fire. However, as 
shown in 
Section \ref{PDEpoint}
and as 
interpreted in many population structured equation \cite{Cloez,DHKR,BP}, $p(s,X(t))$ is closer to a hazard rate, 
i.e. a positive quantity such that $p\left(s,X(t)\right)\dd t$ is informally the probability to fire  given that the neuron has not fired yet. 
In particular, it
could be not bounded by 1 and does not need to integrate to 1. A toy example is obtained if  $p\left(s,X(t)\right)=\lambda>0$, where  a steady state solution is $n(s,t)=\lambda e^{-\lambda s} {\bf 1}_{s\geq 0}$: this is the density of an exponential variable 
with parameter $\lambda$. 

However, based on the interpretation of $p\left(s,X(t)\right)$  as a probability bounded by 1, one of the main model that Pakdaman, Perthame and Salort consider is $p\left(s,X(t)\right)={\bf 1}_{s\geq \sigma(X(t))}$. This again can be easily interpreted by looking at \eqref{IF}. Indeed, since in the IF models the spike happens when the threshold $\theta$ is reached, one can consider that $p\left(s,X(t)\right)$ should be equal to 1 whenever
$$V(t)=V_r+(V_L-V_r)e^{-(t-T)/\tau_m}+ X(t)\geq \theta,$$
and 0 otherwise.
Since $V_L-V_r<0$, $p\left(s,X(t)\right)=1$  is indeed equivalent to $s=t-T$ larger than some decreasing function of $X(t)$. This has the double advantage to give a formula for the refractory period ($\sigma(X(t))$) and to model excitatory systems:  the refractory period decreases when the whole firing rate increases via $X(t)$ and this makes the neurons fire even more. 
This is for this particular case that Pakdaman, Perthame and Salort have shown existence of oscillatory behavior \cite{PPS2}.


Another important parameter in the (PPS) model and introduced in \cite{PPS1} is $J$, which can be seen with our formalism
as $\int h$ and which describes the  network connectivity or the strength of the interaction.
In \cite{PPS1} 
it has been proved that,
for highly or weakly connected 
networks, (PPS) model exhibits relaxation to 
steady state and periodic solutions have also been numerically observed for moderately connected
networks.
The authors in \cite{PPS2} 
have quantified the regime
where relaxation to a stationary solution occurs in terms 
of $J$ and described periodic solution for intermediate values
of $J$.

Recently, in \cite{PPS3}, the (PPS) model has been extended
including a fragmentation term, which describes the
adaptation and fatigue of the neurons. In this sense, this
new term incorporates the past activity of the neurons.
For this new model, in the linear case there is exponential 
convergence to the steady states, while in the
 weakly nonlinear
case  a total desynchronization in the network is proved.
Moreover, for greater nonlinearities, synchronization can again been numerically observed.




\vspace{-0.25cm}
\section{Point processes and conditional intensities as models for spike trains 
\label{secPP}}


We first start by quickly reviewing the main basic 
concepts
and notations of point processes, 
in particular, conditional 
intensities and Ogata's thinning \cite{ogatathin}. We refer the interested reader to \cite{Bre} for 
exhaustiveness 
and to \cite{BBVKF} for a much more condensed version, with the main useful 
notions.
\vspace{-0.3cm}
\subsection{Counting processes and conditional intensities}
We focus on locally
finite point processes on $\mathbb{R}$, equipped with the borelians  
$\mathcal{B}(\R)$. 
\begin{defn}[Locally finite point process]
  A  locally finite point process $N$ on $\mathbb{R}$  is a random set of points such that 
  it has almost surely (a.s.) a finite number of points in finite intervals.
  Therefore, associated to $N$ there is an ordered sequence of extended real 
  valued random times $(T_z)_{z\in\Z}$:   
$\cdots\leq T_{-1}\leq T_{0}\leq 0<T_{1}\leq \cdots$.
\end{defn}
\noindent For a measurable set $A$, $N_A$ denotes the number of points of $N$ in $A$. 
This is a random variable with values in 
{$\N\cup \{\infty\}$}.
\begin{defn}[Counting process associated to a point process]
  The process on $\R_+$ defined by $t\mapsto N_t:=N_{(0,t]}$ is called the 
    counting process associated to the point process $N$. 
\end{defn}

\noindent 
The natural and the predictable filtrations are fundamental for the present work.
\begin{defn}[Natural filtration of a point process]
 The natural filtration of $N$ is the family    
 $\left( \mathcal{F}_t^{N}  \right)_{t\in \mathbb{R} }$ of   $\sigma$-algebras 
 defined by  $\mathcal{F}_t^{N}=\sigma\left(N\cap(-\infty,t]\right)$.
\end{defn}
\begin{defn}[Predictable filtration of a point process]
 The predictable filtration of $N$ is the family of $\sigma$-algebra  
 $\left( \mathcal{F}_{t-}^{N}  \right)_{t\in \mathbb{R} }$ defined
 by $\mathcal{F}_{t-}^{N} = \sigma\left( N\cap(-\infty,t)\right)$.
\end{defn}
\vspace{-0.2cm}
The intuition behind this concept is that $\mathcal{F}_t^{N}$ contains all 
the information given by the point process at time $t$. 
In particular, it contains the information whether $t$ is a point of 
the process or not 
while $\mathcal{F}_{t-}^{N}$ only contains the 
information given by the point process strictly before $t$.
Therefore, it does not contain (in general) the information whether 
$t$ is a point or not. In this sense, 
$\mathcal{F}_{t-}^{N}$  represents (the information contained in) the past.

Under some rather classical conditions 
\cite{Bre}, which are always 
assumed to be satisfied here, one can associate to 
$(N_t)_{t\geq 0}$ a stochastic intensity $\lambda(t,\mathcal{F}_{t-}^{N})$, 
which is a non negative random  quantity. 
The notation $\lambda(t,\mathcal{F}_{t-}^{N})$
for the intensity 
refers to the predictable version of the intensity 
associated to the natural filtration  
and $(N_t-\int_0^t \lambda(u,\mathcal{F}_{u-}^{N})\dd u)_{t\geq 0}$ forms a local 
martingale \cite{Bre}.  
Informally, $\lambda(t,\mathcal{F}_{t-}^{N})\dd t$ represents the probability 
to have a new point in interval $[t,t+\dd t)$ given the past. 
Note that $\lambda(t,\mathcal{F}_{t-}^{N})$ 
should not be understood as a function, in the same way as density is 
for random variables. It is a  "recipe" 
explaining how the probability 
to find a new point at time $t$ depends on the past configuration
: since 
the past configuration 
depends on its own past, this is closer to a 
recursive formula. In this respect, 
the intensity should obviously
depend on $N\cap(-\infty,t)$ and not on $N\cap(-\infty,t]$ to predict the 
occurrence at time $t$, since 
we cannot 
 know whether 
$t$ is already a point or not.

The distribution of the point process $N$ on $\R$ is
completely characterized by the  knowledge of the intensity
$\lambda(t,\mathcal{F}_{t-}^{N})$ on $\R_+$ and the distribution of 
$N_-=N\cap\R_-$, 
which  is denoted by $\mathbb{P}_0$ in the sequel.
The  information about $\mathbb{P}_0$  is necessary since each point of $N$ 
may depend on the occurrence of all the previous points: if for all $t>0$, one knows 
the "recipe" $\lambda(t,\mathcal{F}_{t-}^{N})$ 
that gives the probability of a new point at time $t$
given the past 
configuration, 
one still needs to know the distribution of 
$N_-$ to obtain the whole process.

\noindent Two main assumptions are used 
depending on the type of results
 we seek:

$\left(\mathcal{A}^{\mathbb{L}^1,a.s.}_{\lambda, loc}\right)
\,
\textrm{\begin{tabular}{|l}  for any $T\geq 0$,  $\int_{0}^{T} \lambda(t,\mathcal{F}_{t-}^{N}) \dd t$ is finite a.s. 
\end{tabular}}$

$
\left(\mathcal{A}^{\mathbb{L}^1,exp}_{\lambda,loc}\right)
\,
\textrm{\begin{tabular}{|l}  for any $T\geq 0$,  $\esp{\int_{0}^{T} \lambda(t,\mathcal{F}_{t-}^{N}) \dd t}$ is finite. 
\end{tabular}}$\\
Clearly $\left(\mathcal{A}^{\mathbb{L}^1,exp}_{loc}\right)$ implies $\left(\mathcal{A}^{\mathbb{L}^1,a.s.}_{loc}\right)$. Note that $\left(\mathcal{A}^{\mathbb{L}^1,a.s.}_{loc}\right)$ implies non-explosion in finite time for the counting processes $(N_{t})$.

\begin{defn}[Point measure associated to a point process]
 The point measure associated to $N$ is denoted by $N(\dd t)$ and defined by
 $N(\dd t)=\sum_{i\in \Z} \delta_{T_i}(\dd t)$,
 where $\delta_u$ is the Dirac mass in $u$.
\end{defn}
By analogy with (PPS), and since points of point processes correspond to 
spikes (or times of discharge) for the considered 
neuron in spike train analysis, $N(\dd t)$ is the
microscopic equivalent of  the distribution of discharging neurons $m(t)\dd t$.
Following this analogy, and since $T_{N_t}$ 
is the last point less or equal to $t$ for every $t\geq 0$, the age $S_t$ at time $t$ is defined 
by $S_t=t-T_{N_t}$.
In particular, if $t$ is a point of $N$, then $S_{t}=0$.
Note that $S_t$ is $\mathcal{F}_t^{N}$ measurable for every $t\geq 0$
and therefore, $S_{0}=-T_{0}$ is $\mathcal{F}_{0}^{N}$ measurable.
To define an age at time $t=0$, one assumes that
\smallskip

$
\left(\mathcal{A}_{T_0}\right)
\,
\textrm{\begin{tabular}{|l} 
There exists {a} first point before $0$ for the process $N_-$, i.e. $-\infty< T_0$. 
\end{tabular}}
$
\smallskip

\noindent As we have remarked before, conditional intensity should depend on 
$N\cap(-\infty, t)$. Therefore, it cannot be function of $S_t$, since 
$S_t$ informs us if $t$ is a point or not. 
 That is the main reason for considering this $\mathcal{F}_{t-}^{N}$ measurable variable
\begin{equation}\label{eq:def:age:predictable}
S_{t-}=t-T_{N_{t-}},
\end{equation}
where $T_{N_{t-}}$ is the last point strictly before $t$ (see Figure \ref{fig:thin}). 
 Note also that knowing $(S_{t-})_{t\geq 0}$ or $(N_t)_{t\geq 0}$ is completely equivalent given $\mathcal{F}^{N}_{0}$.

The last and most crucial equivalence between (PPS) and the present point process set-up, consists in noting that the quantities $p(s,X(t))$ and $\lambda(t,\mathcal{F}_{t-}^{N})$ have informally the same meaning: they both represent a firing rate, i.e. both give the rate of discharge as a function of the past.
This dependence is made more explicit in $p(s,X(t))$ than in $\lambda(t,\mathcal{F}_{t-}^{N})$.

\vspace{-0.3cm}
\subsection{Examples}\label{sec:examples}
Let us review the basic point processes models of spike trains and see what kind of analogy is likely to exist between both models ((PPS) equation  and point processes). These informal analogies are possibly 
exact mathematical results (see Section \ref{sec:ex}).

\vspace{-0.2cm}
\paragraph{Homogeneous Poisson process} 
This is the simplest case where $\lambda(t,\mathcal{F}_{t-}^{N})=\lambda,$
with $\lambda$  a fixed positive 
constant
representing
the firing rate. 
There is no dependence in time $t$ (it is homogeneous) and no dependence with respect to the past. This case should be equivalent to $p(s,X(t))=\lambda$ in (PPS). 
This can be made even more explicit. Indeed in the case where the Poisson process exists on the whole real line (stationary case), it is easy to see that
$$\proba{S_{t-}>s}= \proba{N_{[t-s,t)}=0}=\exp(-\lambda s),$$
meaning that the age $S_{t-}$ obeys an exponential distribution with parameter $\lambda$, i.e. the steady state of the toy example developed for (PPS) when $ p(s,X(t))=\lambda$. 

\vspace{-0.2cm}
\paragraph{Inhomogeneous Poisson process}
To model non stationarity, one can use $\lambda(t,\mathcal{F}_{t-}^{N})=\lambda(t),$
which only depends on time. This case should be equivalent to the replacement of $p(s,X(t))$ in (PPS) by $\lambda(t)$.

\vspace{-0.2cm}
\paragraph{Renewal process}
This model is very useful to take refractory period into account. It corresponds to the case where the ISIs (delays between spikes) are independent and identically distributed (i.i.d.)
 with a certain given density $\nu$ on $\R_+$. The associated hazard rate is 
$$f(s)=\frac{\nu(s)}{\int_{s}^{+\infty} \nu(x) \dd x},$$
when $\int_{s}^{+\infty} \nu(x) \dd x>0$.
Roughly speaking, $f(s)\dd s$ is the probability that a neuron spikes with 
age $s$ given that 
its age is larger than $s$.
In this case, considering the set of spikes as the point process $N$, it is easy to show 
(see the Appendix \ref{ApRenewal})
 that its corresponding intensity is
$\lambda(t,\mathcal{F}_{t-}^{N})=f(S_{t-})$
which only depends on the age. One can also show quite easily that the process $(S_{t-})_{t> 0}$, which is equal to $(S_{t})_{t>0}$ almost everywhere (a.e.),
 is a Markovian process in time.  This renewal setting should be equivalent in the (PPS) framework to $p(s,X(t))=f(s).$

Note that many people consider IF models \eqref{IF} with Poissonian inputs with or without additive white noise. In both cases,  the system erases all memory after each spike and therefore the ISIs are i.i.d. Therefore as long as we are only interested by the spike trains and their point process models, those IF models are just a particular case of renewal process~\cite{BH,CCT,Delarue_Inglis_Rubenthaler_Tanre_2012,PPCV}.

\paragraph{Wold process and more general structures}
Let $A^1_t$ be the delay (ISI) between the last point and the 
occurrence just before 
(see also Figure \ref{fig:thin}), $A^1_t= T_{N_{t-}}-T_{N_{t-}-1}.$
A Wold process \cite{wold-lai,Daley:VereJones} is then 
characterized by $\lambda(t,\mathcal{F}_{t-}^{N})=f(S_{t-},A^1_t).$
This model has been 
matched to several real data thanks to goodness-of-fit tests \cite{Pouzat09} and is therefore one of our main example with the next discussed Hawkes process case. One can show in this case that the successive ISI's form a Markov chain of order 1 and that the continuous time process $(S_{t-}, A^1_t)$ is also Markovian.

This case should be equivalent to the replacement of $p(s,X(t))$ in (PPS) by $f(s,a^1)$, 
with $a^1$ denoting
the delay between the two previous spikes
. Naturally in this case,  one should expect a PDE of higher dimension 
with third variable $a^1$. 

More generally, one could define
\begin{equation}\label{eq:successive:ages}
A^k_t= T_{N_{t-}-(k-1)}-T_{N_{t-}-k},
\end{equation}
and point processes with intensity
$\lambda(t,\mathcal{F}_{t-}^{N})=f(S_{t-},A^1_t,...,A^k_t)$.
Those processes satisfy more generally that their ISI's form a Markov chain of  order $k$ and that the continuous time process $(S_{t-}, A^1_t,...,A^k_t)$ is also Markovian 
(see the Appendix \ref{ApWold}).

\begin{remark}
The dynamics of the successive ages is pretty simple.
On the one hand, 
the dynamics of the vector of the successive ages $(S_{t-},A^1_t,...,A^k_t)_{t>0}$ is deterministic between two jumping times. 
The first coordinate increases with rate $1$.
On the other hand, the dynamics at any jumping time $T$ is given by the following shift:
\begin{equation}\label{eq:dynamic:successive:ages}
\left\{
\begin{aligned}
& \text{the age process goes to $0$, i.e. } S_{T}=0, \\
& \text{the first delay becomes the age, i.e. } A_{T+}^{1}=S_{T-}, \\
& \text{the other delays are shifted, i.e. } A_{T+}^{i}=A_{T}^{i-1} \text{ for all } i\leq k.
\end{aligned}
\right.
\end{equation}
\end{remark}

\paragraph{Hawkes processes}
The most classical setting is the linear (univariate) Hawkes process, which corresponds to
$$\lambda(t,\mathcal{F}_{t-}^{N})=\mu + \int_{-\infty}^{t-} h(t-u) N(du),$$
where the positive parameter $\mu$ is called the spontaneous rate and the non negative function 
$h$, with support in $\R_+$,  
is called the interaction function, which is generally assumed to satisfy $\int_{\R_+} h <1$ to guarantee the existence of a stationary version \cite{Daley:VereJones}. 
This model has also been
matched to several real neuronal data thanks to goodness-of-fit tests  \cite{RBRGTM}. 
Since it can mimic synaptic integration, as 
explained below, this represents the main example of the present work.
 
In the case where $T_0$ tends to $-\infty$, this is equivalent to say that there is no point on the negative half-line and in this case, one can rewrite
$$\lambda(t,\mathcal{F}_{t-}^{N})=\mu + \int_{0}^{t-} h(t-u) N(du).$$
By  analogy between $N(\dd t)$ and $m(t)\dd t$, one sees that $\int_{0}^{t-} h(t-u) N(du)$ is indeed the analogous of $X(t)$ the synaptic integration in \eqref{X}. So one could expect that the PDE analogue is given by $p(s,X(t))=\mu+X(t)$. In Section \ref{sec:ex}, we show that this does not hold stricto sensu, whereas the other analogues work well.

Note that this model shares also some link with IF models. Indeed, the formula for the intensity is close to the formula for the voltage  \eqref{IF}, with the same flavor for the synaptic integration term. The main difference comes from the fact that when the voltage reaches a certain threshold, it fires deterministically for the IF model, whereas the higher the intensity, the more likely is the spike for the Hawkes model, but without certainty. In this sense Hawkes models seem closer to (PPS) since as we discussed before, the term $p(s,X(t))$ is closer to a hazard rate and never imposes deterministically the presence of a spike.

To model inhibition (see \cite{ieee} for instance), one can use functions $h$ that may take negative values and in this case
$\lambda(t,\mathcal{F}_{t-}^{N})=\left(\mu + \int_{-\infty}^{t-} h(t-u) N(du)\right)_+,$
which should correspond to $p(s,X(t))=\left(\mu+X(t)\right)_+$. 
Another possibility is 
$\lambda(t,\mathcal{F}_{t-}^{N})=\exp\left(\mu + \int_{-\infty}^{t-} h(t-u) N(du)\right),$
which is inspired by the generalized linear model as used by \cite{Pillow} and which should correspond to $p(s,X(t))=\exp\left(\mu+X(t)\right)$. 

Note finally that Hawkes models in Neuroscience (and their variant) are usually multivariate meaning that they model interaction between spike trains thanks to interaction functions between point processes, each process representing a neuron. To keep the present analogy as simple as possible, we do not deal with  those multivariate models in the present article. Some open questions in this direction are presented in conclusion.

\subsection{Ogata's thinning algorithm}
\label{sec:Ogata:Thinning}
To turn the analogy between $p(s,X(t))$ and $\lambda(t,\mathcal{F}_{t-}^{N})$  into a rigorous result on the  PDE level, we need to understand the intrinsic dynamics of the point process. This dynamics is 
often not
explicitly described in the literature (see 
e.g. the reference  book by Br\'emaud \cite{Bre}) because martingale theory provides a nice mathematical setting in which one can perform all the computations. However, when one wants to simulate point processes based on the knowledge of their intensity, there is indeed a dynamics that is required to obtain a practical algorithm. This method has been described at first by Lewis in the Poisson setting \cite{Lewis_Simul} and generalized by Ogata in \cite{ogatathin}. If there is a sketch of proof in \cite{ogatathin}, we have been unable to find any complete mathematical proof of this construction in the literature and we propose a full and mathematically complete version of this proof  with minimal assumptions in the Appendix \ref{sec:Thinning}. 
%
Let us just informally describe here, how this construction works.

The principle consists in assuming that is given an external homogeneous Poisson process $\Pi$ of intensity 1 in $\R_+^2$ and with associated point measure $\Pi\left(\dd t,\dd x\right) =\sum_{(T,V)\in {\Pi}}\delta_{\left(T,V\right)}(\dd t,\dd x)$.  This means in particular that
\begin{equation}\label{Pi}
\esp{\Pi(\dd t,\dd x)}=\dd t \ \dd x.
\end{equation}
Once a realisation of $N_-$ fixed, which implies that $\mathcal{F}_0^N$ is known and which can be seen as an initial condition for the dynamics, the construction of the process $N$ on $\R_+$ only depends on $\Pi$.

More precisely, if we know the intensity $\lambda(t,\mathcal{F}_{t-}^{N})$ in the sense of the "recipe" that explicitly depends on $t$ and $N\cap(-\infty,t)$, then once a realisation of $\Pi$ and of $N_-$ is fixed, the dynamics to build a point process $N$  with intensity $\lambda(t,\mathcal{F}_{t-}^{N})$ for $t\in\R_+$ is purely deterministic. It consists (see also Figure \ref{fig:thin}) in successively projecting
on the abscissa axis  the points that are 
below the graph of $\lambda(t,\mathcal{F}_{t-}^{N})$. Note that
{a point projection}
may change the shape of $\lambda(t,\mathcal{F}_{t-}^{N})$, just after the projection. Therefore 
the graph of $\lambda(t,\mathcal{F}_{t-}^{N})$ evolves thanks to the realization of $\Pi$. For a more mathematical description, see Theorem~\ref{thm:Thinning:R2} in the 
Appendix  \ref{sec:Thinning}.
 Note in particular that the construction ends on any finite interval $\left[ 0,T \right]$ a.s. if $\left(\mathcal{A}^{1,a.s}_{\lambda,loc}\right)$ holds.

Then the point process $N$, result of Ogata's thinning, is given by the union of $N_-$ on $\R_-$ and the projected points on $\R_+$. It admits the desired intensity 
$\lambda(t,\mathcal{F}_{t-}^{N})$ on $\mathbb{R}_{+}$. Moreover, the point measure can be represented by
\begin{equation}\label{defthin}
{\bf 1}_{t>0} \ N(\dd t)=\sum_{
\begin{subarray}{c}(T,X)\in \Pi~/\\ 
X \leq\lambda\left(T,\mathcal{F}_{T-}^{N} \right)
\end{subarray}}
\delta_{T}(\dd t)=
\left(\int_{x=0}^{\lambda \left(t,\mathcal{F}_{t-}^{N}\right)}\Pi \left(\dd t,\dd x\right)\right).
\end{equation}
\noindent NB: The last equality comes from the following convention. If $\delta_{(c,d)}$ is a Dirac mass in $(c,d)\in\R_+^2$, then
$\int_{x=a}^b \delta_{(c,d)}(\dd t,\dd x)$, as a distribution in $t$, is $\delta_{c}(\dd t)$ if $d\in[a,b]$ and 0 otherwise.
\begin{figure}[h!]
\begin{center}
\includegraphics[scale=0.8]{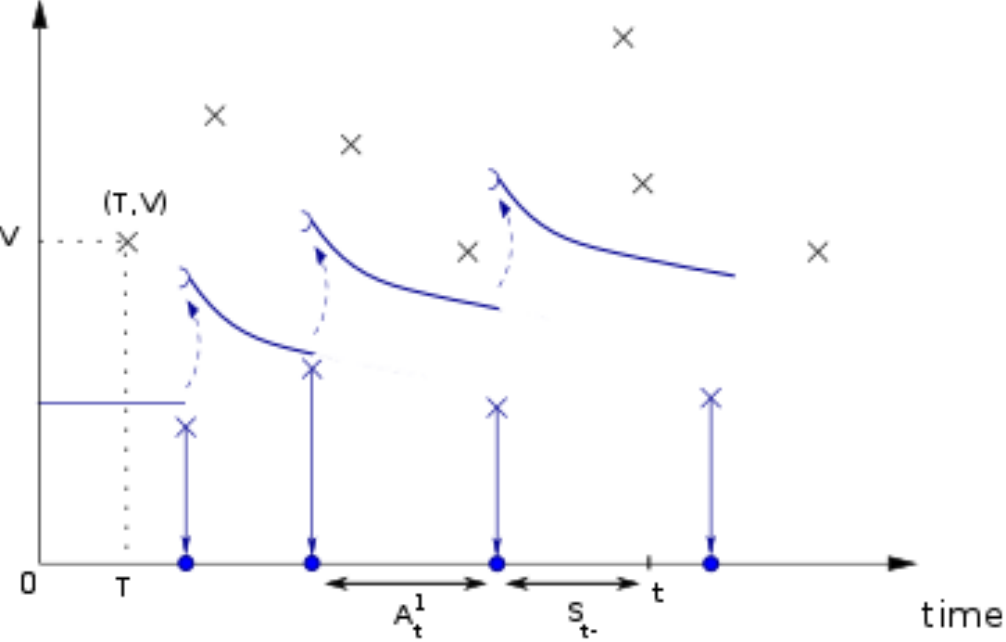}\\
~\vspace{-0.5cm}\\
\caption{\label{fig:thin} Example of Ogata's thinning algorithm on a linear Hawkes process with interaction function $h(u)=e^{-u}$ and no point before 0 (i.e. $N_-=\emptyset$). The crosses represent a realization of $\Pi$, Poisson process of intensity 1 on $\mathbb{R}_+^2$. The blue piecewise continuous line represents the intensity $\lambda(t,\mathcal{F}_{t-}^{N})$, which starts in 0 with value $\mu$ and then jumps each time a point of $\Pi$ is present underneath it.   The resulting Hawkes process (with intensity $\lambda(t,\mathcal{F}_{t-}^{N})$) is given by the blue circles. Age $S_{t-}$ at time $t$ and the quantity $A^1_t$ are also represented.}
~\vspace{-1cm}
\end{center}
\end{figure}

\section{\label{PDEpoint}From point processes to PDE}

Let us now present our main results. Informally, we want to describe the evolution of the distribution in $s$ of the age $S_{t}$ according to the time $t$. Note that at fixed time $t$, $S_{t-}=S_{t}$ a.s. and therefore it is the same as the distribution of $S_{t-}$. We prefer to study $S_{t-}$ since its predictability, i.e. its dependence in $N\cap(-\infty, t)$, makes all definitions proper from a microscopic/random point of view. 
Microscopically, the interest lies in the evolution of  $\delta_{S_{t-}}(\dd s)$ as a random measure.
But it should also be seen as a distribution in time, for equations like (PPS) to make sense.
Therefore, 
we need to go from a distribution only in $s$ to a 
distribution in both $s$ and $t$. Then one can either focus on the
 microscopic 
level, where 
the realisation of $\Pi$ in Ogata's thinning construction is fixed
or focus on the expectation of such a distribution. 

\subsection{\label{clean}A clean setting for bivariate distributions in age and time}

In order to obtain, from a point process,  (PPS) system we 
need to define bivariate distributions in $s$ and $t$ and  marginals 
(at least in $s$), in such a way that weak solutions of (PPS) are 
correctly defined.
Since we want to 
possibly
consider more than two variables for generalized 
Wold processes, we consider the following definitions.

In the following, $<\varphi,\nu>$ denotes the integral of the integrable function $\varphi$ with respect to the measure $\nu$.

Let $k\in \N$. For every bounded measurable function $\varphi$ of $(t,s,a_1,...,a_k)\in\R_+^{k+2}$, one can define 
$$\varphi^{(1)}_t(s,a_1,...,a_k)=\varphi(t,s,a_1,...,a_k) \quad \mbox{and} \quad \varphi^{(2)}_s(t,a_1,...,a_k)=\varphi(t,s,a_1,...,a_k).$$

Let us now define two sets of regularities for $\varphi$.\\
$\mathcal{M}_{c,b}(\R^{k+2}_+) 
\,
\textrm{\begin{tabular}{|l} The function $\varphi$ belongs to $\mathcal{M}_{c,b}(\R^{k+2}_+)$ if and only if  \\ $\quad\bullet$ $\varphi$ is a measurable bounded function, \\ $\quad\bullet$ there exists $T>0$ such that for all $t>T$, $\varphi^{(1)}_t=0$. \end{tabular}}$\\

\

\noindent$C_{c,b}^{\infty}(\R^{k+2}_+) 
\,
 \textrm{\begin{tabular}{|l} The function $\varphi$ belongs to $C_{c,b}^{\infty}(\R^{k+2}_+)$  if and only if \\ $\quad\bullet$ $\varphi$ is continuous, uniformly bounded, \\ $\quad\bullet$ $\varphi$ has  uniformly bounded derivatives of every order,\\ $\quad\bullet$ there exists $T>0$ such that for all $t>T$, $\varphi^{(1)}_t=0$. 
 \end{tabular}}$\\

Let $\left(\nu_{1}^{t}\right)_{t\geq 0}$ be a (measurable w.r.t. $t$) family
of positive measures on $\mathbb{R}_{+}^{k+1}$, and $\left(\nu_{2}^{s}\right)_{s\geq 0}$ be a (measurable
w.r.t. $s$) family of positive measures $\mathbb{R}_{+}^{k+1}$.
Those families satisfy the Fubini property if 
\\
\vspace{0.2cm}
\noindent$
\left(\mathcal{P}_{Fubini}\right)
\, \textrm{\begin{tabular}{|l}  for any $\varphi\in \mathcal{M}_{c,b}(\R_+^{k+2})$, 
$
\quad \int\langle \varphi^{(1)}_t,\nu_{1}^{t}\rangle \dd t=\int\langle\varphi^{(2)}_s, \nu_{2}^{s}\rangle \dd s.
$
\end{tabular}}$

In this case, one can define $\nu$, measure on $\R^{k+2}_+$, by the unique measure on $\R^{k+2}_+$ such that for any test function $\varphi$ in $\mathcal{M}_{c,b}(\R^{k+2}_+)$,
$$<\varphi,\nu>= \int\langle \varphi^{(1)}_t,\nu_{1}^{t}\rangle \dd t=\int\langle\varphi^{(2)}_s, \nu_{2}^{s}\rangle \dd s.$$
To simplify notations, 
for any such measure  $\nu(t,\dd s,\dd a_1,...,\dd a_k)$, we define
$$\nu(t,\dd s,\dd a_1,...,\dd a_k)= \nu_{1}^{t}(\dd s,\dd a_1,...,\dd a_k),  \quad \nu(\dd t,s,\dd a_1,...,\dd a_k)= \nu_{2}^{s}(\dd t,\dd a_1,...,\dd a_k).$$ 

In the sequel, we need in particular a measure on $\R_+^2$, $\eta_x$, defined for any real $x$ by its marginals that satisfy $\left(\mathcal{P}_{Fubini}\right)$ as follows
\begin{equation}\label{etax}\forall \, t,s \geq 0,\qquad  \eta_x(t,\dd s)= \delta_{t-x}(\dd s) {\bf 1}_{t-x>0}\quad \mbox{and}\quad \eta_x(\dd t,s)= \delta_{s+x}(\dd t){\bf 1}_{s\geq 0}.
\end{equation}
 It represents a Dirac mass "travelling" on the positive diagonal originated in $(x,0)$.

\subsection{The microscopic construction of a random PDE}

For a fixed realization of $\Pi$, we therefore want to define a random distribution $U(\dd t,\dd s)$ 
in terms of its marginals, thanks to $\left(\mathcal{P}_{Fubini}\right)$, such that, $U(t,\dd s)$ represents the distribution at time $t > 0$ of the age $S_{t-}$, i.e. 
\begin{equation}\label{upt}
\forall\, t> 0, \quad U(t,\dd s)=\delta_{S_{t-}}(\dd s)
\end{equation}
and satisfies similar equations as (PPS).
This is done in the following proposition. 
\begin{prop}\label{prop:Microscopic:System}
Let $\Pi$, $\mathcal{F}^N_{0}$ and an intensity $\left( \lambda(t,\mathcal{F}_{t-}^{N}) \right)_{t> 0}$ be given as in Section \ref{sec:Ogata:Thinning} and satisfying $\left(\mathcal{A}_{T_0}\right)$ and $\left(\mathcal{A}^{\mathbb{L}^1,a.s.}_{\lambda, loc}\right)$. On the event $\Omega$ of probability 1, where Ogata's thinning is well defined, let $N$ be the point process on $\mathbb{R}$ that is constructed thanks to Ogata's thinning with associated predictable age process $(S_{t-})_{t> 0}$ and whose points are denoted $\left( T_{i} \right)_{i\in \Z}$.
Let the (random) measure $U$ and its corresponding marginals be defined by
\begin{equation}\label{eq:def:Upsilon}
U\left(\dd t,\dd s\right)=\sum_{i=0}^{+\infty} \eta_{T_i}(\dd t,\dd s) \ \mathds{1}_{0\leq t\leq T_{i+1}}.
\end{equation}
Then, on $\Omega$, $U$
satisfies $\left(\mathcal{P}_{Fubini}\right)$ and $U(t,\dd s)=\delta_{S_{t-}}(\dd s)$. Moreover, on $\Omega$, $U$ is a solution in the weak sense of the following system 
 
\begin{gather}\label{eq:edp:micro}
\f{\partial}{\p {t}}U\left(\dd t,\dd s\right) +\f{\p}{\partial s} U\left(\dd t,\dd s\right) + \left(\int_{x=0}^{\lambda\left(t,\mathcal{F}_{t-}^{N}\right)}\Pi\left(\dd t,\dd x\right)\right) U\left(t,\dd s\right)=0,\\\label{eq:edp:micro:bound}
U\left(\dd t,0\right)=\int_{s\in \mathbb{R}_{+}} \left(\int_{x=0}^{\lambda\left(t,\mathcal{F}_{t-}^{N}\right)}\Pi\left(\dd t,\dd x\right)\right) U\left(t,\dd s\right)+\delta_0(\dd t)\1_{T_0=0}, \\
\label{eq:edp:micro:init}
U\left(0,\dd s\right)=\delta_{-T_0}(\dd s)\1_{T_0<0}=U^{in}(\dd s)\1_{s>0},
\end{gather}
where $U^{in}(ds)=\delta_{-T_0}(ds)$.
The weak sense means that 
 for any 
 $\varphi \in C^{\infty}_{c,b}(\mathbb{R}_+^{2})$, 
\begin{align}
&\!\!\!\!\!\!\int_{{\R_+\times \R_+}}\left(\frac{\partial}{\partial {t}}\varphi\left(t,s\right)+\f{\partial}{\partial {s}}\varphi\left(t,s\right)\right)
U\left(\dd t,\dd s\right) +
\nonumber
\\
&~\int_{{\R_+\times \R_+}}\!\![\varphi\left(t,0\right) - 
\varphi\left(t,s\right)]
\left(\int_{x=0}^{\lambda\left(t,\mathcal{F}_{t-}^{N}\right)}\Pi\left(\dd t,\dd x\right)\right)  U\left(t,\dd s\right) 
+ \varphi(0,-T_0)=0.
\label{eq:edp:micro:weak}
\end{align}
\end{prop}

\noindent The proof of Proposition  \ref{prop:Microscopic:System} is included in Appendix \ref{proofProp:Microscopic:System}. Note also that thanks to the Fubini property, the boundary condition~\eqref{eq:edp:micro:bound} is satisfied also in a strong sense.

System \eqref{eq:edp:micro}--\eqref{eq:edp:micro:init} is
a random microscopic version of (PPS) if $T_0<0$, where $n(s,t)$ the density of the age at time $t$ is replaced by $U(t,\cdot)=\delta_{S_{t-}}$, the Dirac mass in the age at time $t$. 
The assumption $T_0<0$ is satisfied \emph{a.s.} if $T_0$ has a density, but this may not be the case for instance if the experimental device gives an impulse at time zero (e.g.~\cite{Pouzat09} studied Peristimulus time histograms (PSTH),  where the  spike trains are locked on a stimulus given at time 0).

This result may seem rather poor from a PDE point of view. However, since 
this equation is satisfied at a microscopic level, we are able to define correctly all the important quantities at a macroscopic level.  Indeed, the analogy between $p(s,X(t))$ and $\lambda(t,\mathcal{F}_{t-}^{N})$ is actually on the random microscopic scale a replacement of $p(s,X(t))$ by $\int_{x=0}^{\lambda\left(t,\mathcal{F}_{t-}^{N}\right)}\Pi\left(\dd t,\dd x\right)$, 
whose expectancy given the 
past is, heuristically speaking, equal
 to $\lambda\left(t,\mathcal{F}_{t-}^{N}\right)$ because the mean behaviour of $\Pi$ is given by the Lebesgue measure (see \eqref{Pi}).
Thus, the main question at this stage is : can we make this argument valid by taking the expectation of $U$? 
This is addressed in the next section.

 The property $\left(\mathcal{P}_{Fubini}\right)$ and the quantities $\eta_{T_i}$ mainly allows to define $U(\dd t,0)$ as well as $U(t,\dd s)$. As expected, with this definition, \eqref{upt} holds as well as
\begin{equation}\label{eq:relation:U:N}
U\left(\dd t,0 \right) = \mathds{1}_{t\geq 0} \ N(\dd t),
\end{equation}
i.e. the spiking measure (the measure in time with age 0) is the point measure.

Note also that the initial condition is given by  $\mathcal{F}^N_{0}$, 
since $\mathcal{F}^N_{0}$ fixes in particular the value of 
$T_{0}$ and $\left(\mathcal{A}_{T_0}\right)$ is required to give 
sense to the age at time $0$. 
To understand the initial condition, remark that
 if $T_{0}=0$, then $U(0,\cdot)=0\neq \lim_{t\rightarrow 0^{+}} U(t,\cdot) = \delta_{0}$ by definitions of $\eta_{T_i}$
 but that
 if $T_{0}<0$, $U(0,\cdot)= \lim_{t\rightarrow 0^{+}} U(t,\cdot)=\delta_{-T_{0}}$.

The 
conservativeness (i.e. for all $t\geq0$, $\int_0^\infty U(t,\dd s)=1$) is obtained  by using (a sequence of test functions converging to) $\varphi=\1_{t\leq T}.$ 

Proposition  \ref{prop:Microscopic:System} shows that the (random) measure $U$,
defined by \eqref{eq:def:Upsilon}, in terms of a given  point process $N$, 
is a weak solution of System \eqref{eq:edp:micro}-\eqref{eq:edp:micro:init}. 
The study of the well-posedness of this system could be addressed following, for instance, the ideas given in \cite{canizo2013measure}. In this case
$U$ should be the unique solution of system \eqref{eq:edp:micro}--\eqref{eq:edp:micro:init}.

As last comment about Proposition  \ref{prop:Microscopic:System}, we 
analyse the particular case of the linear Hawkes process, in the following
remark.
\begin{remark}
In the 
 linear Hawkes process, 
$\lambda(t,\mathcal{F}_{t-}^{N})=\mu + \int_{-\infty}^{t-} h(t-z) N(\dd z)$.
Thanks to~\eqref{eq:relation:U:N} one decomposes the intensity into a term given by the initial condition plus a term given by the measure $U$, 
$
\lambda(t,\mathcal{F}_{t-}^{N})=\mu + 
F_{0}(t) + 
\int_0^{t-} h(t-z) U(\dd z,0),
$
where $F_{0}(t)=\int_{-\infty}^{0} h(t-z) N_{-}(\dd z)$ is $(\mathcal{F}^{N}_{0})$-measurable and considered as an initial condition. Hence,
\eqref{eq:edp:micro}--\eqref{eq:edp:micro:init} 
 becomes a closed system in the sense that $\lambda(t,\mathcal{F}_{t-}^{N})$ is now an explicit function of the solution of the system. This is not true in general.
\end{remark}
\subsection{The PDE satisfied in expectation}
In this section, we want to find the system satisfied by the expectation of the random measure $U$. 
First, we need to give a proper definition of such an object. The construction is based on the construction of $U$ and is summarized in the following 
proposition.
(The proofs of all the results of this subsection 
are in 
 Appendix \ref{proofProp:Microscopic:System}).
\begin{prop}\label{prop:Fubini:nu}
Let $\Pi$, $\mathcal{F}^N_{0}$ and an intensity $\left( \lambda(t,\mathcal{F}_{t-}^{N}) \right)_{t> 0}$ be given as in Section \ref{sec:Ogata:Thinning} and satisfying $\left(\mathcal{A}_{T_0}\right)$ and  $\left(\mathcal{A}^{\mathbb{L}^1,exp}_{\lambda,loc}\right)$. Let $N$ be the process resulting of Ogata's thinning
and let $U$ be the random measure defined by~\eqref{eq:def:Upsilon}. Let $\mathbb{E}$ denote the expectation with respect to $\Pi$ and $\mathcal{F}_0^N$.

Then for any test function 
$\varphi$ 
in $\mathcal{M}_{c,b}(\mathbb{R}_+^{2})$,
 $\mathbb{E}\left[ \int 
\varphi(t,s) U(t,\dd s) \right]$ and 
$\mathbb{E}\left[ \int 
\varphi(t,s) U(\dd t,s) \right]$ are finite and one can define
 $u(t,\dd s)$ and $u(\dd t,s)$ by  
$$
\begin{cases}
\displaystyle \forall\, t\geq 0,\quad \int 
\varphi(t,s) u(t,\dd s) = \mathbb{E}\left[ \int 
\varphi(t,s) U(t,\dd s) \right],\\
\forall\, s\geq 0,\quad \int 
\displaystyle \varphi(t,s) u(\dd t,s) = \mathbb{E}\left[ \int 
\varphi(t,s) U(\dd t,s) \right].
\end{cases}
$$
Moreover, $u(t,\dd s)$ and $u(\dd t,s)$ satisfy  
$\left(\mathcal{P}_{Fubini}\right)$ and one can define $u(\dd t,\dd s)=u(t,\dd s) 
\dd t=u(\dd t,s) \dd s$ on $\R_+^2$, such that for any test function 
$\varphi$ 
in $\mathcal{M}_{c,b}(\mathbb{R}_+^{2})$,
$$
\int 
\varphi(t,s) u(\dd t,\dd s) = \mathbb{E}\left[ \int 
\varphi(t,s) U(\dd t,\dd s) \right],
$$
quantity which is finite.
\end{prop}
In particular, since
$\int 
\varphi(t,s) u(t,\dd s) = \mathbb{E}\left[ \int 
\varphi(t,s) U(t,\dd s) \right]=\mathbb{E}\left[
\varphi(t,S_{t-})\right]$,
$u(t,\cdot)$ is therefore the distribution of $S_{t-}$, the (predictable version of the) age at time $t$. Now let us show that as expected, $u$ satisfies a system similar to (PPS).

\begin{thm}\label{prop: Expectation System}
Let $\Pi$, $\mathcal{F}^N_{0}$ and an intensity $\left( \lambda(t,\mathcal{F}_{t-}^{N}) \right)_{t> 0}$ be given as in Section \ref{sec:Ogata:Thinning} and  satisfying $\left(\mathcal{A}_{T_0}\right)$ and  $\left(\mathcal{A}^{\mathbb{L}^1,exp}_{\lambda,loc}\right)$. 
If $N$ is the process resulting of Ogata's thinning, 
$(S_{t-})_{t> 0}$ its associated predictable age process,  
$U$  its associated random measure, defined by  \eqref{eq:def:Upsilon},
and $u$ its associated  mean  measure,  defined in Proposition~\ref{prop:Fubini:nu}, then,
there exists a bivariate measurable function $\rho_{\lambda,\mathbb{P}_{0}}$ satisfying
\begin{equation}\label{eq:assumptions:rho}
\left\{
\begin{aligned}
& \forall\, T\geq 0,\, \int_{0}^{T} \int_{s} \rho_{\lambda,\mathbb{P}_{0}}(t,s) u(\dd t,\dd s) <\infty, \\
& \rho_{\lambda,\mathbb{P}_{0}}(t,s)=\mathbb{E}\left[ \lambda\left(t,\mathcal{F}_{t-}^{N}\right) \middle|S_{t-}=s \right] \quad  u(\dd t,\dd s) \emph{-
a.e}
\end{aligned}
\right.
\end{equation}
and such that $u$ is solution in the weak sense of the following system 
\begin{gather}\label{eq:edp:expect}
\f{\partial}{\p t}u\left(\dd t,\dd s\right)+\f{\p}{\partial s}u\left(\dd t,\dd s\right) +\rho_{\lambda,\mathbb{P}_{0}}(t,s)u\left(\dd t,\dd s\right)=0,\\
u\left(\dd t,0\right)=
\int_{s\in \mathbb{R}_+}\rho_{\lambda,\mathbb{P}_{0}}(t,s) u\left(t,\dd s\right)\,\dd t +\delta_0(\dd t)u^{in}(\{0\}),
\label{eq:edp:expect:bound} \\
u\left(0,\dd s\right)=u^{in}(\dd s)\1_{s>0},
\label{eq:edp:expect:init}
\end{gather}
where $u^{in}$ is the law of $-T_0.$ The weak sense means here that for any 
$\varphi\in C^{\infty}_{c,b}(\mathbb{R}_+^{2})$, 
\begin{align}\label{eq:edp:expect:weak}
&\!\!\!\!\!\!\!\int_{\R_+\times \R_+}\left(\f{\p}{\partial t}+\f{\p}{\partial s}\right) \varphi\left(t,s\right)u\left(\dd t,\dd s\right) + \nonumber \\
&~~~\int_{\R_+\times \R_+} [\varphi(t,0)-\varphi(t,s) ]\rho_{\lambda,\mathbb{P}_{0}}(t,s) u(\dd t,\dd s)  
+\int_{\R_+} \varphi(0,s)u^{in}(\dd s)=0,
\end{align}
\end{thm}
Comparing this system to (PPS), one first sees that 
$n(\cdot,t)$, the density of the age at time $t$, 
 is replaced by the mean  measure $u(t,\cdot)$. If  $u^{in}\in L^1(\R_+)$ we have $u^{in}(\{0\})=0$ so we get an equation which is exactly of renewal type, as (PPS). In the general case where $u^{in}$ is only a probability measure, the difference with (PPS) lies in the term $\delta_0(\dd t)u^{in}(\{0\})$ in the boundary condition for $s=0$ and in the term $\1_{s>0}$ in the initial condition for $t=0.$ Both these extra terms are linked to the possibility for  the initial measure $u^{in}$ to charge zero. This possibility is not considered in~\cite{PPS1} - else, a similar extra term would be needed in the setting of~\cite{PPS1} as well. As said above in the comment of Proposition~\ref{prop:Microscopic:System}, we want to keep this term here since it models the case where there is a specific stimulus at time zero~\cite{Pouzat09}.  

In general and without more assumptions on $\lambda$, it is not clear 
that 
$u$ is not only a measure 
satisfying $\left(\mathcal{P}_{Fubini}\right)$ but also absolutely continuous 
wrt to $\dd t\ \dd s$ and 
that the equations can be satisfied in a strong sense.

Concerning 
 $p(s,X(t))$, which has always been thought of as the equivalent of $\lambda(t,\mathcal{F}_{t-}^{N})$, it is not replaced by $\lambda(t,\mathcal{F}_{t-}^{N})$, which would have no meaning in general since this is a random quantity, nor by $\esp{\lambda(t,\mathcal{F}_{t-}^{N})}$ which would have been a first possible guess; it is replaced by $\esp{\lambda(t,\mathcal{F}_{t-}^{N})|S_{t-}=s}$. Indeed intuitively, since
$$\mathbb{E}\left[\int_{x=0}^{\lambda\left(t,\mathcal{F}_{t-}^{N}\right)} \Pi\left(\dd t,\dd x\right) \middle| \mathcal{F}_{t-}^{N} \right]= \lambda\left(t,\mathcal{F}_{t-}^{N}\right)\dd t, $$
the corresponding weak term can be interpreted 
as, for any test function $\varphi$,

\begin{eqnarray*}
\mathbb{E}\left[\int\varphi\left(t,s\right) \left(\int_{x=0}^{\lambda\left(t,\mathcal{F}_{t-}^{N}\right)} \Pi\left(\dd t,\dd x\right)\right) U\left(t,\dd s\right)\right]& = &\mathbb{E}\left[\int\varphi\left(t,s\right) \lambda\left(t,\mathcal{F}_{t-}^{N}\right) \delta_{S_{t-}}(\dd s) \dd t\right]\\
&=& \int_t \mathbb{E}\left[\varphi\left(t,S_{t-}\right)\lambda\left(t,\mathcal{F}_{t-}^{N}\right)\right] \dd t\\
&=& \int_t \mathbb{E}\left[\varphi\left(t,S_{t-}\right) \esp{\lambda\left(t,\mathcal{F}_{t-}^{N}\right)|S_{t-}} \right] \dd t,
\end{eqnarray*}
which is exactly $\int \varphi(t,s)\rho_{\lambda,\mathbb{P}_0}(t,s) u(\dd t,\dd s).$

This conditional expectation makes dependencies particularly complex, but this also enables to derive equations even in non-Markovian setting (as Hawkes processes for instance, see Section \ref{sec:ex}). More explicitly, $\rho_{\lambda,\mathbb{P}_{0}}(t,s) $ is a function of the time $t$, of the age $s$, but it also depends on $\lambda$, the shape of the intensity of the underlying process and on the distribution of the initial condition $N_-$, that is $\mathbb{P}_0$. As 
explained in Section \ref{secPP}, it is both the knowledge of $\mathbb{P}_{0}$ and 
$\lambda$
that characterizes the distribution of the process and in general the conditional expectation cannot be reduced to something depending on less than that.  In Section \ref{sec:ex}, we discuss several examples of point processes where one can (or cannot) reduce the dependence.

 Note that here again, we can prove that the equation is conservative by taking (a sequence of functions converging to) $\varphi=\1_{t\leq T}$ as a test function.

A direct corollary of Theorem~\ref{prop: Expectation System} can be deduced thanks to the law of large numbers. This can be seen as the interpretation of (PPS) equation at a macroscopic level, when the population of neurons is i.i.d.. 

\begin{cor}\label{prop:Macroscopic:System}
Let $\left( N^{i} \right)_{i=1}^{\infty}$ be some i.i.d. point processes with intensity given by $\lambda(t,\mathcal{F}_{t-}^{N^{i}})$ on $(0,+\infty)$  satisfying $\left(\mathcal{A}^{\mathbb{L}^1,exp}_{\lambda,loc}\right)$ and associated predictable age processes $(S_{t-}^{i})_{t> 0}$. Suppose  furthermore that the distribution of $N^{1}$ on $(-\infty,0]$ is given by $\mathbb{P}_{0}$ which is such that $\mathbb{P}_{0}(N_{-}^{1}=\emptyset)=0$. 

 Then there exists a measure $u$ satisfying $\left(\mathcal{P}_{Fubini}\right)$, weak solution of Equations~\eqref{eq:edp:expect} and~\eqref{eq:edp:expect:bound}, 
with $\rho_{\lambda, \mathbb{P}_0}$ defined by
$$\rho_{\lambda, \mathbb{P}_0}(t,s)= \esp{\lambda\left(t,\mathcal{F}_{t-}^{N^1}\right)|S_{t-}^1=s}, \quad u(\dd t,\dd s)-a.e.$$
and with $u^{in}$ distribution of the age at time 0, such that  
for any $\varphi\in C^{\infty}_{c,b}(\mathbb{R}_+^{2})$
\begin{equation}\label{eq:convergence:LLN}
\forall\, t>0,\quad \int \varphi(t,s) \left( \frac{1}{n}\sum_{i=1}^{n} \delta_{S_{t}^{i}}(\dd s) \right) \xrightarrow[n\rightarrow \infty]{a.s.} \int \varphi(t,s) u(t,\dd s),
\end{equation}
\end{cor}
In particular, informally, the fraction of neurons at time $t$ with age in $[s,s+\dd s)$ in this i.i.d. population of neurons indeed tends to $u(t,\dd s)$.


\section{Application to the various  examples \label{sec:ex}}
Let us now apply these results to the examples presented in Section \ref{sec:examples}.

\subsection{When the intensity only depends on time and age}
\label{sec:renewal}
If $\lambda\left(t,\mathcal{F}_{t-}^{N}\right)=f\left(t,S_{t-}\right)$
(homogeneous and inhomogeneous Poisson processes and renewal processes
are particular examples)
 then the intuition giving that $p(s,X(t))$ is analogous to $\lambda\left(t,\mathcal{F}_{t-}^{N}\right)$ works. Let us assume that $f(t,s)\in L^\infty(\R_+^2).$ We have
$\esp{\lambda\left(t,\mathcal{F}_{t-}^{N}\right)|S_{t-}=s}=f(t,s)$.  Under this assumption, we may apply Theorem~\ref{prop: Expectation System}, so that we know that the mean measure $u$ associated to the random process is a solution of System~\eqref{eq:edp:expect}--\eqref{eq:edp:expect:init}.
Therefore the mean measure $u$ satisfies a completely explicit PDE of the type (PPS) with $\rho_{\lambda,\mathbb{P}_{0}}(t,s)=f(t,s)$ replacing $p(s,X(t))$. In particular,  in this case $\rho_{\lambda,\mathbb{P}_{0}}(t,s)$ does not depend  on the initial condition. As already underlined, in general,  the distribution of the process is characterized by $
\lambda\left(t,\mathcal{F}_{t-}^{N}\right)=f\left(t,S_{t-}\right)$ and by the distribution of $N_-$. 
Therefore, in this  special case,
this dependence is actually reduced to the function $f$ and the distribution of $-T_0$. 
 Since
$ f(\cdot,\cdot) \in L^\infty\big([0,T]\times \R_+\big),$
assuming also $u^{in}\in L^1(\R_+),$ it is well-known that there exists a unique solution 
$u$ such that $\left( t\mapsto u(t,\cdot) \right)\in {\cal C}\bigl([0,T],L^1(\R_+)\bigr)$, see for instance~\cite{BP} Section 3.3. p.60.
{Note that following \cite{canizo2013measure} uniqueness for measure solutions may also be established,  h}ence the mean measure $u$ associated to the random process is the unique solution of System~\eqref{eq:edp:expect}--\eqref{eq:edp:expect:init}, and it is in ${\cal C}\bigl([0,T],L^1(\R_+)\bigr)$: the PDE formulation, together with existence and uniqueness, has provided a regularity result on $u$ which is obtained under weaker assumptions than through  Fokker-Planck / Kolmogorov equations. This is another possible application field of our results: using the PDE formulation to gain regularity. Let us now develop the Fokker-Planck / Kolmogorov approach for renewal processes.
\paragraph{Renewal processes}
The renewal process, i.e. when $\lambda\left(t,\mathcal{F}_{t-}^{N}\right)=f\left(S_{t-}\right)$, with  $f$ a continuous function on $\mathbb{R}_{+}$, has particular properties. As noted in Section \ref{sec:examples}, the renewal age process $(S_{t-})_{t>0}$ is an homogeneous Markovian process.  
It is known for a long time that it is easy to derive PDE on the corresponding density  through Fokker-Planck / Kolmogorov equations, once the variable of interest (here the age) is Markovian (see for instance~\cite{bossy_2010}). Here we briefly follow this line to see what kind of PDE can be derived through the Markovian properties and to compare the equation with the  (PPS) type system derived in Theorem \ref{prop: Expectation System}.

Since $f$ is continuous, the infinitesimal generator\footnote{The infinitesimal generator of an homogeneous Markov process $(Z_t)_{t\geq 0}$ is the operator $\mathcal{G}$ which is defined to act on every function $\phi : \mathbb{R}^{n}\rightarrow \mathbb{R}$ in a suitable space $\mathcal{D}$ by
$$
\mathcal{G} \phi (x) = \lim_{t \rightarrow 0^{+}} \frac{\mathbb{E}\left[ \phi(Z_{t}) \middle| Z_{0}=x \right] - \phi(x)}{t}.
$$} of $(S_{t})_{t> 0}$ is given by
\begin{equation}\label{eq:inf:generator}
(\mathcal{G}\phi)(x)=\phi'(x) + f(x)\left( \phi(0) - \phi(x) \right),
\end{equation}
for all $\phi \in C^{1}(\mathbb{R}_{+})$ (see~\cite{markovian:renewal}). Note that, since for every $t>0$ $S_{t-}=S_{t}$ a.s., the process $(S_{t-})_{t> 0}$ is also Markovian with the same infinitesimal generator.

Let us now define for all $t>0$ and all $\phi \in C^{1}(\mathbb{R}_{+})$,
$$
P_{t}\phi(x) = \mathbb{E}\left[ \phi(S_{t-}) \middle| S_{0}=x \right] = \int \phi(s) u_{x}(t,\dd s),
$$
where $x\in \mathbb{R}_{+}$ and $u_{x}(t,\cdot)$ is the distribution of $S_{t-}$ given that $S_{0}=x$. Note that $u_x(t,\dd s)$ corresponds to the marginal in the sense of $\left(\mathcal{P}_{Fubini}\right)$  of 
 $u_x$ given by Theorem~\ref{prop: Expectation System} with  $\rho_{\lambda,\mathbb{P}_0}(t,s)=f(s)$ and initial condition $\delta_{x}$, i.e. $T_{0}=-x$ a.s. 

In this homogeneous Markovian case, the forward Kolmogorov equation gives 
$$
\f{\partial}{\p {t}} P_{t} = P_{t} \mathcal{G}.
$$
 Let $\varphi \in C^{\infty}_{c,b}(\mathbb{R}_+^{2})$
and let $t> 0$. This implies that
\begin{eqnarray*}
\f{\partial}{\p {t}} \left( P_{t}\varphi(t,s) \right) & = & P_{t} \mathcal{G} \varphi(t,s) + P_{t} \f{\p}{\partial t}\varphi(t,s)\\
&=& P_{t} \left[ \f{\partial}{\p {s}}\varphi(t,s) + f(s)\left( \varphi(t,0) - \varphi(t,s) \right) + \f{\p}{\partial {t}}\varphi(t,s) \right].
\end{eqnarray*}
Since $\varphi$ is compactly supported in time, an integration with respect to $t$ yields
$$
-P_0 \varphi(0,s) = \int P_{t} \left(\f{\p}{ \partial {t}}+\f{\p}{\partial {s}} \right) \varphi(t,s) \dd t +\int P_{t} f(s)\left( \varphi(t,0) - \varphi(t,s) \right) \dd t,
$$
or equivalently
\begin{equation}\label{eq:edp:markovian:weak}
-\varphi(0,x) = \int\left(\f{\p}{\partial {t}}+\f{\p}{\partial {s}}\right) \varphi\left(t,s\right)u_{x}\left(t,\dd s\right) \dd t - \int ( \varphi(t,s)-\varphi(t,0) ) f(s) u_{x}(t,\dd s) \dd t,
\end{equation}
in terms of $u_{x}$. This is exactly Equation~\eqref{eq:edp:expect:weak} with $u^{in}=\delta_x$.

The result of Theorem \ref{prop: Expectation System} is stronger than the application of the forward Kolmogorov equation on  homogeneous Markovian systems 
since the result of Theorem \ref{prop: Expectation System} never used the Markov assumption and can be applied to non Markovian processes 
(see  Section~\ref{sec:hawkes:process:computations}).
So the present work is a general set-up where one can deduce PDE even from non Markovian microscopic random dynamics. {Note also that only boundedness assumptions and not continuity ones are necessary to directly obtain \eqref{eq:edp:markovian:weak} via Theorem \ref{prop: Expectation System}:} to obtain the classical Kolmogorov theorem, one would have assumed $f\in {\cal C}^0(\R_+^2)$ rather than $f\in L^\infty(\R_+^2).$

\subsection{\label{homMar} Generalized Wold process}

In the case where 
$\lambda\left(t,\mathcal{F}_{t-}^{N} \right)= f(S_{t-},A^1_t,...,A^k_t)$,
with $f$ being a non-negative function,
one can define in a similar way $ u_{k}\left(t,s,a_1,\dots,a_k\right)$ which is informally the distribution at time $t$ of the processes with age $s$ and past given by $a_1,...a_k$ for the last $k$ ISI's.
We want to investigate this case not for its Markovian properties, which are nevertheless presented in Proposition~\ref{prop:markovian:property} in the appendix for sake of completeness, but because this is the first basic example where the initial condition is indeed impacting $\rho_{\lambda,\mathbb{P}_0}$ in Theorem \ref{prop: Expectation System}.

\noindent To do so, the whole machinery applied on $u(\dd t,\dd s)$  is first extended in the  next result to $u_{k}\left(\dd t,\dd s,\dd a^1,\dots,\dd a^k\right)$ which represents the dynamics of the age and the last $k$ ISI's. This could have been done in a very general way by an easy generalisation of Theorem \ref{prop: Expectation System}. However to avoid too cumbersome equations, we express it only for generalized Wold processes to provide a clean setting to illustrate 
the impact of the initial conditions on $\rho_{\lambda,\mathbb{P}_0}$.
Hence, we similarly define a random distribution $U_{k}(\dd t,\dd s,\dd a_{1},\dots,\dd a_{k})$ such that its evaluation at any given time $t$ exists and is
\begin{equation}\label{upt:k}
U_{k}(t,\dd s,\dd a_{1},\dots,\dd a_{k})=\delta_{(S_{t-},A^1_t,...,A^k_t)}(\dd s,\dd a_{1},\dots ,\dd a_{k}).
\end{equation}
The following result states the PDE satisfied by $u_k=\esp{U_k}$.
\begin{prop}\label{prop: Expectation System:k}
Let $k$ be a positive integer and 
 $f$ be some non negative function on $\R_+^{k+1}$. Let $N$ be a generalized Wold process with predictable age process $(S_{t-})_{t> 0}$, associated points $(T_i)_{i\in \Z}$ and intensity $\lambda(t,\mathcal{F}_{t-}^{N})= f(S_{t-},A^1_t,...,A^k_t)$
satisfying $\left(\mathcal{A}^{\mathbb{L}^1,exp}_{\lambda,loc}\right)$, where $A^{1}_{t},\dots,A^{k}_{t}$ are the successive ages defined by~\eqref{eq:successive:ages}.
Suppose that $\mathbb{P}_{0}$ is such that $\mathbb{P}_{0}(T_{-k}>-\infty) = 1$.
Let $U_{k}$ be defined by
\begin{equation}\label{eq:def:Upsilon:k}
U_{k}\left(\dd t,\dd s,\dd a_{1},\dots,\dd a_{k}\right)=\sum_{i=0}^{+\infty} \eta_{T_i}(\dd t,\dd s) \prod_{j=1}^{k} \delta_{A_{T_{i}}^{j}}(\dd a_{j}) \ \mathds{1}_{0\leq t\leq T_{i+1}},
\end{equation}
{If $N$ is the result of Ogata's thinning on the Poisson process $\Pi$,} then $U_k$ satisfies \eqref{upt:k} and $\left(\mathcal{P}_{Fubini}\right)$ 
a.s.
in $\Pi$ and $\mathcal{F}^N_{0}$. 
Assume that the initial condition $u^{in}_k$, defined as the distribution of $(-T_{0},A^{1}_{0},\dots,A^{k}_{0})$ which is a random vector in $\mathbb{R}^{k+1}$, is such that $u^{in}_k(\{0\}\times \R^k_+)=0$. Then
 $U_k$ admits a mean measure $u_{k}$ which also satisfies $\left(\mathcal{P}_{Fubini}\right)$ and the following system in the weak sense: 
on $\R_+\times \R_+^{k+1}$,
\begin{gather}\label{eq:edp:expect:k}
\! \big\{ \!\f{\p}{\partial {t}}\!+\!\f{\p}{\partial {s}}\! \big\} u_{k}\!\left(\dd t,\dd s,\dd a_{1},...,\dd a_{k}\right)\! +\! f\!\left(s,a_{1},...,a_{k}\right)\!u_{k}\!\left(\dd t,\dd s,\dd a_{1},...,\dd a_{k}\right)\!=0,\\
u_{k}\left(\dd t,0,\dd s,\dd a_{1},... ,\dd a_{k-1}\right)\!=\!\!\int\limits_{a_{k}=0}^{\infty}\!\! f\!\left(s,a_{1},...,a_{k}\right) u_{k}\!\left(t,\dd s,\dd a_{1},...,\dd a_{k}\right)\,\dd t,
\label{eq:edp:expect:bound:k} \\
u_{k}\left(0,\dd s,\dd a_1,\dots,\dd a_k\right)=u^{in}_k\left(\dd s,\dd a_1,\dots,\dd a_k\right).\label{eq:edp:expect:init:k}
\end{gather}
\end{prop}
We have assumed $u^{in}_k(\{0\}{\times \R_+^k})=0$ {(i.e. $T_0\not = 0$ a.s.)} for the sake of simplicity, but this assumption may of course be relaxed and Dirac masses at $0$ should then be added in a similar way as in Theorem~\ref{prop: Expectation System}. 

If $f\in L^\infty(\R_+^{k+1}),$ we may apply Proposition~\ref{prop: Expectation System:k}, so that the mean measure $u_k$ satisfy System~\eqref{eq:edp:expect:k}--\eqref{eq:edp:expect:init:k}. 
Assuming an initial condition $u^{in}_k\in L^1(\R_+^{k+1}),$ we can prove exactly as for the renewal equation (with a Banach fixed point argument for instance) that  there exists a unique solution 
$u_k$ such that $\left( t\mapsto u_{k}(t,\cdot) \right) \in {\cal C} \big(\R_+,L^1(\R_+^{k+1})\big)$~\cite{BP} 
to the generalized Wold case, 
the boundary assumption on the $k$th penultimate point before time $0$ being necessary to give sense to the successive ages at time $0$. By uniqueness, this proves that the mean measure $u_k$ is this solution, so that it belongs to ${\cal C} \big(\R_+,L^1(\R_+^{k+1})\big):$ Proposition~\ref{prop: Expectation System:k} leads to a regularity result on the mean measure.

Now that we have clarified the dynamics of the successive ages, one can look at this system from the point of view of Theorem \ref{prop: Expectation System}, that is when only two variables $s$ and $t$ are considered. In this respect, let us note that $U$ defined by~\eqref{eq:def:Upsilon} is such that
$
U(\dd t,\dd s)=\int_{a_{1},\dots,a_{k}} U_{k}(\dd t,\dd s,\dd a_{1},\dots,\dd a_{k}).
$
Since the integrals and the expectations are exchangeable in the weak sense, the mean measure $u$ defined in Proposition~\ref{prop:Fubini:nu} is such that
$
u(\dd t,\dd s)=\int_{a_{1},\dots,a_{k}} u_{k}(\dd t,\dd s,\dd a_{1},\dots,\dd a_{k}).
$
But \eqref{eq:edp:expect:k} in the weak sense means, for all $\varphi \in C_{c,b}^{\infty}(\R^{k+2})$,
\begin{multline}\label{eq:edp:markovian:weak:k}
\int \left(\f{\p}{\partial {t}}+\f{\p}{\partial {s}}\right)\varphi(t,s,a_1,...,a_k)u_{k}\left(\dd t,\dd s,\dd a_1,\dots,\dd a_k\right)\\
+ \int \left[\varphi\left(t,0,a_1,\dots,a_k\right)-\varphi(t,s,a_1,\dots,a_k) \right]f\left(s,a_1,\dots,a_k\right)u_{k}\left(\dd t,\dd s,\dd a_1,\dots,\dd a_k\right)\\
+ \int\varphi\left(0,s,a_1,\dots,a_k\right)u_{k}^{in}\left(\dd s,\dd a_1,\dots,\dd a_k\right)=0.
\end{multline}
Letting $\psi\in C_{c,b}^{\infty}(\R^{2})$ and $\varphi\in C_{c,b}^{\infty}(\R^{k+2})$ being such that
$$
\forall \, t,s,a_{1},\dots ,a_{k},\quad \varphi(t,s,a_{1},\dots ,a_{k}) = \psi(t,s),
$$
we end up proving that the function $\rho_{\lambda,\mathbb{P}_{0}}$ defined in Theorem~\ref{prop: Expectation System} satisfies
\begin{equation}\label{rhoMar}
\rho_{\lambda,\mathbb{P}_{0}}(t,s) u\left(\dd t,\dd s\right) = \int_{a_{1},\dots ,a_{k}} f\left(s,a_1,\dots,a_k \right) u_{k}\left(\dd t,\dd s,\dd a_1,\dots,\dd a_k\right), 
\end{equation}
$u(\dd t,\dd s)-$almost everywhere ($a.e.$).
Equation \eqref{rhoMar} means exactly from a probabilistic point of view that
$$\rho_{\lambda,\mathbb{P}_{0}}(t,s) = \esp{f(S_{t-},A^1_t,...,A^k_t)|S_{t-}=s}, \quad u(\dd t,\dd s)-a.e.$$
Therefore, 
in the particular case of generalized Wold process, the quantity $\rho_{\lambda,\mathbb{P}_{0}}$ depends on the shape of the intensity (here the function $f$) and also on $u_{k}$. But, by Proposition \ref{prop: Expectation System:k},  $u_{k}$ depends on its initial condition given by the distribution of $(-T_{0},A^{1}_{0},\dots,A^{k}_{0})$, and not only $-T_{0}$ as in the initial condition for $u$. That is, as announced in the remarks following Theorem \ref{prop: Expectation System}, $\rho_{\lambda,\mathbb{P}_{0}}$ depends in particular on the whole distribution of the underlying process before time $0$, namely $\mathbb{P}_{0}$ and not only on the initial condition for $u$. Here, for generalized Wold processes, it only depends on the last $k$ points before time $0$. For more general non Markovian settings, the integration cannot be simply described by a measure $u_k$ in  dimension ($k+2$) being integrated with respect to $\dd a^1...\dd a^k$. In general, the integration has to be done on all the "randomness" hidden behind the dependence of $\lambda(t,\mathcal{F}_{t-}^{N})$ with respect to the past once $S_{t-}$ is fixed and in this sense it depends on the whole distribution $\mathbb{P}_0$ of $N_-$. This is made even clearer on the following non  Markovian example: the Hawkes process.

\vspace{-0.3cm}
\subsection{Hawkes process}\label{sec:hawkes:process:computations}
As we have seen in Section \ref{sec:examples}, there are many different examples of Hawkes processes that can all be expressed as
$
\lambda\left(t,\mathcal{F}_{t-}^{N} \right) = \phi\left( \int_{-\infty}^{t-} h\left(t-x\right) N(\dd x)\right)
$,
where the main case is $\phi(\theta)=\mu+\theta$, for $\mu$ some positive constant, which is the linear case.

When there is no point before $0$, 
$\lambda\left(t,\mathcal{F}_{t-}^{N} \right) = \phi\left( \int_{0}^{t-} h\left(t-x\right) N(\dd x)\right)$.
In this case, the interpretation is so close to (PPS) that the first guess, which is wrong, would be that the analogous in (PPS) is 
\begin{equation}\label{pX}
p(s,X(t))=\phi(X(t)),
\end{equation}
where $X(t)=\esp{\int_{0}^{t-} h\left(t-x\right) N(\dd x)}=\int_0^t h\left(t-x\right) u(\dd x,0)$.
This is wrong, even in the linear case 
since $\lambda\left(t,\mathcal{F}_{t-}^{N} \right)$ depends on all the previous points. Therefore 
$\rho_{\lambda,\mathbb{P}_0}$ defined by~\eqref{eq:assumptions:rho}
corresponds to a conditioning given only the last point. 

By looking at this problem through the generalized Wold approach, one can hope that for $h$ decreasing fast enough:
$$\lambda\left(t,\mathcal{F}_{t-}^{N} \right)\simeq \phi\left(h(S_{t-})+h(S_{t-}+A^1_t)+...+h(S_{t-}+A^1_t+...+A^k_t)\right).$$

In this sense and with respect to generalized Wold processes described in the previous section, we are informally  integrating on "all the previous points" except the last one and not integrating over all the previous points. This is informally why \eqref{pX} is wrong even in the linear case.
Actually, $\rho_{\lambda,\mathbb{P}_0}$ is computable for linear Hawkes processes
: {we show in the next section that}
$\rho_{\lambda,\mathbb{P}_{0}}(t,s)\not=\phi( \int_{-\infty}^t h(t-x)u(\dd x,0) )={\mu+\int_0^{\infty} h(t-x) u(\dd x,0)}$ and {that} 
$ \rho_{\lambda,\mathbb{P}_{0}}$ explicitly depends on $\mathbb{P}_0$.

\subsubsection{Linear Hawkes process}


We are interested in Hawkes processes with a past before time $0$ given
by 
$\mathcal{F}^N_{0}$, which 
is 
not necessarily the past given by a stationary Hawkes 
process. 
To illustrate the fact that the past is impacting the value of $\rho_{\lambda,\mathbb{P}_0}$, we focus on two particular cases:\\
$\left(\mathcal{A}_{N_-}^{1}\right) \,
\textrm{\begin{tabular}{|l} $N_{-}=\{T_{0}\}$ a.s. and $T_{0}$ admits a bounded density $f_{0}$ on $\mathbb{R}_{-}$ \end{tabular}}$\\
$\left(\mathcal{A}_{N_-}^{2}\right) \,
\textrm{\begin{tabular}{|l} $N_{-}$ is an homogeneous Poisson process with intensity $\alpha$ on $\mathbb{R}_{-}$ \end{tabular}}$

Before stating the main result, we need some technical 
definitions. 
Indeed the proof  is based on the underlying  branching structure of the linear Hawkes process described in Section \ref{secClus} of the appendix and the following functions $(L_s,G_s)$ are naturally linked to this branching decomposition (see Lemma \ref{prop:implicit:formulas:L:G}).
%
%
\begin{lem}\label{unicityLG}
Let $h\in L^1(\R_+)$ such that $\|h\|_{L^1}<1$.
For all $s\geq 0,$ there exist a unique solution  $(L_s,G_s)\in L^1 (\R_+)\times L^\infty(\R_+)$ of the following system 
\begin{eqnarray}\label{eq:Gs}
\log(G_s(x))=\int_{0}^{(x-s)\vee 0} G_s(x-w) h(w) \dd w -\int_{0}^x h(w) \dd w,\\ 
\label{eq:Ls}
L_{s}(x)= \int_{s\wedge x}^{x} \left(h\left(w\right)+ L_{s}(w) \right) G_{s}(w) h(x-w) \, \dd w,
\end{eqnarray}
where $a\vee b$ (resp. $a\wedge b$) denotes the maximum (resp. minimum) between $a$ and $b$.
Moreover,  $L_s(x\leq s)\equiv 0,$ $ G_s:\R_+\to[0,1],$ and $L_s$ is uniformly bounded in $L^1.$ 
\end{lem}
This result allows to define two other important quantities, $K_s$ and $q$, by, {for all  $s,t \geq 0, z\in \R$,}
\begin{multline}
K_{s}(t,z) := \int_{0}^{(t-s)\vee 0} \left[ h(t-x) +L_{s}(t-x) \right] G_{s}(t-x) h(x-z) \dd x, 
\\
\log(q(t,s,z))\!:=\! -\int_{(t-s)\vee 0}^{t}\!\!\!\!\!h(x-z) \dd x  - \int_{0}^{(t-s)\vee 0}\!\!\!\!\!\left[ 1-G_{s}(t-x) \right] h(x-z) \dd x. \label{Kqdef}
\end{multline}
Finally, the following result is just an obvious remark that helps to understand the resulting system.

\begin{remark}\label{unicityv}
For a non negative $\Phi\in L^\infty(\R_+)$ and $v^{in}\in L^\infty(\R_+)$, there exists a unique solution $v\in L^\infty(\R_+^2)$ in the weak sense to the following system, 
\begin{eqnarray}\label{eq:PDEPhi:v}
\f{\p}{\p t} v(t,s) + \f{\p}{\p s} v(t,s)+ \Phi(t,s) v(t,s)=0,\\
v(t,0)=1
\qquad v(t=0,s)=v^{in}(s)\label{eq:PDEPhi:bound:v}
\end{eqnarray}
Moreover $t\mapsto v(t,.)$ is in $C(R_+,L^1_{loc}(\R_+))$. 

If $v^{in}$ is a survival function (i.e. non increasing from 0 to 1), then $v(t,.)$ is a survival function  and $-\partial_s v$ is a probability measure for all $t>0$.
\end{remark}

\begin{prop}\label{prop:conditional:intensity:hawkes1}
Using the notations of Theorem~\ref{prop: Expectation System}, let $N$ be a Hawkes process with past before $0$ given by $N_-$  satisfying either $\left(\mathcal{A}_{N_-}^{1}\right)$ or $\left(\mathcal{A}_{N_-}^{2}\right)$ 
and with intensity on $\R_+$ given by
$$
\lambda(t,\mathcal{F}_{t-}^{N})=\mu + \int_{-\infty}^{t-} h(t-x) N(\dd x),
$$
where $\mu$ is a positive real number and $h \in L^\infty(\R_+)$ is a non-negative function with support in $\R_+$ such that $\int h <1$. 

Then, the mean measure $u$ defined in Proposition~\ref{prop:Fubini:nu} satisfies Theorem~\ref{prop: Expectation System} and moreover its integral $v(t,s):=\int\limits_s^\infty u(t,d\sigma)$  is the unique solution of the system \eqref{eq:PDEPhi:v}--\eqref{eq:PDEPhi:bound:v}
where $v^{in}$ is the survival function of $-T_0$,  and where  $\Phi=\Phi_{\P_0}^{\mu,h} \in L^{\infty}(\R_+)$ is defined by 
\begin{equation}\label{monPhi}
\Phi_{\P_0}^{\mu,h} =\Phi_+^{\mu,h} +\Phi_{-,\P_0}^{h},
\end{equation}
where for all non negative $s, t$
\begin{equation}\label{eq:implicit:equation:Phi:+}
\Phi^{\mu,h}_{+}(t,s) = \mu \left( 1+ \int_{s\wedge t}^{t} (h(x)+L_{s}(x)) G_{s}(x) \dd x \right),
\end{equation}
 and where
under Assumption $\left(\mathcal{A}_{N_-}^{1}\right)$, 
\begin{equation}\label{eq:Phi:-:1pt}
\Phi^{h}_{-,\mathbb{P}_{0}}(t,s) =\frac{\int_{-\infty}^{0\wedge(t-s)} \left( h(t-t_{0}) + K_{s}(t,t_{0}) \right)  q(t,s,t_{0}) f_{0}(t_{0}) \dd t_{0}}{\int_{-\infty}^{0\wedge(t-s)} q(t,s,t_{0}) f_{0}(t_{0}) \dd t_{0}},
\end{equation}
or, under Assumption $\left(\mathcal{A}_{N_-}^{2}\right)$, 
\begin{equation}\label{eq:Phi:-:Poisson}
\Phi^{h}_{-,\mathbb{P}_{0}}(t,s) =\alpha \int_{-\infty}^{0\wedge(t-s)} \left( h(t-z) + K_{s}(t,z) \right) q(t,s,z) \dd z.
\end{equation}
In these formulae, $L_s,~G_s,~K_s$ and $q$ are given by Lemma \ref{unicityLG} and \eqref{Kqdef}.
Moreover
\begin{equation}\label{rhoPhi}
\forall\, s\geq 0, \quad \int_s^{+\infty} 
\rho_{\lambda,\mathbb{P}_{0}}
(t,x)u(t,\dd x)=  \Phi^{\mu,h}_{\mathbb{P}_{0}}(t,s) \int_s^{+\infty}u(t,\dd x).
\end{equation}
\end{prop}
The proof is included in Appendix~\ref{app:Hawkes}. Proposition~\ref{prop:conditional:intensity:hawkes1}  gives a purely analytical definition for $v,$ and thus for $u,$ in two specific cases, namely $\left(\mathcal{A}_{N_-}^{1}\right)$ or $\left(\mathcal{A}_{N_-}^{2}\right)$.   In the general case,  treated in Appendix~\ref{ApExamples} (Proposition~\ref{prop:conditional:intensity:hawkes}),  there remains a dependence with respect to the initial condition $\mathbb{P}_{0}$, \emph{via} the function $\Phi^{h}_{-,\mathbb{P}_{0}}$.

\begin{remark}Contrarily to the general result in Theorem~\ref{prop: Expectation System}, Proposition~\ref{prop:conditional:intensity:hawkes1} focuses on the equation satisfied by $v(\dd t,s)=\int_s^{+\infty} u(\dd t,\dd x)$ because in Equation~\eqref{eq:PDEPhi:v} the function parameter $\Phi=\Phi_{\P_0}^{\mu,h}$ may be defined independently of the definitions of $v$ or $u,$ which is \emph{not} the case for the rate $\rho_{\lambda,\mathbb{P}_{0}}$ appearing in Equation~\eqref{eq:edp:expect}. Thus, it is possible to depart from the system of equations defining $v,$ study it, prove existence, uniqueness and regularity for $v$ under some assumptions on the initial distribution $u^{in}$ as well as on the birth function $h,$ and then deduce regularity or asymptotic properties for $u$ without any previous knowledge on the underlying process. 
\\ 
In Sections~\ref{sec:renewal} and~\ref{homMar}, we were able to use the PDE formulation to prove that the distribution of the ages $u$ has a density. Here, since we only obtain a closed formula for $v$ and not for $u$, we would need to derive Equation~
\eqref{eq:PDEPhi:v} in $s$ to obtain a similar result, so that we need to prove more regularity on $\Phi_{\P_0}^{\mu,h}$. Such regularity for $\Phi_{\P_0}^{\mu,h}$ is not obvious since it depends strongly on the assumptions on $N_-$. This paves the way for future research, where the PDE formulation would provide regularity on the distribution of the ages, as done above for renewal and Wold processes.
\end{remark}

\begin{remark} These two cases $\left(\mathcal{A}_{N_-}^{1}\right)$ and $\left(\mathcal{A}_{N_-}^{2}\right)$  highlight the dependence with respect to all the past before time $0$ (i.e. $\mathbb{P}_0$) and not only the initial condition (i.e. the age at time $0$). In fact, they can give the same initial condition $u^{in}:$ for instance, $\left(\mathcal{A}_{N_-}^{1}\right)$ with $-T_{0}$  exponentially distributed with parameter $\alpha>0$ gives the same law for $-T_0$ as $\left(\mathcal{A}_{N_-}^{2}\right)$ with parameter $\alpha$. However, if we fix some non-negative real number $s$, one can show that $\Phi^{h}_{-,\mathbb{P}_{0}}(0,s)$ is different in those two cases. 
It is clear from the definitions that for every real number $z$, $q(0,s,z)=1$ and $K_{s}(0,z)=0$.
%
 Thus, in the first case,
$$
\Phi^{h}_{-,\mathbb{P}_{0}}(0,s)= \frac{\int_{-\infty}^{-s} h(-t_{0})   \alpha e^{\alpha t_{0}} \dd t_{0}}{\int_{-\infty}^{-s} \alpha e^{\alpha t_{0}} \dd t_{0}} = \frac{\int_{s}^{\infty}h(z)\alpha e^{-\alpha z} \dd z}{\int_{s}^{\infty}\alpha e^{-\alpha z}\dd z},
$$
while in the second case,
$
\Phi^{h}_{-,\mathbb{P}_{0}}(0,s) = \alpha\int_{-\infty}^{-s} h(-z) \dd z = \alpha\int_{s}^{\infty} h(w) \dd w.
$
Therefore $\Phi^{h}_{-,\mathbb{P}_{0}}$ clearly depends on $\mathbb{P}_{0}$ and not just on the distribution of the last point before 0, and so is $\rho_{\lambda, \mathbb{P}_0}$. \end{remark}


{\begin{remark}
If we  follow our first guest, $\rho_{\lambda,\P_0}$ would be either $\mu+\int_0^t h(t-x) u(\dd x,0)$ or 
$\mu+\int_{-\infty}^t h(t-x) u(\dd x,0)$. 
In particular, it would 
not depend on the age $s$. Therefore by \eqref{rhoPhi}, so 
would $\Phi_{\P_0}^{\mu,h}$. But for instance at time $t=0$, when $N_-$ is an homogeneous Poisson process of parameter $\alpha$, $\Phi_{\P_0}^{\mu,h}(0,s)= \mu + \alpha\int_s^{+\infty} h(w) \dd w$, which obviously depends on $s$. Therefore the intuition linking Hawkes processes and (PPS) does not apply. 
\end{remark}
}


\subsubsection{Linear Hawkes process with no past before time 0}

A classical framework in point processes theory is the case {in $\left(\mathcal{A}_{N_-}^1\right)$} where 
$T_{0}\rightarrow -\infty$, or equivalently, when $N$ has intensity
$\lambda(t,\mathcal{F}_{t-}^{N})=\mu +\int_{0}^{t-} h(t-x) N(\dd x)$.
The problem in this case is that the age at time $0$ is not finite. The age is only finite for times greater than the first spiking time $T_{1}$. 

Here again, the quantity $v(t,s)$ reveals more informative and easier to use: having the distribution of $T_0$ going to $-\infty$ means that $\mbox{Supp}(u^{in})$ goes to $+\infty$, so that the initial condition for $v$ 
tends to value uniformly $1$ for any $0\leq s<+\infty.$ If we can prove that the contribution of $\Phi^{h}_{-,\P_0}$ vanishes, the following system is a good candidate to be the limit system:
\begin{gather}\label{eq:edp:mass:creation}
\f{\p}{\partial {t}} v^\infty\left(t,s\right)+\f{\p}{\partial {s}}v^\infty\left(t,s\right) +\Phi_+^{\mu,h}\left(t,s\right)v^\infty\left(t,s\right)=0,\\
v^\infty\left(t,0\right)=1,\qquad v^\infty(0,s)=1,
\label{eq:edp:mass:creation:bound}
\end{gather}
where $\Phi^{\mu,h}_{+}$ is defined in Proposition~\ref{prop:conditional:intensity:hawkes1}. This leads us to the following proposition. 

\vspace{-0.2cm}
\begin{prop}\label{prop:limit:edp:support:infinity}
Under the assumptions and notations of Proposition \ref{prop:conditional:intensity:hawkes1}, consider for all $M\geq 0$, 
$v_{M}$ 
the unique solution of system~\eqref{eq:PDEPhi:v}-\eqref{eq:PDEPhi:bound:v} with 
$\Phi$ 
given by Proposition~\ref{prop:conditional:intensity:hawkes1}, case  $\left(\mathcal{A}_{N_-}^{1}\right)$, with $T_{0}$ uniformly distributed in $[-M-1,-M]$.
Then, as $M$ goes to infinity, $v_{M}$ converges uniformly on any set of the type $(0,T)\times(0,S)$ towards the unique solution $v^\infty$ of System \eqref{eq:edp:mass:creation}-\eqref{eq:edp:mass:creation:bound}. 
\end{prop}

\section*{Conclusion}
We present in this article a bridge between univariate point 
processes, that 
can model the behavior of one neuron through 
its spike 
train, and a deterministic age structured PDE introduced 
by Pakdaman, Perthame and Salort, named (PPS).  
More precisely Theorem \ref{prop: Expectation System} present
a PDE that is satisfied by the distribution $u$ of the age $s$ at time $t$, where the age represents the delay between time $t$ and the last spike before $t$. This is done in a very weak sense and some technical structure, namely $\left(\mathcal{P}_{Fubini}\right)$, is required. 

The main point is that the "firing rate" which is a deterministic quantity written as $p(s,X(t))$ in (PPS) becomes the conditional expectation of the intensity given the age at time $t$ in Theorem \ref{prop: Expectation System}.
This first makes clear that $p(s,X(t))$ should be interpreted as a hazard rate, which gives the probability that a neuron fires given that it has not fired yet. Next, it makes clearly rigorous several "easy guess" bridges between both set-ups  when the intensity only depends on the age. But it also explained why when the intensity has a more complex shape (Wold, Hawkes), this term can keep in particular the memory  of all that has happened before time~$0$. 

One of the main point of the present study is the Hawkes process, for which what was clearly expected was a legitimation of the term $X(t)$ in the firing rate $p(s,X(t))$ of (PPS), which 
models
 the synaptic integration. This is not the case, and the interlinked equations that have been found for the cumulative distribution function $v(t,\cdot)$ do not have a simple nor direct deterministic interpretation. However one should keep in mind that the present  bridge, in particular in the population wide approach, has been done for independent neurons.  This has been done to keep the complexity of the present work reasonable as a first step. But it is also quite obvious that interacting neurons cannot be independent. So one of the main question is: can we recover (PPS) as a limit with precisely a term of the form $X(t)$ if we consider multivariate Hawkes processes that really model interacting neurons ?

\vspace{-0.3cm}
\section*{Acknowledgment}
This research was partly supported by the European Research Council (ERC Starting Grant SKIPPERAD number 306321), by the french Agence Nationale de la Recherche (ANR 2011 BS01 010 01 projet Calibration) and by the interdisciplanary axis MTC-NSC of the University of Nice Sophia-Antipolis. MJC acknowledges support from the projects MTM2014-52056-P (Spain)  and P08-FQM-04267 from Junta de Andaluc\'ia (Spain).
We  warmly thank  Fran\c{c}ois Delarue for very fruitful discussions and ideas.
\appendix

\vspace{-0.2cm}
\section{Proofs linked with the PDE}

\vspace{-0.2cm}
\subsection{Proof of Proposition \ref{prop:Microscopic:System}}
\label{proofProp:Microscopic:System}
First, let us verify that  $U$ satisfies Equation~\eqref{upt}. For any $t> 0$,
$$
U(t,\dd s)=\sum_{i\geq0} \eta_{T_i}(t,\dd s) \mathds{1}_{0\leq t\leq T_{i+1}},
$$
by definition of $U$. Yet,
$\eta_{T_i}(t,\dd s) = \delta_{t-T_{i}}(\dd s) \mathds{1}_{t>T_{i}}$,
and the only $i\in \mathbb{N}$ such that $T_{i}<t\leq T_{i+1}$ is $i=N_{t-}$. So, for all $t> 0$,
$U(t,\dd s)= \delta_{t-T_{N_{t-}}}(\dd s) = \delta_{S_{t-}}(\dd s)$.

Secondly, let us verify that $U$ satisfies $\left(\mathcal{P}_{Fubini}\right)$. Let $\varphi\in \mathcal{M}_{c,b}(\R_+^2)$, 
and let $T$ be such that for all $t>T$, $\varphi^{(1)}_t=0$.
Then since $U(t,\dd s)=\sum_{i=0}^{+\infty} \eta_{T_i}(t,\dd s) \mathds{1}_{0\leq t\leq T_{i+1}}$,
\begin{multline*}
\left| \int_{\mathbb{R}_{+}}\left( \int_{\mathbb{R}_{+}} \varphi(t,s) U(t,\dd s)\right) \dd t \right| \leq   \int_{\mathbb{R}_{+}}\left( \int_{\mathbb{R}_{+}} |\varphi(t,s)| \sum_{i\geq 0} \eta_{T_i}(t,\dd s) \mathds{1}_{0\leq t\leq T_{i+1}}  \right) \dd t\\
= \sum_{i\geq 0}  \int_{\mathbb{R}_{+}} |\varphi(t,t-T_i)| \mathds{1}_{t> T_i} \mathds{1}_{0\leq t\leq T_{i+1}}   \dd t 
= \sum_{i\geq 0}  \int_{\max(0,T_i)}^{T_{i+1}} |\varphi(t,t-T_i)| \dd t\\
= {\int_0^{T_1} |\varphi(t,t-T_0)| +  \sum_{i/0< T_i<T} \int_{T_i}^{T_{i+1}} |\varphi(t,t-T_i)| \dd t.}
\end{multline*}
Since there is a finite number of points of $N$ between $0$ and 
$T$, 
on $\Omega$, this quantity  is finite and one can exchange $\sum_{i\geq 0} $ and $\int_{t=0}^{+\infty}\int_{s=0}^{+\infty}$. 
Therefore, since all the $\eta_{T_i}$ satisfy 
$\left(\mathcal{P}_{Fubini}\right)$ and 
$\varphi(t,s) \mathds{1}_{0\leq t\leq T_{i+1}}$ is in $\mathcal{M}_{c,b}(\R_+^2)$, so does $U$.

For the dynamics of $U$, similar computations lead for every 
$\varphi\in C^{\infty}_{c,b}(\mathbb{R_+}^2)$ to
$$
\int\varphi\left(t,s\right)U\left(\dd t,\dd s\right) = \sum_{i\geq 0}\int_{\max(0,-T_{i})}^{T_{i+1}-T_{i}}\varphi\left(s+T_{i},s\right)\dd s.$$
We also  have
\begin{align}
& \!\!\!\int\left(\f{\p}{\partial {t}}+\f{\p}{\partial {s}}\right)\varphi\left(t,s\right)U\left(\dd t,\dd s\right)  =  \sum_{i\geq 0}\int_{\max(0,-T_{i})}^{T_{i+1}-T_{i}}\left(\f{\p}{\partial {t}}+\f{\p}{\partial {s}}\right)\varphi\left(s+T_{i},s\right)\dd s\nonumber \\
&~~  =\sum_{i\geq 1} \left[\varphi\left(T_{i+1},T_{i+1}-T_{i}\right)-\varphi\left(T_{i},0\right)\right] + \varphi(T_{1},T_{1}-T_{0}) - \varphi(0,-T_{0}). \label{eq:Transport:entire}
\end{align}
It remains to express the term with
$\int_{x=0}^{\lambda\left(t,\mathcal{F}_{t-}^{N}\right)}\Pi\left(\dd t,\dd x\right) =\sum_{i\geq 0}\delta_{T_{i+1}}(\dd t)$, that is
\begin{align}
&\!\!\!\!\int\varphi\left(t,s\right)U\left(t,\dd s\right)\sum_{i\geq 0}\delta_{T_{i+1}}\left(\dd t\right)  =   \int \left(\int \varphi\left(t,s\right)U\left(t,\dd s\right) \right) \sum_{i\geq 0}\delta_{T_{i+1}}\left(\dd t\right)\nonumber \\
&\hspace{3cm}  =  \int \varphi\left(t,S_{t-}\right) \sum_{i\geq 0}\delta_{T_{i+1}}(\dd t)
  =  \sum_{i\geq 0}\varphi\left(T_{i+1},T_{i+1}-T_{i}\right),\label{eq:Lose:entire}
\end{align}
and, since $\int U\left(t,\dd s\right)=1$ for all $t>0$,
\begin{equation}
\int \int\varphi\left(t,0\right) U\left(t,\dd s\right) \sum_{i\geq 0}\delta_{T_{i+1}}(\dd t) =\sum_{i\geq 0}\varphi\left(T_{i+1}.0\right),\label{eq:Reinitialisation:entire}
\end{equation}
Identifying all the terms in the right-hand side of Equation~\eqref{eq:Transport:entire}, 
this lead to Equation~\eqref{eq:edp:micro:weak}, which is the weak formulation of System~\eqref{eq:edp:micro}--\eqref{eq:edp:micro:init}.

\vspace{-0.2cm}
\subsection{Proof of Proposition \ref{prop:Fubini:nu}}

Let $\varphi\in \mathcal{M}_{c,b}(\R_+^2)$, 
and let $T$ be such that for all $t>T$, $\varphi^{(1)}_t=0$.
Then,
\begin{equation}\label{eq:upt:bounded}
\int |\varphi(t,s)| U(t,\dd s) \leq || \varphi ||_{L^\infty} \mathds{1}_{0\leq t \leq T}
,
\end{equation}
since at any fixed time $t>0$, $\int U(t,\dd s) = 1$.
Therefore,
the expectation $\mathbb{E}\left[ \int \varphi(t,s) U(t,\dd s) \right]$ is well-defined and finite and so $u(t,.)$ is well-defined.

On the other hand, at any fixed age $s$,
\begin{eqnarray*}
\int |\varphi(t,s)| U(\dd t,s) & = & \sum_{i=0}^{\infty} |\varphi(s+T_{i},s)| \mathds{1}_{0\leq s \leq T_{i+1}-T_{i}}
\\
& = & 
\sum_{i\geq 0} |\varphi(s+T_{i},s)| \mathds{1}_{0 \leq s+T_{i} \leq T} \mathds{1}_{0\leq s \leq T_{i+1}-T_{i}},
\end{eqnarray*}
because for all $t>T$, $\varphi^{(1)}_t=0$. Then, one can deduce the following bound
\begin{multline*}
\int |\varphi(t,s)| U(\dd t,s) \\ 
\leq  |\varphi(s+T_{0},s)| \mathds{1}_{-T_{0}\leq s \leq T-T_{0}} \mathds{1}_{0\leq s \leq T_{1}-T_{0}}+ \sum_{i\geq 1} |\varphi(s+T_{i},s)| \mathds{1}_{0\leq s \leq T} \mathds{1}_{T_{i}\leq T}\\
 \leq ||\varphi||_{L^\infty} \left( \mathds{1}_{-T_{0}\leq s \leq T-T_{0}} + N_{T}  \mathds{1}_{0\leq s \leq T}\right) .
\end{multline*}
Since the intensity is $L_{loc}^{1}$ in expectation, 
$
\mathbb{E}\left[ N_{T} \right]= \mathbb{E}\left[ \int_{0}^{T} \lambda(t,\mathcal{F}_{t-}^{N}) \dd t \right]
\!\!
< 
\!\!
\infty
$
{and}
\begin{equation}\label{eq:ups:bounded}
\mathbb{E}\left[ \int |\varphi(t,s)| U(\dd t,s) \right] \leq ||\varphi||_{L^\infty} \left( \mathbb{E}\left[ \mathds{1}_{-T_{0}\leq s\leq T-T_{0}} \right] + \mathbb{E}\left[ N_{T} \right]\mathds{1}_{0\leq s \leq T} \right),
\end{equation}
so the expectation is well-defined and finite and so $u(\cdot,s)$ is well-defined.

Now, let us show $\left(\mathcal{P}_{Fubini}\right)$. 
{First} Equation~\eqref{eq:upt:bounded} implies
$$
\int\ \mathbb{E}\left[ \int |\varphi(t,s)| U(t,\dd s) \right] \dd t \leq T || \varphi ||_{L^\infty},
$$
and Fubini's theorem implies that the following integrals are well-defined and that the following equality holds,
\begin{equation}\label{eq:fubini:nut}
\int\ \mathbb{E}\left[ \int \varphi(t,s) U(t,\dd s) \right] \dd t = \mathbb{E}\left[ \int \int  \varphi(t,s) U(t,\dd s) \dd t \right].
\end{equation}
Secondly, Equation~\eqref{eq:ups:bounded} implies
$$
\int\ \mathbb{E}\left[ \int |\varphi(t,s)| U(\dd t,s) \right] \dd s \leq  || \varphi ||_{L^\infty} \left( T + T \mathbb{E}\left[ N_{T} \right] \right),
$$
by exchanging the integral with the expectation and Fubini's theorem implies that the following integrals are well-defined and that the following equality holds,
\begin{equation}\label{eq:fubini:nus}
\int\ \mathbb{E}\left[ \int \varphi(t,s) U(\dd t,s) \right] \dd s = \mathbb{E}\left[ \int \int  \varphi(t,s) U(\dd t,s) \dd s \right].
\end{equation}
Now, it only remains to use $\left(\mathcal{P}_{Fubini}\right)$ for $U$ to deduce that the right members of Equations~\eqref{eq:fubini:nut} and~\eqref{eq:fubini:nus} are equal. Moreover, $\left(\mathcal{P}_{Fubini}\right)$ for $U$ tells that these two quantities are equal to
$\mathbb{E}\left[ \int \int  \varphi(t,s) U(\dd t,\dd s) \right]$.
This 
{concludes the proof.}

\vspace{-0.2cm}
\subsection{Proof of Theorem \ref{prop: Expectation System}}
\label{sec:proof:thm:expectation:system}
Let 
$
\rho_{\lambda,\mathbb{P}_0}(t,s):=\liminf_{\varepsilon\downarrow 0} \frac{\mathbb{E}\left[ \lambda(t,\mathcal{F}_{t-}^{N}) \mathds{1}_{\left| S_{t-}-s \right|\leq \varepsilon} \right]}{\mathbb{P}\left( \left| S_{t-}-s \right|\leq \varepsilon \right)},
$
for every $t>0$ and $s\geq 0$.
Since $(\lambda(t,\mathcal{F}_{t-}^{N}))_{t>0}$ and $(S_{t-})_{t>0}$ 
are predictable processes, and a fortiori progressive processes 
(see page $9$ in \cite{Bre}), 
$\rho_{\lambda,\mathbb{P}_0}$ is a measurable function of $(t,s)$.

For every $t>0$, let 
$\mu_{t}$ {be} 
 the measure defined by $\mu_{t}(A) = \mathbb{E}\left[ \lambda(t,\mathcal{F}_{t-}^{N}) \mathds{1}_{A}(S_{t-}) \right]$ for 
all measurable set $A$. 
{Since} Assumption $\left(\mathcal{A}^{\mathbb{L}^1,exp}_{\lambda,loc}\right)$
implies {that}  $\dd t$-a.e. 
$\esp{\lambda(t,\mathcal{F}_{t-}^{N})}<\infty$
and 
since $u(t,\dd s)$ is  the distribution of $S_{t-}$,
$\mu_{t}$  is absolutely continuous with respect to $u(t,\dd s)$ for $\dd t$-almost every $t$.

Let $f_{t}$ denote the Radon Nikodym derivative of $\mu_{t}$ with respect to $u(t,\dd s)$. For $u(t,\dd s)$-a.e. 
$s$, $f_{t}(s)= \mathbb{E}\left[ \lambda(t,\mathcal{F}_{t-}^{N}) 
\middle|\, S_{t-}=s \right]$ by definition of the conditional expectation. 
Moreover, a Theorem of Besicovitch \cite[Corollary 2.14]{mattila_1999} claims that for $u(t,\dd s)$-a.e. $s$, $f_{t}(s)= \rho_{\lambda,\mathbb{P}_0}(t,s)$.
Hence, {the equality}$\rho_{\lambda,\mathbb{P}_0}(t,s)= \mathbb{E}\left[ \lambda(t,\mathcal{F}_{t-}^{N}) \middle|\, S_{t-}=s \right]$ {holds}
$u(t,\dd s)\dd t=u(\dd t,\dd s)$-almost everywhere.

Next, in order to use $\left(\mathcal{P}_{Fubini}\right)$, let us note that for any $T,K>0$, 
\begin{equation}\label{eq:rho:bounded:measurable}
\rho_{\lambda,\mathbb{P}_{0}}^{K,T}: (t,s)\mapsto \left(\rho_{\lambda,\mathbb{P}_0}(t,s)\wedge K\right) {\bf 1}_{0\leq t \leq T}  \in \mathcal{M}_{c,b}(\R_+^2)
\end{equation}
Hence, $\int \int \rho_{\lambda,\mathbb{P}_{0}}^{K,T}(t,s) u(\dd t,\dd s) = \int \left( \int \rho_{\lambda,\mathbb{P}_{0}}^{K,T}(t,s) u(t,\dd s)\right)  \dd t $
which is always upper bounded by
$\int_0^T\left(\int \rho_{\lambda,\mathbb{P}_0}(t,s) u(t,\dd s)\right) \dd t=\int_0^T \mu_t(\R_+) \dd t = \int_0^T \mathbb{E}\left[ \lambda(t,\mathcal{F}_{t-}^{N})\right] 
\dd t<\infty.$

Letting $K\to \infty$, one has that $\int_0^T \int \rho_{\lambda,\mathbb{P}_0}(t,s)u(\dd t,\dd s)$ is finite for all $T>0$.

Once $\rho_{\lambda,\mathbb{P}_0}$ correctly defined, the proof of Theorem~\ref{prop: Expectation System} is  a direct consequence of Proposition \ref{prop:Microscopic:System}. 

More precisely, let us show that~\eqref{eq:edp:micro:weak} implies~\eqref{eq:edp:expect:weak}.
Taking the expectation of~\eqref{eq:edp:micro:weak} gives that for all $\varphi\in C^{\infty}_{c,b}(\R_+^2)$,
\begin{multline}\label{eq:edp:weak:expect-1}
\mathbb{E}\left[\int \left[\varphi\left(t,s\right)-\varphi(t,0)\right]
 \left(\int_{x=0}^{\lambda\left(t,\mathcal{F}_{t-}^{N}\right)} \Pi\left(\dd t,\dd x\right)\right) U\left(t,\dd s\right)\right] -\int\varphi\left(0,s\right)u^{in}\left(\dd s\right)\\
 -\int\left(\partial_{t}+\partial_{s}\right) \varphi\left(t,s\right)u\left(\dd t,\dd s\right) =0.
\end{multline}

Let us denote $\psi(t,s):=\varphi(t,s)-\varphi(t,0)$. Due to
Ogata's thinning construction,
$\left(\int_{x=0}^{\lambda\left(t,\mathcal{F}_{t-}^{N}\right)} \Pi\left(\dd t,\dd x\right)\right) = N(\dd t) \mathds{1}_{t>0}$
where $N$ is the point process constructed by thinning, and so,
\begin{equation}\label{eq:toto1}
\! \! \mathbb{E}\left[\int\psi\left(t,s\right) \left(\int_{x=0}^{\lambda\left(t,\mathcal{F}_{t-}^{N}\right)} \Pi\left(\dd t,\dd x\right)\right) U\left(t,\dd s\right)\right] =
\mathbb{E}\left[\int_{t>0} \psi\left(t,S_{t-}\right) N(\dd t) \right].
\end{equation}
But $\psi(t,S_{t-})$ is a  $(\mathcal{F}_{t}^{N})$-predictable process and
$$
\mathbb{E}\left[ \int_{t>0} | \psi(t,S_{t-}) | \lambda(t,\mathcal{F}_{t-}^{N}) \dd t  \right]\leq \|\psi\|_{L^\infty} \  \mathbb{E}\left[ \int_{0}^{T}\lambda(t,\mathcal{F}_{t-}^{N}) \dd t \right] < \infty,
$$
hence, using the martingale property of the predictable intensity,
\begin{equation}\label{eq:toto2}
\mathbb{E}\left[\int_{t>0} \psi\left(t,S_{t-}\right) N(\dd t) \right] = \mathbb{E}\left[\int_{t>0} \psi\left(t,S_{t-}\right) \lambda\left(t,\mathcal{F}_{t-}^{N}\right) \dd t \right].
\end{equation}
Moreover, thanks to Fubini's Theorem, the right-hand term is finite and equal to $\int \mathbb{E}[ \psi\left(t,S_{t-}\right) \lambda(t,\mathcal{F}_{t-}^{N}) ]\dd t$,  which can also be seen as
\begin{equation}\label{eq:toto3}
\int \mathbb{E}\left[ \psi\left(t,S_{t-}\right) \rho_{\lambda,\mathbb{P}_{0}}(t,S_{t-}) \right]\dd t = \int \psi(t,s) \rho_{\lambda,\mathbb{P}_{0}}(t,s) u(t,\dd s) \dd t.
\end{equation}
For all $K>0$, $\left( (t,s)\mapsto \psi(t,s) \left(\rho_{\lambda,\mathbb{P}_{0}}(t,s)\wedge K\right) \right) \in \mathcal{M}_{c,b}(\R_+^2)$ and, from $\left(\mathcal{P}_{Fubini}\right)$, it is clear that
$
\int \! \psi(t,s) \left(\rho_{\lambda,\mathbb{P}_{0}}(t,s)\wedge K\right) u(t,\dd s) \dd t = \int \! \psi(t,s) \left(\rho_{\lambda,\mathbb{P}_{0}}(t,s)\wedge K\right) u(\dd t,\dd s).$ 
Since one can always upper-bound this quantity in absolute value by $\|\psi\|_{L^\infty} \int_0^T \int_s \rho_{\lambda,\mathbb{P}_0}(t,s) u(\dd t,\dd s)$, this is finite. Letting $K\to \infty$ one can show that
\begin{equation}\label{eq:toto4}
\int \psi(t,s) \rho_{\lambda,\mathbb{P}_{0}}(t,s) u(t,\dd s) \dd t = \int \psi(t,s) \rho_{\lambda,\mathbb{P}_{0}}(t,s) u(\dd t,\dd s).
\end{equation}
Gathering~\eqref{eq:toto1}-\eqref{eq:toto4} with~\eqref{eq:edp:weak:expect-1} gives~\eqref{eq:edp:expect:weak}.

\vspace{-0.2cm}
\subsection{Proof of Corollary \ref{prop:Macroscopic:System}}
For all $i\in \mathbb{N}^{*}$, let us denote $N^{i}_{+}=N^{i}\cap (0,+\infty)$ and $N^{i}_{-}=N^{i}\cap \mathbb{R}_{-}$. Thanks to Proposition~\ref{prop:inversion:theorem}, the processes $N_{+}^{i}$ can be seen as constructed via thinning of independent Poisson processes on $\mathbb{R}_{+}^{2}$.
Let $(\Pi^{i})_{i \in \mathbb{N}}$ be the sequence of point measures associated to independent Poisson processes of intensity $1$ on $\mathbb{R}_{+}^{2}$ given by Proposition~\ref{prop:inversion:theorem}.
Let 
$T_0^{i}$ denote the {closest} point 
{to 0 in}
$N^{i}_{-}$. In particular, $(T_0^{i})_{i \in \mathbb{N}^{*}}$ is a sequence of i.i.d. random variables.

For each $i$, let $U^{i}$ denote the solution of the microscopic equation corresponding to $\Pi^{i}$ and $T_0^{i}$ as defined in Proposition~\ref{prop:Microscopic:System} by~\eqref{eq:def:Upsilon}. Using~\eqref{upt}, it is clear that $\sum_{i=1}^{n} \delta_{S_{t-}^{i}}(\dd s)=\sum_{i=1}^{n}  U^{i}(t,\dd s)$ for all $t>0$. Then, for every $\varphi \in C^{\infty}_{c,b}(\mathbb{R}_+^{2})$,
\begin{equation*}
\int \varphi(t,s) \left( \frac{1}{n}\sum_{i=1}^{n} \delta_{S_{t}^{i}}(\dd s) \right) = \frac{1}{n} \sum_{i=1}^{n} \int \varphi(t,s) U^{i}(t,\dd s).
\end{equation*}
The right-hand side is a sum $n$ i.i.d. random variables with mean $\int \varphi(t,s) u(t,\dd s)$, so~\eqref{eq:convergence:LLN} clearly follows from the law of large numbers.

\vspace{-0.2cm}
\section{Proofs linked with the various examples}
\label{ApExamples}

\vspace{-0.2cm}
\subsection{Renewal process} \label{ApRenewal}

\begin{prop}
With the notations of Section \ref{secPP}, let $N$ be a point process on $\mathbb{R}$, with predictable age process $(S_{t-})_{t> 0}$, 
such that $T_{0}=0$ a.s. 
The following statements are equivalent:
\begin{enumerate}[(i)]
\item $N_{+}=N\cap (0,+\infty)$ is a renewal process with ISI's distribution given by some density $\nu:\mathbb{R}_{+}\rightarrow \mathbb{R}_{+}$.
\item $N$ admits $\lambda(t,\mathcal{F}_{t-}^{N})=f(S_{t-})$ as an intensity on $(0,+\infty)$ and $\left( \lambda(t,\mathcal{F}_{t-}^{N}) \right)_{t> 0}$ satisfies  $\left(\mathcal{A}^{\mathbb{L}^1,a.s.}_{\lambda, loc}\right)$, for some $f:\mathbb{R}_{+}\rightarrow \mathbb{R}_{+}$.
\end{enumerate}
In such a case,  for all $x\geq 0$,
$f$ and $\nu$
{ satisfy}
\begin{align}
& {\bullet}~ \nu(x)=f(x)\exp(-\int_{0}^{x}f(y) dy) {\mbox{ with the convention } \exp(-\infty)=0,} \label{eq:relation:nu:f:renewal}\\
& {\bullet}~ f(x)=\frac{\nu(x)}{\int_{x}^{\infty} \nu(y) dy}{\quad \mbox{ if } \int_{x}^{\infty} \nu(y) dy\not =0, \mbox{ else } f(x)=0. }\label{eq:relation:f:nu:renewal}
\end{align}

\end{prop}
\begin{proof}
{For (ii) $\Rightarrow$ (i).} Since $T_{0}=0$ a.s., {Point (2) of Proposition \ref{prop:markovian:property} given later on for the general Wold case}
implies that the ISI's of $N$ forms a 
Markov
chain of order $0$ i.e. they are i.i.d. with density given by~\eqref{eq:relation:nu:f:renewal}.

{For (i) $\Rightarrow$ (ii).}
Let $x_{0}=\inf \{x\geq 0,  \int_{x}^{\infty} \nu(y) dy=0 \}$. It may be infinite. Let us define $f$ by~\eqref{eq:relation:f:nu:renewal} for every $0 \leq x< x_{0}$ and let $\tilde{N}$ be a point process on $\mathbb{R}$ such that $\tilde{N}_{-}=N_{-}$ and $\tilde{N}$ admits $\lambda(t,\mathcal{F}_{t-}^{\tilde{N}})=f(S^{\tilde{N}}_{t-})$ as an intensity on $(0,+\infty)$ where $(S^{\tilde{N}}_{t-})_{t>0}$ is the predictable age process associated to $\tilde{N}$. Applying {(ii) $\Rightarrow$ (i)}
to $\tilde{N}$ 
{gives}
that the ISI's of $\tilde{N}$ are i.i.d. with density given by
$$
\tilde{\nu}(x)=\frac{\nu(x)}{\int_{x}^{\infty}\nu(y) dy} \exp\left( -\int_{0}^{x} \frac{\nu(y)}{\int_{y}^{\infty}\nu(z) \dd z} dy \right),
$$
for every $0 \leq x< x_{0}$ and $\tilde{\nu}(x)=0$ for $x\geq x_{0}$.
It is clear that $\nu=\tilde{\nu}$ since the function 
$
x\mapsto \frac{1}{ \int_{x}^{\infty}\nu(y) dy} \exp \left( -\int_{0}^{x} \frac{\nu(y)}{\int_{y}^{\infty}\nu(z) \dd z} dy \right) 
$ 
is differentiable with derivative equal to $0$. 
{Since $N$ and $\tilde{N}$ are renewal processes with same density $\nu$ and same first point $T_0=0$, they have the same distribution.}
Since the intensity characterizes a point process, $N$ {also} admits $\lambda(t,\mathcal{F}_{t-}^{N})=f(S^{N}_{t-})$ as an intensity on $(0,+\infty)$. 
Moreover, since $N$ is a renewal process, it is non-explosive in finite time and so $\left( \lambda(t,\mathcal{F}_{t-}^{N}) \right)_{t> 0}$ satisfies $\left(\mathcal{A}^{\mathbb{L}^1,a.s.}_{\lambda, loc}\right)$.
\end{proof}

\vspace{-0.2cm}
\subsection{Generalized Wold processes}
\label{ApWold}
In this Section, we suppose that there exists $k\geq 0$ such that the underlying point process $N$ has intensity
\begin{equation}\label{eqwold}
\lambda\left(t,\mathcal{F}_{t-}^{N} \right)= f(S_{t-},A^1_t,...,A^k_t),
\end{equation}
where $f$ is a function and the $A^{i}$'s are defined by Equation~\eqref{eq:successive:ages}.

\subsubsection{Markovian property and the resulting PDE}

{Let   $N$ be a point process of intensity given by \eqref{eqwold}. If $T_{-k}>-\infty$, its associated age process $(S_t)_t$ can be defined up to $t>T_{-k}$. Then let, for any integer $i\geq -k$, 
\begin{equation}\label{eqa}
\mathbb{A}_{i} = T_{i+1}-T_{i}=S_{T_{i+1}-}
\end{equation} and denote $(\mathcal{F}_{i}^{\mathbb{A}})_{i\geq -k}$ the natural filtration associated to $(\mathbb{A}_{i})_{i\geq - k}$.}

For any $t\geq 0$, and point process $\Pi$ on $\mathbb{R}_{+}^{2}$, let us denote $\Pi_{\geq t}$ (resp. $\Pi_{>t}$) the restriction to $\mathbb{R}_{+}^{2}$ (resp. $(0,+\infty)\times \mathbb{R}_{+}$) of the point process $\Pi$ shifted $t$ time units to the left on the first coordinate. That is, $\Pi_{\geq t}(C\times D)=\Pi((t+C)\times D)$ for all $C\in \mathcal{B}(\mathbb{R}_{+}), D\in \mathcal{B}(\mathbb{R}_{+})$ (resp. $C\in \mathcal{B}((0,+\infty))$).

\begin{prop}\label{prop:markovian:property}

Let consider $k$ a non-negative integer, 
$f$ some non negative function on $\R_+^{k+1}$
and
$N$ a generalized Wold process { of intensity given by \eqref{eqwold}}. 
Suppose that $\mathbb{P}_{0}$ is such that 
$\mathbb{P}_{0}(T_{-k}>-\infty) = 1$ 
and {that}
$\left( \lambda(t,\mathcal{F}_{t-}^{N}) \right)_{t> 0}$ satisfies  $\left(\mathcal{A}^{\mathbb{L}^1,a.s.}_{\lambda, loc}\right)$. Then,
\begin{enumerate}
\item {If $(X_{t})_{t\geq 0}=\left( (S_{t-}, A^1_t,...,A^k_t)\right)_{t\geq 0}$, then} for any 
finite non-negative stopping time $\tau$, 
$(X^{\tau}_{t})_{t{\geq}
 0}=(X_{t+\tau})_{t{\geq}
  0}$ is independent of $\mathcal{F}_{\tau-}^{N}$ given $X_{\tau}$.
\item the process $(\mathbb{A}_{i})_{i\geq 1}$ {given by \eqref{eqa}} forms a Markov chain of order $k$ with transition measure given by
\begin{equation}\label{eq:relation:nu:f:generalized:wold}
\nu(\dd x,y_{1},...,y_{k})= f(x,y_{1},...,y_{k}) \exp\left( -\int_{0}^{x} f(z,y_{1},...,y_{k}) \dd z \right) \dd x.
\end{equation}
If $T_0=0$ a.s., this holds for $(\mathbb{A}_{i})_{i\geq 0}$.
\end{enumerate}
 If $f$ is continuous then {$\mathcal{G}$}, the infinitesimal generator of $(X_{t})_{t\geq 0}$, is given by
\begin{multline}\label{eq:infinitesimal:generator:wold}
{\forall \phi \in C^{1}(\mathbb{R}_{+}^{k+1})}, \quad (\mathcal{G}\phi)(s,a_{1},... ,a_{k})=  \\
\f{\p}{\p s}\phi(s,a_{1},... ,a_{k}) 
+ f(s,a_{1},... ,a_{k})\left( \phi(0,s,a_{1},... ,a_{k-1}) - \phi(s,a_{1},... ,a_{k}) \right).
\end{multline}
\end{prop}

\begin{proof}
First, let us show the first point of the Proposition. Let $\Pi$ be such that $N$ is the process resulting of Ogata's thinning with Poisson measure $\Pi$. The existence of such a measure is assured by Proposition~\ref{prop:inversion:theorem}.
We show that for any finite stopping time $\tau$, the process $(X^{\tau}_{t})_{t\geq 0}$ can be expressed as a function of $X_{\tau}$ and $\Pi_{\geq \tau}$ which is the restriction to $\mathbb{R}_{+}^{2}$ of the Poisson process $\Pi$ shifted $\tau$ time units to the left on the first coordinate.
Let $e_{1}=(1,0,\dots,0)\in \mathbb{R}^{k+1}$. For all $t\geq 0$, let $Y_{t}=X_{\tau}+te_{1}$ and define
$$
R_{0}=\inf \left\{ t\geq 0 ,\, \int_{[0,t]}\int_{x=0}^{f(Y_{w})} \Pi_{\geq \tau}(\dd w,\dd x) =1 \right\}.
$$
Note that $R_{0}$ may be null, in particular when $\tau$ is a jumping time of the underlying point process $N$.
It is easy to check that $R_{0}$ can be expressed as a measurable function of $X_{\tau}$ and $\Pi_{\geq \tau}$. Moreover, it is clear that $X^{\tau}_{t\wedge R_{0}}=Y_{t\wedge R_{0}}$ for all $t\geq 0$. So, $R_{0}$ can be seen as the delay until the first point of the underlying process $N$ after time $\tau$. Suppose that $R_{p}$, {the delay until the $(p+1)$th point,}  is constructed for some $p\geq 0$ and let us show how $R_{p+1}$ can be constructed. For $t\geq R_{p}$, let $Z_{t}=\theta(X^{\tau}_{R_{p}})+te_{1}$, where $\theta:(x_{1},\dots,x_{k+1})\mapsto(0,x_{1},\dots,x_{k})$ is a right shift operator modelling the dynamics described by~\eqref{eq:dynamic:successive:ages}. Let us define
\begin{equation}\label{eq:def:R:p}
R_{p+1}=\inf \left\{ t> R_{p} ,\, \int_{(R_{p},R_{p}+t]}\int_{x=0}^{f(Z_{w})} \Pi_{\geq \tau}(\dd w,\dd x) =1 \right\}.
\end{equation}
Note that for any $p\geq 0$, $R_{p+1}$ cannot be null. It is coherent with the fact that the counting process $(N_{t})_{t>0}$ only admits jumps with height $1$. It is easy to check that $R_{p+1}$ can be expressed as a measurable function of $\theta(X^{\tau}_{R_{p}})$ and $\Pi_{> \tau+R_{p}}$. 
It is  {also} clear that $X^{\tau}_{t\wedge R_{p+1}}=Z_{t\wedge R_{p+1}}$ for all $t\geq R_{p}$. So, $R_{p+1}$ can be seen as the delay until the $(p+2)$th point of the 
process $N$ after time $\tau$. By induction, $X^{\tau}_{R_{p}}$ can be expressed as a function of $X_{\tau}$ and $\Pi_{\geq \tau}$, 
{ and this holds} for $R_{p+1}$ and $X^{\tau}_{R_{p+1}}$ {too}.

To conclude, 
remark that the process $(X^{\tau}_{t})_{t\geq 0}$ is a measurable function of $X_{\tau}$ and all the 
$R_{p}$'s for $p\geq 0$.
Thanks to the independence of the Poisson measure $\Pi$, $\mathcal{F}_{\tau-}^{N}$ is independent of $\Pi_{\geq \tau}$. Then, since $(X^{\tau}_{t})_{t\geq 0}$ is a function of $X_{\tau}$ and $\Pi_{\geq \tau}$, $(X^{\tau}_{t})_{t\geq 0}$ is independent of $\mathcal{F}_{\tau-}^{N}$ given $X_{\tau}$ which concludes the first point.

For { Point (2), fix} 
 $i\geq 1$ 
{and apply Point (1)} 
  with $\tau=T_{i}$.  It appears that in this case, $R_{0}=0$ and $R_{1}=\mathbb{A}_{i}$. Moreover, $R_{1}=\mathbb{A}_{i}$ can be expressed as a function of $\theta(X_{\tau})$ and $\Pi_{> \tau}$. 
However, $\theta(X_{\tau})=(0,\mathbb{A}_{i-1},\dots,\mathbb{A}_{i-k})$ and $\mathcal{F}_{i-1}^{\mathbb{A}}\subset \mathcal{F}_{T_{i}}^{N}$. Since $\tau=T_{i}$, $\Pi_{> \tau}$ is independent of $\mathcal{F}_{T_{i}}^{N}$ and so $\mathbb{A}_{i}$ is independent of $\mathcal{F}_{i-1}^{\mathbb{A}}$ given $(\mathbb{A}_{i-1},\dots,\mathbb{A}_{i-k})$. That is, $(\mathbb{A}_{i})_{i\geq 1}$ forms a Markov chain of order $k$. 

Note that if $T_{0}=0$ a.s. (in particular it is non-negative), then one can use the previous argumentation with $\tau=0$ and conclude that the Markov chain starts one time step 
{earlier,} i.e. $(\mathbb{A}_{i})_{i\geq 0}$ forms a Markov chain of order $k$. 

{For \eqref{eq:relation:nu:f:generalized:wold},}
$R_{1}=\mathbb{A}_{i}$, defined by~\eqref{eq:def:R:p}, has the same distribution as the first point of a Poisson process with intensity $\lambda(t)=f(t,\mathbb{A}_{i-1},\dots ,\mathbb{A}_{i-k})$ thanks to the thinning Theorem. Hence, the transition measure of $(\mathbb{A}_{i})_{i\geq 1}$ is given by~\eqref{eq:relation:nu:f:generalized:wold}.

Now that $(X_{t})_{t\geq 0}$ is Markovian, one can compute its infinitesimal generator. Suppose that $f$ is continuous and let $\phi \in C_{{b}}^{1}(\mathbb{R}_{+}^{k+1})$, The generator of $(X_{t})_{t\geq 0}$ is defined by
$
\mathcal{G}\phi(s,a_{1},\dots,a_{k})=\lim_{h\rightarrow0^{+}}\frac{P_{h}-\text{Id}}{h}\phi,
$ 
where
\begin{eqnarray*}
P_{h}\phi\left(s,a_{1},\dots,a_{k}\right) & = & \mathbb{E}\left[\phi\left(X_{h}\right)\middle| X_{0}=(s,a_{1},\dots,a_{k})\right]\\
 & = & \mathbb{E}\left[\phi\left(X_{h}\right) \mathds{1}_{\{ N([0,h])=0 \}} \middle| X_{0}=(s,a_{1},\dots,a_{k})\right]\\
 &  & +\mathbb{E}\left[\phi\left(X_{h}\right) \mathds{1}_{\{ N([0,h])>0 \}} \middle| X_{0}=(s,a_{1},\dots,a_{k})\right]\\
 &=& E_{0} + E_{>0}.
\end{eqnarray*}
The case with no jump is easy to compute,
\begin{equation}\label{eq:E0}
E_{0}=\phi\left(s+h,a_{1},\dots,a_{k}\right)\left(1- f\left(s,a_{1},\dots,a_{k}\right)h\right) +o(h),
\end{equation}
thanks to the continuity of $f$.
When $h$ is small, the probability to have more than two jumps in $[0,h]$ is a $o(h)$, so the second case can be reduced to the case with exactly one random jump (namely $T$),
\begin{eqnarray}
E_{>0} &= & \mathbb{E}\left[\phi\left(X_{h}\right) \mathds{1}_{\{ N([0,h])=1 \}} \middle| X_{0}=(s,a_{1},\dots,a_{k})\right] + o(h) \nonumber\\
 & = & \mathbb{E}\left[\phi\left(\theta(X_{0}+T) +(h-T)e_{1}\right) \mathds{1}_{\{ N\cap[0,h]=\{T\} \}} \middle| X_{0}=(s,a_{1},\dots,a_{k})\right] + o(h) \nonumber\\
& = & \mathbb{E}\left[ \left( \phi\left(0,s,a_{1},\dots,a_{k-1}\right) +o({1})\right) \mathds{1}_{\{ N\cap[0,h]=\{T\} \}} \middle| X_{0}=(s,a_{1},\dots,a_{k})\right] + o(h) \nonumber\\
& = & \phi\left(0,s,a_{1}\dots,a_{k-1}\right) \left(f\left(s,a_{1},\dots,a_{k}\right)h\right) +o\left(h\right), \label{eq:E>0}
\end{eqnarray}
thanks to the continuity of $\phi$ and $f$. 
Gathering~\eqref{eq:E0} and~\eqref{eq:E>0} with the definition of the generator gives~\eqref{eq:infinitesimal:generator:wold}.
\end{proof}

\subsubsection{Sketch of proof of Proposition~\ref{prop: Expectation System:k}}
\label{sec:generalized:systems}
Let $N$ be the point process construct by Ogata's thinning of the Poisson process $\Pi$ and $U_k$ be as defined in Proposition~\ref{prop: Expectation System:k}. By an easy generalisation of Proposition \ref{prop:Microscopic:System}, one can prove that on the event $\Omega$ of probability 1, 
 where Ogata's thinning is well defined, and where $T_0<0$, $U_{k}$ satisfies $\left(\mathcal{P}_{Fubini}\right)$,  
~\eqref{upt:k} and on $\R_+\times \R_+^{k+1}$, the following system in the weak sense
\begin{gather*}
\!\! \left( \f{\p}{\p t}+\f{\p}{\p s}\right) U_k\left(\dd t,\dd s,\dd {\bf a}\right) + \left(\int_{x=0}^{f(s,a_{1},...,a_{k})}\Pi\left(\dd t,\dd x\right)\right) U_k\left(t,\dd s,\dd {\bf a}\right)=0,\\
U_k\left(\dd t,0,\dd s,\dd a_{1},... ,\dd a_{k-1}\right)=\int_{a_{k}\in \mathbb{R}} \left(\int_{x=0}^{f(s,a_{1},...,a_{k})}\Pi\left(\dd t,\dd x\right)\right) U_k\left(t,\dd s,\dd {\bf a}\right),
\end{gather*}
with $\dd {\bf a}=\dd a_{1}\times...\times \dd a_{k}$ and  initial condition 
$U^{in}=\delta_{(-T_{0},A^{1}_{0},\dots,A^{k}_{0})}$.

Similarly to Proposition \ref{prop:Fubini:nu}, one can also prove that 
 for any test function $\varphi$ in $\mathcal{M}_{c,b}(\mathbb{R}_+^{k+2})$,
$\mathbb{E}\left[ \int \varphi(t,s,{\bf a}) U_k(t,\dd s,\dd {\bf a}) \right]$ and $\mathbb{E}\left[ \int \varphi(t,s,{\bf a}) U_k(\dd t,s,\dd {\bf a}) \right]$ are finite
and one can define $u_{k}(t,\dd s,\dd {\bf a})$ and $u_{k}(\dd t,s,\dd {\bf a})$ by,  for all 
$\varphi$ in $\mathcal{M}_{c,b}(\mathbb{R}_+^{k+2})$, 
\begin{equation*}
\int \varphi(t,s,{\bf a})u_{k}(t,\dd s,\dd {\bf a})=\mathbb{E}\left[ \int \varphi(t,s,{\bf a})U_{k}(t,\dd s,\dd {\bf a}) \right],
\end{equation*}
for all $t\geq 0$, and
\begin{equation*}
\int \varphi(t,s,{\bf a})u_{k}(\dd t,s,\dd {\bf a})=\mathbb{E}\left[ \int \varphi(t,s,{\bf a})U_{k}(\dd t,s,\dd {\bf a}) \right],
\end{equation*}
for all  $s\geq 0$. Moreover, $u_{k}(t,\dd s,\dd {\bf a})$ and $u_{k}(\dd t,s,\dd {\bf a})$ satisfy  $\left(\mathcal{P}_{Fubini}\right)$ and one can define $u_{k}(\dd t,\dd s,\dd {\bf a})=u_{k}(t,\dd s,\dd {\bf a}) \dd t=u_{k}(\dd t,s,\dd {\bf a}) \dd s$ on $\R_+^2$, such that for any test function $\varphi$ in $\mathcal{M}_{c,b}(\mathbb{R}_+^{k+2})$,
\begin{equation*}
\int \varphi(t,s,{\bf a})u_{k}(\dd t,\dd s,\dd {\bf a})=\mathbb{E}\left[ \int \varphi(t,s,{\bf a})U_{k}(\dd t,\dd s,\dd {\bf a}) \right],
\end{equation*}
quantity which is finite.
The end of the proof is completely analogous to the one of Theorem \ref{prop: Expectation System}.

\vspace{-0.2cm}
\subsection{Linear Hawkes  processes}
\label{app:Hawkes}

\vspace{-0.2cm}
\subsubsection{Cluster decomposition \label{secClus}}

\begin{prop}\label{prop:cluster}
Let $g$  be a non negative $L^1_{loc}(\R_+)$ function and $h$ a non negative $L^1(\R_+)$ function  such that $\|h\|_1<1$. Then  the branching point process $N$ is defined as $\cup_{k=0}^\infty N_k$ the set of all the points in all generations constructed as follows{:}
\begin{itemize}
\item {A}ncestral points {are} $N_{anc}$ distributed as a Poisson process of intensity $g$; $N_0:=N_{anc}$ can be seen as the points of generation $0$.
\item Conditionally to $N_{anc}$, each ancestor $a\in N_{anc}$  gives birth, independently of anything else, to children points $N_{1,a}$ according to a Poisson process of intensity $h(.-a)$; $N_1=\cup_{a\in N_{anc}} N_{1,a}$ forms the first generation points.
\item[]\hspace{-1cm} Then the construction is recursive in $k$, the number of generations:
\item Denoting $N_k$ the set of  points in generation $k$, then conditionally to $N_k$, each point $x\in N_k$ gives birth, independently of anything else, to children points $N_{k+1,x}$ according to a Poisson process of intensity $h(.-x)$; $N_{k+1}=\cup_{x\in N_{k}} N_{k+1,x}$ forms the  points of the $(k+1)$th generation.
\end{itemize}
This construction ends almost surely in every finite interval. Moreover the intensity of $N$ exists and is given by
$$\lambda(t,\mathcal{F}_{t-}^{N})= g(t) + \int_{0}^{t-} h(t-x) N(\dd x).$$
\end{prop}
This is the cluster representation of the Hawkes process. When $g\equiv \nu$, this has been proved in \cite{Hawkes_Oakes_1974}. However up to our knowledge this has not been written for a general function $g$.
\begin{proof}
First, let us fix some $A>0$. The process ends up almost surely in $[0,A]$ because there is a.s. a finite number of ancestors in $[0,A]$: if we consider the family of points attached to one particular ancestor, the number of points in each generation form a sub-critical Galton Watson process with reproduction distribution, a Poisson variable with mean $\int h <1$ and whose extinction is consequently almost sure. 

Next, to prove that $N$ has intensity
$$H(t)= g(t) + \int_{0}^{t-} h(t-x) N(\dd x),$$
we exhibit a particular thinning construction, where on one hand, $N$ is indeed a branching process by construction as defined by the proposition and, which, on the other hand, guarantees that Ogata's thinning project the points below $H(t)$. We can always assume that $h(0)=0$, since changing the intensity of Poisson process in the branching structure at one particular point has no impact. Hence $H(t)=g(t) + \int_{0}^{t} h(t-x) N(\dd x).$

The construction is recursive in the same way. Fix some realisation $\Pi$ of a Poisson process on $\R_+^2$.

For $N_{anc}$, project the points below the curve $t\to g(t)$ on $[0,A]$. By construction, $N_{anc}$ is a Poisson process of intensity $g(t)$ on $[0,A]$. Note that for the identification (see Theorem \ref{thm:Thinning:R2}) we just  need to do it on finite intervals and that the ancestors that may be born after time $A$ do not have any descendants in $[0,A]$, so we can discard them, since they do not appear in $H(t)$, for $t\leq A$.

Enumerate the points in $N_{anc}\cap [0,A]$ from $T_1$ to $T_{N_{0,\infty}}.$ 
\begin{itemize}
\item The children of $T_1$, $N_{1,T_1}$, are given by the projection of the points of $\Pi$ whose ordinates are in the strip $t\mapsto (g(t),g(t)+h(t-T_1)]$. As before, by the property of spatial independence of $\Pi$, this is a Poisson process of intensity $h(.-T_1)$ conditionally to $N_{anc}$.
\item Repeat until $T_{N_{0,\infty}}$, where $N_{1,T_{N_{0,\infty}}}$ are given by the projection of the points of $\Pi$ whose ordinates are in the strip $t\mapsto (g(t)+\sum_{i=1}^{N_{0,\infty}-1} h(t-T_i),g(t)+\sum_{i=1}^{N_{0,\infty}} h(t-T_i)]$. As before, by the property of independence of $\Pi$, this is a Poisson process of intensity $h(.-T_{N_{0,\infty}})$ conditionally to $N_{anc}$ and because the consecutive strips do not overlap, this process is completely independent of the previous processes $(N_{1,T_i})$'s that have been constructed.
\end{itemize}
Note that at the end of this first generation, $N_1=\cup_{T\in N_{anc}} N_{1,T}$ consists of the projection of  points of $\Pi$ in the strip $t\mapsto (g(t),g(t)+\sum_{i=1}^{N_{0,\infty}} h(t-T_i)]$. They therefore form a Poisson process of intensity $\sum_{i=1}^{N_{0,\infty}} h(t-T_i)=\int h(t-u) N_{anc}(du),$ conditionally to $N_{anc}.$

For generation $k+1$ replace in the previous construction $N_{anc}$ by $N_k$ and $g(t)$ by $g(t)+\sum_{j=0}^{k-1} \int h(t-u) dN_j(u)$. Once again we end up for each point $x$ in $N_k$ with  a process of children $N_{k+1,x}$ which is a Poisson process of intensity $h(t-x)$ conditionally to $N_k$ and which is totally independent of the other $N_{k+1,y}$'s.
Note also that as before, $N_{k+1}=\cup_{x\in N_k} N_{k+1,x}$ is a Poisson process of intensity $\int h(t-u) N_k(du)$, conditionally to $N_0,...,N_k$.

Hence we are indeed constructing a branching process as defined by the proposition. Because the underlying Galton Watson process ends almost surely, as shown before, it means that there exists a.s. one generation $N_{k^*}$ which will be completely empty and our recursive contruction ends up too.

The main point is to realize that at the end the points in $N=\cup_{k=0}^\infty N_k$ are exactly the projection of the points in $\Pi$ that are below 
$$t\mapsto g(t)+\sum_{k=0}^\infty \int h(t-u) N_k(du)= g(t)+\sum_{k=0}^\infty \int_0^t h(t-u) N_k(du)$$
hence below
$$t\mapsto g(t)+\int_0^t h(t-u) N(du)=H(t).$$
Moreover $H(t)$ is $\mathcal{F}_t^N$ predictable.
Therefore by Theorem \ref{thm:Thinning:R2}, $N$ has intensity $H(t)$, which concludes the proof.
\end{proof}

A cluster process $N_c$, is a branching process, as defined before, which admits  intensity $\lambda(t,\mathcal{F}_{t-}^{N_{c}}) = h(t) + \int_{0}^{t-} h(t-z) N_{c}(\dd z)$. Its distribution  only depends on the function $h$. It corresponds to the family generated by one ancestor at time 0 in Proposition \ref{prop:cluster}. Therefore, by Proposition\ref{prop:cluster}, a Hawkes process with empty past ($N_-=\emptyset$)  of intensity $\lambda(t,\mathcal{F}_{t-}^{N}) = g(t) + \int_{0}^{t-} h(t-z) N(\dd z)$ can always be seen as the union of $N_{anc}$ and of all the $a+N_c^a$ for $a \in N_{anc}$ where the $N_c^a$ are i.i.d. cluster processes. 

For a Hawkes process $N$ with non empty past, $N_-$,  this is more technical. Let $N_{anc}$ be a Poisson process of intensity $g$ on $\R_+$ and $\left(N_c^V\right)_{V\in N_{anc}}$ be  a sequence of i.i.d. cluster processes associated to $h$. Let also 
\begin{equation}\label{eq:def:N>0}
N_{>0}= N_{anc}\cup \left( \bigcup_{V\in N_{anc}} V+N_{c}^{V} \right).
\end{equation}
As we prove below, this  represents the points in $N$ that do not depend on $N_-$. The points that are depending on $N_-$ are constructed as follows independently of $N_{>0}$.  Given $N_{-}$, let  $\left( N_{1}^{T} \right)_{T\in N_{-}}$ denote a sequence of independent Poisson processes with respective intensities $\lambda_{T}(v)=h(v-T) \mathds{1}_{(0,\infty)}(v)$. Then, given $N_{-}$ and $\left( N_{1}^{T} \right)_{T\in N_{-}}$, let $\left( N_{c}^{T,V} \right)_{V \in N_{1}^{T},T \in N_{-}}$ be a sequence of i.i.d. cluster processes associated to $h$. The points depending on the past $N_-$ are given by the following formula as proved in the next Proposition:
\begin{equation}\label{eq:def:Nleq0}
N_{\leq 0}= N_{-} \cup\left( \bigcup_{T\in N_{-}} N_{1}^{T} \cup \left( \bigcup_{V\in N_{1}^{T}} V+N_{c}^{T,V} \right)  \right) .
\end{equation}

\begin{prop}\label{prop:decomp:Hawkes}
Let $N=N_{\leq 0} \cup N_{>0}$, where $N_{>0}$ and $N_{\leq0}$ are given by \eqref{eq:def:N>0} and \eqref{eq:def:Nleq0}. Then $N$ is a linear Hawkes process with past given by $N_-$ and intensity on $(0,\infty)$ given by 
$\lambda(t,\mathcal{F}_{t-}^{N})= g(t) + \int_{-\infty}^{t-} h(t-x) N(\dd x)$,
where $g$ and $h$ are as in Proposition \ref{prop:cluster}.
\end{prop}
\begin{proof}
Proposition\ref{prop:cluster} yields that $N_{>0}$ has intensity
\begin{equation}\label{eq:intensity:Ngeq0}
\lambda_{N_{>0}}(t,\mathcal{F}_{t-}^{N_{>0}})=g(t) + \int_{0}^{t-} h(t-x) N_{>0}(\dd x),
\end{equation}
and that, given $N_{-}$, for any $T\in N_{-}$, $N_{H}^{T}=N_{1}^{T} \cup \left( \bigcup_{V\in N_{1}^{T}} V+N_{c}^{T,V} \right)$ has intensity
\begin{equation}\label{eq:intensity:NH}
\lambda_{N_{H}^{T}}(t,\mathcal{F}_{t-}^{N_{H}^{T}})=h(t-T) + \int_{0}^{t-} h(t-x) N_{H}^{T}(\dd x),
\end{equation}
Moreover, all these processes are independent given $N_{-}$. For any $t\geq 0$, one can note that
$
\mathcal{F}_{t}^{N_{\leq 0}} \subset \mathcal{G}_t:= \mathcal{F}^{N_{-}}_{0} \vee \left( \bigvee_{T\in N_{-}} \mathcal{F}_{t}^{N_{H}^{T}} \right)$,
and so $N_{\leq 0}$ has intensity
\begin{equation}\label{eq:intensity:Nleq0}
\lambda_{N_{\leq 0}}(t,\mathcal{G}_{t-}) = \sum_{T\in N_{-}} \lambda_{N_{H}^{T}}(t,\mathcal{F}_{t-}^{N_{H}^{T}}) = \int_{-\infty}^{t-}  h(t-x) N_{\leq 0}(\dd x)
\end{equation}
on $(0,+\infty)$. Since this last expression is $\mathcal{F}_{t}^{N_{\leq 0}}$-predictable, by page $27$ in \cite{Bre}, this is also $\lambda_{N_{\leq 0}}(t,\mathcal{F}_{t-}^{N_{\leq 0}}).$ Moreover, $N_{\leq 0}$ and $N_{>0}$ are independent by construction and, for any $t\geq 0$,
$
\mathcal{F}_{t}^{{N}}\subset \mathcal{F}_{t}^{N_{\leq 0}} \vee \mathcal{F}_{t}^{N_{>0}}$.
Hence, as before, $N$ has intensity on $(0,+\infty)$ given by
$$
\lambda(t,\mathcal{F}_{t-}^{{N}}) = \lambda(t,\mathcal{F}_{t-}^{N_{\leq 0}}) + \lambda(t,\mathcal{F}_{t-}^{N_{>0}}) = g(t) + \int_{-\infty}^{t-} h(t-x) {N}(\dd x).
$$
\end{proof}

\vspace{-0.2cm}
\subsubsection{A general result for linear Hawkes processes} 
The following proposition is a consequence of Theorem~\ref{prop: Expectation System} applied to Hawkes processes with general past $N_-.$ 
\begin{prop}\label{prop:conditional:intensity:hawkes}
Using the notations of Theorem~\ref{prop: Expectation System}, let $N$ be a Hawkes process with past before $0$ given by $N_-$ of distribution $\mathbb{P}_0$ and with intensity on $\R_+$ given by
$$
\lambda(t,\mathcal{F}_{t-}^{N})=\mu + \int_{-\infty}^{t-} h(t-x) N(dx),
$$
where $\mu$ is a positive real number and $h$ is a non-negative function with support in $\R_+$ such that $\int h <1$. 
Suppose that $\mathbb{P}_0$ is such that
\begin{equation}\label{eq:assumption:P0}
\sup_{t\geq 0} \mathbb{E}\left[ \int_{-\infty}^{0} h(t-x) N_{-}(dx) \right] <\infty.
\end{equation}
Then, the mean measure $u$ defined in Proposition~\ref{prop:Fubini:nu} satisfies Theorem~\ref{prop: Expectation System} and moreover its integral $v(t,s):=\int\limits_s^\infty u(t,d\sigma)$ is a solution of the system \eqref{eq:PDEPhi:v}--\eqref{eq:PDEPhi:bound:v} where $v^{in}$ is the survival function of $-T_0$,  and where  $\Phi=\Phi_{\P_0}^{\mu,h}$ is given by 
$
\Phi_{\P_0}^{\mu,h} =\Phi_+^{\mu,h} +\Phi_{-,\P_0}^{\mu,h},
$
with $\Phi_+^{\mu,h} $ given by \eqref{eq:implicit:equation:Phi:+}
and  $\Phi^{\mu,h}_{-,\mathbb{P}_{0}}$ given by, 
\begin{equation}
\forall\, s,t \geq 0, \quad \Phi^{\mu,h}_{-,\mathbb{P}_{0}}(t,s)=\mathbb{E}\left[ \int_{-\infty}^{t-} h(t-z) N_{\leq 0}(dz) \middle|\,N_{\leq 0}\left( [t-s,t) \right)=0 \right].
\end{equation}
Moreover, \eqref{rhoPhi} holds. 
\end{prop}

\vspace{-0.2cm}
\subsubsection{Proof of the general result of Proposition~\ref{prop:conditional:intensity:hawkes}}
Before proving Proposition~\ref{prop:conditional:intensity:hawkes}, we need some technical preliminaries.

Events of the type $\{S_{t-}\geq s\}$  are equivalent to the fact that the underlying process has no point between $t-s$ and $t$. Therefore, for any point process $N$ and any real numbers $t,s \geq 0$, let
\begin{equation}\label{eq:def:event:E}
\mathcal{E}_{t,s}(N) = \{N\cap [t-s,t) = \emptyset\}.
\end{equation}
Various sets $\mathcal{E}_{t,s}(N)$ are used in the sequel and the following lemma, whose proof is obvious and therefore omitted, is applied several times to those sets.
\begin{lem}\label{lem:conditional:expectation}
Let $Y$ be some random variable and $I(Y)$ some countable set of indices depending on $Y$. Suppose that $\left( X_{i} \right)_{i\in I(Y)}$ is a sequence of random variable{s} which are independent conditionally on $Y$.  Let $A(Y)$ be some event depending on $Y$ and $\forall\, j\in I(Y)$, $B_{j}=B_{j}(Y,X_{j})$ be some event depending on $Y$ and $X_{j}$. Then, for any $i\in I(Y)$, and for all sequence of measurable functions $(f_{i})_{i \in I(Y)}$ such that the following quantities exist,
$$
\mathbb{E}\left[ \sum_{i\in I(Y)} f_{i}(Y,X_{i}) \middle|\, A\#B \right]= \mathbb{E}\left[ \sum_{i\in I(Y)}  \mathbb{E}\left[ f_{i}(Y,X_{i}) \middle|\, Y, B_{i} \right] \middle|\, A\#B \right],
$$
where
$
\mathbb{E}\left[ f_{i}(Y,X_{i}) \middle|\, Y, B_{i} \right] = \frac{\mathbb{E}\left[ f_{i}(Y,X_{i}) \mathds{1}_{B_{i}} \middle|\, Y \right]}{\mathbb{P}\left( B_{i} \middle|\, Y \right)}
$
and $A\#B= A(Y) \cap \left( \bigcap_{j\in I(Y)} B_{j}\right) $.
\end{lem}

The following lemma is linked to Lemma~\ref{unicityLG}.
\begin{lem}\label{prop:implicit:formulas:L:G}
Let $N$ be a linear Hawkes process with no past before time $0$ (i.e. $N_{-}=\emptyset$) and intensity on $(0,\infty)$ given by
$
\lambda(t,\mathcal{F}_{t-}^{N})= g(t) + \int_{0}^{t-} h(t-x) N(\dd x),
$
where $g$ and $h$ are {as in Proposition \ref{prop:cluster} and let for any $x,s\geq 0$} 
$$
\left\{
\begin{aligned}
&L^{g,h}_{s}(x)=\mathbb{E}\left[ \int_{0}^{x} h(x-z) N(\dd z) \middle|\,
{\mathcal{E}_{x,s}(N)}\right]\\
&G^{g,h}_{s}(x) = 
{\mathbb{P}\left( \mathcal{E}_{x,s}(N)\right),}
\end{aligned}
\right.
$$
Then, for any $x,s\geq 0$,
\begin{equation}\label{eq:L:g:h:s}
L^{g,h}_{s}(x)= \int_{s\wedge x}^{x} \left(h\left(z\right)+ L^{h,h}_{s}(z) \right) G^{h,h}_{s}(z) g(x-z) \, \dd z,
\end{equation}
and 
\begin{equation}\label{eq:G:g:h:s}
\log(G^{g,h}_s(x))=\int_{0}^{(x-s)\vee 0} G^{h,h}_s(x-z) g(z) \dd z -\int_{0}^x g(z) \dd z.
\end{equation}
{In particular, $(L_s^{h,h},G_s^{h,h})$ is in $L^1\times L^\infty$ and is a solution of \eqref{eq:Gs}-\eqref{eq:Ls}.}
\end{lem}

\begin{proof}
The statement 
only depends on the distribution of $N$. Hence, thanks to Proposition~\ref{prop:decomp:Hawkes}, it is sufficient to consider $N=N_{anc} \cup \left( \cup_{V\in N_{anc}} V+ N_{c}^{V} \right)$.

Let us show~\eqref{eq:L:g:h:s}. First, let us write 
$
L^{g,h}_{s}(x)=\mathbb{E}\left[ \sum_{X\in N} h(x-X) \middle| \mathcal{E}_{x,s}(N) \right].
$ 
and note that $L^{g,h}_{s}\left(x\right)=0$ if $x\leq s$.
The following decomposition holds
$$
L^{g,h}_{s}(x)=\mathbb{E}\left[ \sum_{V\in N_{anc}}
\left( h(x-V) + \sum_{ W\in N_{c}^{V}} h(x-V-W)  \right) \middle| \mathcal{E}_{x,s}(N) \right].
$$
According to Lemma \ref{lem:conditional:expectation} and the following decomposition,
\begin{equation}\label{eq:decomposition:E:x:s}
\mathcal{E}_{x,s}(N) = \mathcal{E}_{x,s}(N_{anc}) \cap \left( \bigcap_{V\in N_{anc}} \mathcal{E}_{x-V,s}(N_{c}^{V}) \right),
\end{equation}
let us denote $Y=N_{anc}$, $X_{V}=N_{c}^{V}$ and $B_{V}=\mathcal{E}_{x-V,s}(N_{c}^{V})$ for all $V\in N_{anc}$.
Let us fix $V\in N_{anc}$ and compute the conditional expectation of the inner sum with respect to the filtration of $N_{anc}$ which is
\begin{eqnarray}
 \mathbb{E}\left[ \sum_{W\in N_{c}^{V}} h(x-V-W) \middle| Y, B_{V} \right] &=& \mathbb{E}\left[ \sum_{W\in N_{c}} h((x-V)-W) \middle|  \mathcal{E}_{x-V,s}(N_{c})  \right] \nonumber \\
& = & L^{h,h}_{s}(x-V), \label{eq:example:application:lemma}
\end{eqnarray}
since, conditionally on $N_{anc}$, $N_{c}^{V}$ has the same distribution as $N_{c}$ which is a linear Hawkes process with conditional intensity $\lambda(t,\mathcal{F}_{t-}^{N_{c}}) = h(t) + \int_{0}^{t-} h(t-z) N_{c}(\dd z)$.
Using the conditional independence of the cluster processes with respect to $N_{anc}$, one can apply Lemma \ref{lem:conditional:expectation} and deduce that
$$
L^{g,h}_{s}(x) =  \mathbb{E}\left[\sum_{V\in N_{anc}}  \left( h(x-V) + L^{h,h}_{s}(x-V) \right) \middle| \mathcal{E}_{x,s}(N)\right]\\
$$
The following argument is inspired by Moller \cite{Moller_Perfect_Hawkes}. For every $V\in N_{anc}$, we say that $V$ has mark $0$ if $V$ has no descendant or himself in $[x-s,x)$ and mark $1$ otherwise. Let us denote $N_{anc}^{0}$ the set of points with mark $0$ and $N_{anc}^{1}=N_{anc}\setminus N_{anc}^{0}$. For any $V\in N_{anc}$, we have
$
\mathbb{P}\left( V \in N_{anc}^{0} \middle| N_{anc} \right)= G^{h,h}_{s}(x-V) \mathds{1}_{[x-s,x)^{c}}(V),
$
and all the marks are chosen independently given $N_{anc}$.
Hence, $N_{anc}^{0}$ and $N_{anc}^{1}$ are independent Poisson processes and the intensity of $N_{anc}^{0}$ is given by $\lambda(v)=g(v)G^{h,h}_{s}(x-v) \mathds{1}_{[x-s,x)^{c}}(v)$. Moreover, the event { $\left\{ N_{anc}^{1}=\emptyset \right\}$ can be identified to 
$\mathcal{E}_{x,s}(N)$}
and
\begin{eqnarray*}
L^{g,h}_{s}(x)& =& \mathbb{E}\left[ \sum_{V\in N_{anc}^{0}}  \left( h(x-V) + L^{h,h}_{s}(x-V) \right)  \middle| N_{anc}^{1}=\emptyset \right] \nonumber \\
&=& \int_{-\infty}^{x-}\left(h\left(x-w\right)+ L^{h,h}_{s}(x-w) \right) g(w) G^{h,h}_{s}(x-w) \mathds{1}_{[x-s,x)^{c}}(w) \dd w  \\ 
&=&   \int_{0}^{(x-s)\vee 0} \left(h\left(x-w\right)+ L^{h,h}_{s}(x-w) \right) G^{h,h}_{s}(x-w) g(w) \, \dd w, \nonumber
\end{eqnarray*}
where we used the independence between the two Poisson processes. It suffices to substitute $w$ by $z=x-w$ in the integral to get the desired formula. {Since $G_s^{h,h}$ is bounded, it is obvious that $L_s^{h,h}$ is $L^1$.}

Then, let us show~\eqref{eq:G:g:h:s}. First note that if $x<0$, $G^{g,h}_s(x)=1$.
Next, following 
\eqref{eq:decomposition:E:x:s} one has
$G^{g,h}_s(x)=\esp{\mathds{1}_{\mathcal{E}_{x,s}(N_{anc})}  \prod_{X \in N_{anc}} \mathds{1}_{\mathcal{E}_{x-X,s}(N_{c}^{X}) } }.$
This is also
\begin{eqnarray*}
G^{g,h}_s(x)&=&\esp{\mathds{1}_{N_{anc} \cap [x-s,x)=\emptyset} \prod_{V\in N_{anc}\cap [x-s,x)^c} \mathds{1}_{\mathcal{E}_{x-V,s}(N_{c}^{V}) } },\\
&=&\esp{\mathds{1}_{N_{anc} \cap [x-s,x)=\emptyset} \prod_{V\in N_{anc} \cap [x-s,x)^c} G^{h,h}_s(x-V)},
\end{eqnarray*}
by conditioning with respect to $N_{anc}$. Since $ N_{anc} \cap [x-s,x)$ is independent of $ N_{anc} \cap [x-s,x)^c$, this gives
\begin{equation*}
G^{g,h}_s(x)=\exp(-\int_{x-s}^x g(z) \dd z)\esp{\exp \left(\int_{[x-s,x)^c} \log(G^{h,h}_s(x-z)) N_{anc}(\dd z)\right)}.
\end{equation*}
This leads to
$\log(G^{g,h}_s(x))=-\int_{x-s}^x g(z) \dd z + \int_{[x-s,x)^c} (G^{h,h}_s(x-z)-1) g(z) \dd z$, thanks to {Campbell's Theorem \cite{kingman}}.
Then,~\eqref{eq:G:g:h:s} clearly follows from the facts that if $z>x>0$ then $G^{h,h}_s(x-z)=1$ and $g(z)=0$ as soon as $z<0$.
\end{proof}
\paragraph{Proof of Lemma \ref{unicityLG}}
In turn, we use a Banach fixed point argument to prove that for all $s\geq 0$ there exists a unique couple $(L_s,G_s)\in L^1(\R_+)\times L^\infty (\R_+) $ solution to these equations. To do so, let us first study Equation~\eqref{eq:Gs} and define $T_{G,s}: L^\infty(\R_+)\to L^\infty(\R_+)$ by
$
T_{G,s}(f)(x):=\exp\biggl(\int_{0}^{(x-s)\vee 0} f(x-z) h(z) \dd z -\int_{0}^x h(z) \dd z\biggr).$ 
The right-hand side is well-defined since $h\in L^1$ and $f\in L^\infty.$ Moreover we have
$$T_{G,s}(f)(x)\leq e^{\Vert f\Vert_{L^\infty} \left( \int_{0}^{(x-s)\vee 0}  h(z) \dd z -\int_{0}^x h(z) \dd z \right)}
\leq e^{(\Vert f\Vert_{L^\infty} -1)\int_{0}^{(x-s)\vee 0}  h(z) \dd z}.$$
This shows that $T_{G,s}$ maps the ball of radius $1$ of $L^\infty$ into itself, and more precisely into the intersection of the positive cone and the ball. We distinguish two cases: 

\noindent
$-$ If $x<s$, then 
$T_{G,s}(f)(x)=\exp(-\int\limits_0^x h(z) \dd z)$ for any 
$f$,
thus, the unique fixed point is given by 
$G_s:x\mapsto \exp(-\int\limits_0^x h(z) \dd z),$ which does not depend on  $s>x.$ 

\noindent
$-$ And if $x>s$, 
 the functional $T_{G,s}$ is a $k-$contraction in $\{f\in L^\infty(\R_+), \Vert f\Vert_{L^\infty} \leq 1\},$ with $k\leq  \int\limits_{0}^{\infty}  h(z) \dd z<1$, 
by convexity of the exponential. More precisely, using that for all $x,y$, $|e^{x}-e^{y}|\leq e^{\max(x,y)}|x-y|$ we end up with, for $\Vert f\Vert, \Vert g\Vert_{L^\infty} \leq 1$,
\begin{eqnarray*}
\big\lvert T_{G,s}(f)(x)-T_{G,s}(g)(x)\big\rvert &\leq & 
e^{-\int\limits_{0}^x h(z)\dd z} e^{\int\limits_{0}^{x-s} h(z)\dd z} \Vert f-g\Vert_{L^\infty} \int_{0}^{(x-s)}  h(z) \dd z\\
&\leq & \Vert f-g\Vert_{L^\infty} \int_{\mathbb{R}_{+}}  h(z) \dd z.
\end{eqnarray*}

Hence there exists only one fixed point $G_s$ that we can identify with $G_s^{h,h}$ given in Proposition \ref{prop:implicit:formulas:L:G} and $G_s^{h,h}$ being a probability, $G_s$ takes values in $[0,1]$.

Analogously, we define the functional $T_{L,s}: L^1(\R_+)\to L^1(\R_+)$ by
$ 
T_{L,s}(f)(x):=\int_{s\wedge x}^{x} \left(h\left(z\right)+ f(z) \right) G_{s}(z) h(x-z) \, \dd z,$ 
and it is easy to check that $T_{L,s}$ is well-defined as well.
We similarly distinguish the two cases:

\noindent
$-$ If $x<s$, then the unique fixed point is given by $L_s(x)=0$.


\noindent
$-$
And if $x>s$,  
thus $T_{L,s}$ is a $k-$contraction with 
$k\leq \int\limits_0^\infty  h(y) dy<1 $ in $L^1((s,\infty))$ since $\Vert G_s\Vert_{L^\infty} \leq 1:$

$ 
\begin{array}{ll}\Vert T_{L,s}(f) - T_{L,s}(g) \Vert_{L^1}&=\int\limits_s^\infty \big\lvert \int\limits_s^x \bigl( f(z)-g(z) \bigr) G_s(z) h(x-z) \dd z \big\rvert \dd x \\
&\leq \Vert G_s\Vert_{L^\infty} \int\limits_s^\infty \int\limits_v^\infty \big\lvert f(z)-g(z) \big\rvert   h(x-z) \dd x\dd z
\\
&= \Vert G_s\Vert_{L^\infty}\Vert f-g\Vert_{L^1((s,\infty))} \int\limits_0^\infty  h(y) dy.
\end{array}
$ 

\noindent In the same way, there exists only one fixed point $L_s=L_s^{h,h}$ given by Proposition \ref{prop:implicit:formulas:L:G}. In particular $L_s(x\leq s) \equiv 0.$

Finally, as a consequence of Equation~\eqref{eq:Ls} we find that if $L_s$ is the 
unique fixed point of $T_{L,s}$, then
$ 
\Vert L_s \Vert_{L^1(\R_+)}\le \frac{(\int_0^\infty h(y) \ dy)^2}{
1-\int_0^\infty h(y) \ dy}
$ 
and therefore $L_s$ is uniformly bounded in $L^{1}$ with respect to $s$.

\begin{lem}\label{lem:assumptions:Hawkes}
Let $N$ be a linear Hawkes process with past before time $0$ given by $N_{-}$ and intensity on $(0,\infty)$ given by
$
\lambda(t,\mathcal{F}_{t-}^{N})=\mu + \int_{-\infty}^{t-} h(t-x) N(\dd x),
$
where $\mu$ is a positive real number and $h$ is a non-negative function with support in $\mathbb{R}_{+}$, such that $|| h ||_{L^1} <1$. If the distribution of $N_{-}$ satisfies \eqref{eq:assumption:P0} 
then 
$(\mathcal{A}_{\lambda,loc}^{\mathbb{L}^1,exp})$ is satisfied.
\end{lem}
\begin{proof}
For 
{all} $t> 0$, let
$
\overline{\lambda}(t) = \mathbb{E}\left[ \lambda(t,\mathcal{F}_{t-}^{N}) \right].
$
{By} Proposition~\ref{prop:decomp:Hawkes}, 
$ 
\overline{\lambda}(t) =  \mathbb{E}\left[ {\mu + \int_0^{t-} h(t-x) N_{>0}(\dd x)} \right] + \mathbb{E}\left[ {\int_{-\infty}^{t-} h(t-x) N_{\leq 0} (\dd x)} \right]
$ 
which is possibly infinite.

{Let us apply Proposition\ref{prop:implicit:formulas:L:G} with $g\equiv\mu$ and $s=0$, the choice $s=0$ implying that $\mathcal{E}_{t,0}(N_{>0})$ is of probability 1. Therefore
$$\mathbb{E}\left[ \mu + \int_0^{t-} h(t-x) N_{>0}(\dd x) \right]= \mu \left( 1+ \int_{0}^{t} (h(x)+L_{0}(x))  \dd x \right),$$  where ${(L_0,G_0=1)}$ is the solution of Lemma \ref{unicityLG} for $s=0$, {by identification of Proposition \ref{prop:implicit:formulas:L:G}}. Hence 
$\mathbb{E}\left[ \mu + \int_0^{t-} h(t-x) N_{>0}(\dd x) \right]\leq \mu (1 + || h ||_{L^1} + || L_{0} ||_{L^1} )$.}
%
%
On the other hand, 
thanks to Lemma~\ref{lem:computation:g}, we have
$$
\mathbb{E}\left[\int_{-\infty}^{t-} h(t-x) N_{\leq 0} (\dd x)\right] = \mathbb{E}\left[ \sum_{T\in N_{-}} \left(  h(t-T) + \int_{0}^{t} \left[ h(t-x) + L_{0}(t-x) \right] h(x-T) \dd x \right) \right].
$$
Since all the quantities are non negative, one can exchange all the integrals and deduce that
$$
\mathbb{E}\left[ \int_{-\infty}^{t-} h(t-x) N_{\leq 0} (\dd x) \right] 
\leq M(1 + || h ||_{L^1} + || L_{0} ||_{L^1}),
$$
with $M = \sup_{t\geq 0} \mathbb{E}\left[ \int_{-\infty}^{0} h(t-x) N_{-}(\dd x) \right]$ which is finite by assumption. Hence,
$
\overline{\lambda}(t) \leq (\mu + M) (1 + || h ||_{L^1} + || L_{0} ||_{L^1}),
$ 
{ and therefore $(\mathcal{A}_{\lambda,loc}^{\mathbb{L}^1,exp})$ is satisfied.}

\end{proof}
\paragraph{Proof of Proposition \ref{prop:conditional:intensity:hawkes}}
First, {by Proposition \ref{prop:decomp:Hawkes}
\begin{multline*}
\mathbb{E}\left[ \lambda(t,\mathcal{F}_{t-}^{N}) \middle|\,S_{t-}\geq s \right] = \\
\mu + \mathbb{E}\left[\int_0^{t-} h(t-z) N_{>0}(\dd z)\middle| \mathcal{E}_{t,s}(N)\right] + \mathbb{E}\left[\int_{-\infty}^{t-} h(t-z) N_{\leq 0}(\dd z)\middle| \mathcal{E}_{t,s}(N)\right] \\
= \mu + \mathbb{E}\left[\int_0^{t-} h(t-z) N_{>0}(\dd z)\middle| \mathcal{E}_{t,s}(N_{>0})\right] + \mathbb{E}\left[\int_{-\infty}^{t-} h(t-z) N_{\leq 0}(\dd z)\middle| \mathcal{E}_{t,s}(N_{\leq 0})\right]
\end{multline*}
By Lemma~\ref{prop:implicit:formulas:L:G}, we obtain
$\mathbb{E}\left[ \lambda(t,\mathcal{F}_{t-}^{N}) \middle|\,S_{t-}\geq s \right] = \mu+ L_s^{\mu,h}(t)+ \Phi_{-,\P_0}^h(t,s).$
Identifying by Lemma \ref{unicityLG}, $L_s=L_s^{h,h}$ and $G_s=G_s^{h,h}$, we obtain
$$\mathbb{E}\left[ \lambda(t,\mathcal{F}_{t-}^{N}) \middle|\,S_{t-}\geq s \right] = \Phi_+^{\mu,h}(t,s)+ \Phi_{-,\P_0}^h(t,s).$$
Hence $\Phi^{\mu,h}_{\mathbb{P}_{0}}(t,s) = \mathbb{E}\left[ \lambda(t,\mathcal{F}_{t-}^{N}) \middle|\,S_{t-}\geq s \right]$.}
 
 Lemma~\ref{lem:assumptions:Hawkes} ensures that the assumptions of Theorem~\ref{prop: Expectation System} are fulfilled.
Let {$u$ and $\rho^{\mu,h}_{\P_0}= \rho_{\lambda, \P_0}$ be defined accordingly as in Theorem~\ref{prop: Expectation System}. With respect to the PDE system, there are} two possibilities to express $\mathbb{E}\left[ \lambda(t,\mathcal{F}_{t-}^{N}) \mathds{1}_{\left\{ S_{t-}\geq s \right\}} \right]$. The first one involves $\rho_{\lambda,\mathbb{P}_{0}}$ and is
$ 
\mathbb{E}\left[\rho^{\mu,h}_{\mathbb{P}_{0}}(t,S_{t-}) \mathds{1}_{S_{t-}\geq s}\right],
$ 
whereas the second one involves $\Phi^{\mu,h}_{\mathbb{P}_{0}}$ and is
$
\Phi^{\mu,h}_{\mathbb{P}_{0}}(t,s) \mathbb{P}\left( S_{t-}\geq s \right).
$

{This leads to}
$ 
\int_s^{+\infty} \rho^{\mu,h}_{\mathbb{P}_{0}}(t,x)u(t,\dd x)=  \Phi^{\mu,h}_{\mathbb{P}_{0}}(t,s) \int_s^{+\infty}u(t,\dd x)
$, 
since $u(t,\dd s)$ is the distribution of $S_{t-}$.
Let us denote $v(t,s)=\int_s^{+\infty} u(t,\dd x)$: this relation, together with Equation~\eqref{eq:edp:expect} for $u,$ immediately gives us that $v$ satisfies Equation~\eqref{eq:PDEPhi:v} {with $\Phi=\Phi^{\mu,h}_{\mathbb{P}_0}$}. Moreover, $\int_0^{+\infty} u(t,\dd x)=1,$ which gives us the boundary condition {in} \eqref{eq:PDEPhi:bound:v}.

%

\vspace{-0.2cm}
\subsubsection{Study of {the general case for} $\Phi_{-,\P_0}^h$ in Proposition \ref{prop:conditional:intensity:hawkes}}

\begin{lem}\label{lem:computation:g}
Let consider $h$  a non-negative function with support in 
$\R_+$ such that $\int h <1$,  
$N_{-}$ a point process on $\mathbb{R}_{-}$
 with distribution $\mathbb{P}_0$ and 
$N_{\leq 0}$ defined by~\eqref{eq:def:Nleq0}.
If 
$
 \Phi^{h}_{-,\mathbb{P}_{0}}(t,s):=\mathbb{E}\left[ \int_{-\infty}^{t-} h(t-z) N_{\leq 0}(\dd z) \middle|\,
 {\mathcal{E}_{t,s}(N_{\leq 0})}\right],
$ 
for all $s,t \geq 0$, 
then,
\begin{equation}\label{eq:Psi:-:P0}
\Phi^{h}_{-,\mathbb{P}_{0}}(t,s) = \mathbb{E}\left[ \sum_{T\in N_{-}} \left(h(t-T)+ K_{s}(t,T)  \right) \middle| \mathcal{E}_{t,s}(N_{\leq 0}) \right],
\end{equation}
where 
$K_{s}(t,u)$ is given by \eqref{Kqdef}.
\end{lem}

\begin{proof}
Following the decomposition given in Proposition~\ref{prop:decomp:Hawkes}, one has
\begin{multline*}
\Phi^{h}_{-,\mathbb{P}_{0}}(t,s) = \mathbb{E}\left[ \sum_{T\in N_{-}} 
\Bigg( h(t-T)   \right.\\
 \left. + \sum_{V\in N_{1}^{T}} \Bigg( h(t-V) + \sum_{W\in N_{c}^{T,V}} h(t-V-W) \Bigg)\Bigg)  \middle| \mathcal{E}_{t,s}(N_{\leq 0}) \right],
\end{multline*}
where
$
\mathcal{E}_{t,s}(N_{\leq 0}) = \mathcal{E}_{t,s}(N_{-}) \bigcap_{ T'\in N_{-}} \left( \mathcal{E}_{t,s}(N_{1}^{T}) \bigcap_{ V'\in N_{1}^{T}} \mathcal{E}_{t-V',s}( N_{c}^{V'} )\right).
$
Let us fix $T\in N_{-}$, $V\in N_{1}^{T}$ and compute the conditional expectation of the inner sum with respect to $N_{-}$ and $N_{1}^{T}$. In the same way as for~\eqref{eq:example:application:lemma} we end up with
\begin{equation*}
 \mathbb{E}\left[ \sum_{W\in N_{c}^{T,V}} h(t-V-W) \middle| N_{-}, N_{1}^{T}, \mathcal{E}_{t-V,s}(N_{c}^{T,V}) \right] = L^{h,h}_{s}(t-V),
\end{equation*}
since, conditionally on $N_{-}$ and $N_{1}^{T}$,  $N_{c}^{T,V}$ has the same distribution as $N_{c}$.
Using the conditional independence of the cluster processes $(N_{c}^{T,V})_{V\in N_{1}^{T}}$ with respect to $\left( N_{-}, (N_{1}^{T})_{T\in N_{-}} \right)$, one can apply Lemma \ref{lem:conditional:expectation} with $Y=\left( N_{-}, (N_{1}^{T})_{T\in N_{-}} \right)$ and $X_{(T,V)}=N_{c}^{T,V}$ and deduce that
$$
\Phi^{h}_{-,\mathbb{P}_{0}}(t,s)  = \mathbb{E}\left[\sum_{T\in N_{-}}  \left( h(t-T) + \sum_{V\in N_{1}^{T}} \left( h(t-V) + L^{h,h}_{s}(t-V) \right) \right) \middle| \mathcal{E}_{t,s}(N_{\leq 0})  \right].
$$
Let us fix $T\in N_{-}$ and compute the conditional expectation of the inner sum with respect to $N_{-}$ which is
\begin{equation}
\Gamma:= \mathbb{E}\left[ \sum_{V\in N_{1}^{T}} \left( h(t-V) + L^{h,h}_{s}(t-V) \right) \middle| N_{-}, A_{t,s}^{T} \right], \label{eq:expectation:K}
\end{equation}
where $A_{t,s}^{T}=\mathcal{E}_{t,s}(N_{1}^{T})\cap \left( \bigcap_{V' \in N_{1}^{T}}\mathcal{E}_{t-V',s}(N_{c}^{T,V'}) \right)$.
For every $V\in N_{1}^{T}$, we say that $V$ has mark $0$ if $V$ has no descendant or himself in $[t-s,t)$ and mark $1$ otherwise. Let us denote $N_{1}^{T,0}$ the set of points with mark $0$ and $N_{1}^{T,1}=N_{1}^{T}\setminus N_{1}^{T,0}$. 

For any $V\in N_{1}^{T}$, 
$
\mathbb{P}\left( V \in N_{1}^{T,0} \middle| N_{1}^{T} \right)= G^{h,h}_{s}(t-V) \mathds{1}_{[t-s,t)^{c}}(V)
$
and all the marks are chosen independently given $N_{1}^{T}$.
Hence, $N_{1}^{T,0}$ and $N_{1}^{T,1}$ are independent Poisson processes and the intensity of $N_{1}^{T,0}$ is given by $\lambda(v)=h(v-T)\mathds{1}_{[0,\infty)}(v)G^{h,h}_{s}(t-v) \mathds{1}_{[t-s,t)^{c}}(v)$. Moreover, $A_{t,s}^{T}$ {is the event} $\left\{ N_{1}^{T,1}=\emptyset \right\}$, 
 so
\begin{eqnarray*}
\Gamma & = &  \mathbb{E}\left[ \sum_{V\in N_{1}^{T,0}} \left( h(t-V) +L^{h,h}_{s}(t-V) \right) \middle| N_{-}, \left\{ N_{1}^{T,1}=\emptyset \right\} \right] \\
& = & \int_{-\infty}^{t-} \left[ h(t-v) +L^{h,h}_{s}(t-v) \right] h(v-T)\mathds{1}_{[0,\infty)}(v)G^{h,h}_{s}(t-v) \mathds{1}_{[t-s,t)^{c}}(v) dv\\
& = & K_{s}(t,T).
\end{eqnarray*}
Using the independence of the cluster processes, 
one can apply Lemma \ref{lem:conditional:expectation} with $Y=N_{-}$ and $X_{T}=\left( N_{1}^{T}, (N_{c}^{T,V})_{V\in N_{1}^{T}} \right)$ and~\eqref{eq:Psi:-:P0} clearly follows.
\end{proof}

\begin{lem}
\label{prop:conditional:intensity:hawkes:simple:examples}
Under the assumptions and notations of Proposition \ref{prop:conditional:intensity:hawkes}  and Lemma \ref{unicityLG}, 
the function $\Phi^{h}_{-,\mathbb{P}_{0}}$ of Proposition \ref{prop:conditional:intensity:hawkes} can be identified with \eqref{eq:Phi:-:1pt} under $(\mathcal{A}^1_{N_-})$ and with \eqref{eq:Phi:-:Poisson} under $(\mathcal{A}^2_{N_-})$ 
and \eqref{eq:assumption:P0} is satisfied in those two cases.
\end{lem}

\paragraph{Proof.}
Using 
Lemma~\ref{lem:computation:g}, we have
$ 
\Phi^{h}_{-,\mathbb{P}_{0}}(t,s) = \mathbb{E}\left[ \sum_{T\in N_{-}} \left(h(t-T)+ K_{s}(t,T)  \right) \middle| \mathcal{E}_{t,s}(N_{\leq 0}) \right].
$ 
\noindent
{{\it Under $\left(\mathcal{A}^1_{N_-}\right)$}}. On the one hand, for every $t\geq 0$,
\begin{eqnarray*}
\mathbb{E}\left[ \int_{-\infty}^{0} h(t-x) N_{-}(\dd x) \right] &=& \mathbb{E}\left[ h(t-T_{0}) \right] \\
& = &\int_{-\infty}^{0} h(t-t_{0}) f_{0}(t_{0}) \dd t_{0} \leq || f_{0} ||_{L^\infty} \int_{0}^{\infty} h(y) dy,
\end{eqnarray*}
hence $\mathbb{P}_{0}$ satisfies~\eqref{eq:assumption:P0}. On the other hand, since $N_{-}$ is reduced to one point $T_{0}$,
$ 
\Phi^{h}_{-,\mathbb{P}_{0}}(t,s) = \frac{1}{\mathbb{P}\left( \mathcal{E}_{t,s}(N_{\leq 0}) \right)} \mathbb{E}\left[  \left(h(t-T_{0})+ K_{s}(t,T_{0})  \right) \mathds{1}_{\mathcal{E}_{t,s}(N_{\leq 0})}  \right],
$ 
using the definition of the conditional expectation. 
First, we {compute $\P(\mathcal{E}_{t,s}(N_{\leq 0}|T_0)$. To do so, we }
use the decomposition $\mathcal{E}_{t,s}(N_{\leq 0})=\{T_{0}<t-s\} \cap \mathcal{E}_{t,s}(N_{1}^{T_{0}}) \cap \left( \bigcap_{V\in N_{1}^{T_{0}}} \mathcal{E}_{t-V,s}(N_{c}^{T_{0},V})\right)$ and the fact that, conditionally on $N_{1}^{T_{0}}$, for all $V\in N_{1}^{T_{0}}$, $N_{c}^{T_{0},V}$ has the same distribution as $N_{c}$ to deduce that
\begin{equation*}
\mathbb{E}\left[ \mathds{1}_{\mathcal{E}_{t,s}(N_{\leq 0})} \middle|\, T_{0} \right] = \mathds{1}_{T_{0}<t-s} \mathbb{E}\left[ \mathds{1}_{\mathcal{E}_{t,s}(N_{1}^{T_{0}})} \middle| T_{0} \right]  \mathbb{E}\left[ \prod_{V\in N_{1}^{T_{0}}\cap[t-s,t)^{c}} G_{s}(t-V) \middle| T_{0} \right],
\end{equation*}
because the event $\mathcal{E}_{t,s}(N_{1}^{T_{0}})$ involves $N_{1}^{T_{0}}\cap[t-s,t)$  whereas the product involves $N_{1}^{T_{0}}\cap[t-s,t)^{c}$, {both of those processes being} two independent Poisson processes. Their respective intensities are $\lambda(x)=h(x-T_{0})  \mathds{1}_{[(t-s)\vee 0,t)}(x)$ and $\lambda(x)=h(x-T_{0})  \mathds{1}_{[0,(t-s)\vee 0)}(x)$, so we end up with 
$$
\begin{cases}
\mathbb{E}\left[ \mathds{1}_{\mathcal{E}_{t,s}(N_{1}^{T_{0}})} \middle| T_{0} \right] = \exp\left(-\int_{t-s}^{t}h(x-T_{0}) \mathds{1}_{[0,\infty)}(x) \dd x \right)\\
\mathbb{E}\left[ \prod\limits_{V\in N_{1}^{T_{0}}\cap[t-s,t)^{c}} G_{s}(t-V) \middle| T_{0} \right]= \exp\left(- \int_{0}^{(t-s)\vee 0}\left[ 1-G_{s}(t-x) \right] h(x-T_{0}) \dd x \right).
\end{cases}
$$
The product of these two last quantities is exactly $q(t,s,T_{0})$ {given by \eqref{Kqdef}}. 
Note that $q(t,s,T_{0})$ is exactly the probability {that} $T_{0}$ has no descendant in $[t-s,t)$ given $T_{0}$.
Hence,
$
\mathbb{P}\left( \mathcal{E}_{t,s}(N_{\leq 0}) \right) = \int_{-\infty}^{0\wedge(t-s)} q(t,s,t_{0}) f_{0}(t_{0}) \dd t_{0}
$
and~\eqref{eq:Phi:-:1pt} clearly follows.

\noindent
{{\it Under $\left(\mathcal{A}^2_{N_-}\right)$}}. On the one hand, for any $t\geq 0$,
$$
\mathbb{E}\left[ \int_{-\infty}^{0} h(t-x) N_{-}(\dd x) \right] = \mathbb{E}\left[ \int_{-\infty}^{0} h(t-x) \alpha \dd x \right] \leq \alpha \int_{0}^{\infty} h(y) dy,
$$
hence $\mathbb{P}_{0}$ satisfies~\eqref{eq:assumption:P0}. On the other hand, since we are dealing with a Poisson process, we can use 
{the same argumentation of marked Poisson processes as in the proof of Lemma~\ref{prop:implicit:formulas:L:G}.}
For every $T\in N_{-}$, we say that $T$ has mark $0$ if $T$ has no descendant or himself in $[t-s,t)$ and mark $1$ otherwise. Let us denote $N_{-}^{0}$ the set of points with mark $0$ and $N_{-}^{1}=N_{-}\setminus N_{-}^{0}$. For any $T\in N_{-}$, we have
$$
\mathbb{P}\left( T \in N_{-}^{0} \middle| N_{-} \right)= q(t,s,T) \mathds{1}_{[t-s,t)^{c}}(T),
$$
and all the marks are chosen independently given $N_{-}$.
Hence, $N_{-}^{0}$ and $N_{-}^{1}$ are independent Poisson processes and the intensity of $N_{-}^{0}$ is given by 
$$\lambda(z)= \alpha\mathds{1}_{z\leq 0} \, q(t,s,z)  \mathds{1}_{[t-s,t)^{c}}(z)$$ 
Moreover, $\mathcal{E}_{t,s}(N_{\leq 0}) = \left\{ N_{-}^{1}=\emptyset \right\}$. Hence, 
\begin{equation*}
\Phi^{h}_{-,\mathbb{P}_{0}}(t,s) = \mathbb{E}\left[ \sum_{T\in N_{-}^{0}} \left(h(t-T)+ K_{s}(t,T) \right) \middle| N_{-}^{1}=\emptyset \right]
\end{equation*}
which gives~\eqref{eq:Phi:-:Poisson} thanks to the independence of $N_{-}^{0}$ and $N_{-}^{1}$.

\vspace{-0.2cm}
\subsubsection{Proof of Propositions~\ref{prop:conditional:intensity:hawkes1} and~\ref{prop:limit:edp:support:infinity}}

Since we already proved Proposition~\ref{prop:conditional:intensity:hawkes} and Lemma~\ref{prop:conditional:intensity:hawkes:simple:examples}, to obtain Proposition~\ref{prop:conditional:intensity:hawkes1} it only remains to prove that $\Phi_{\P_0}^{\mu,h} \in L^\infty(\R_+^2)$, to ensure uniqueness of the solution by Remark \ref{unicityv}. To do so, it is easy to see that the assumption $h\in L^\infty(\R_+)$ combined with Lemma \ref{unicityLG} giving that $G_s\in[0,1]$ and $L_s\in L^1(\R_+)$ ensures that $\Phi_+^{\mu,h}, q$ and $K_s$  are in $L^\infty(\R_+)$. In  turn, this implies that $\Phi_{-,\mathbb{P}_0}^h$ in both  \eqref{eq:Phi:-:1pt} and  \eqref{eq:Phi:-:Poisson} is in $L^\infty(\R_+)$, which concludes the proof of Proposition \ref{prop:conditional:intensity:hawkes1}. 

\paragraph{Proof of Proposition ~\ref{prop:limit:edp:support:infinity}}
The method of characteristics leads us to rewrite the solution $v$
of \eqref{eq:PDEPhi:v}--\eqref{eq:PDEPhi:bound:v}
by defining $f^{in}\equiv v^{in}$ on $\R_+$, $f^{in}\equiv 1$ on $\R_-$ such that
\begin{equation}\label{eq:explicit:v}
v(t,s)=
\begin{cases}
f^{in}(s-t)e^{-\int_{(t-s)\vee 0}^t \Phi(y,s-t+y)\, dy}, {\mbox{ when } s\geq t} \\
f^{in}(s-t)e^{-\int_{(s-t)\vee 0}^s \Phi(y+t-s,y)\, dy}, {\mbox{ when } t\geq s}.
\end{cases}
\end{equation}
{Let $\mathbb{P}_0^M$ be the distribution of the past given by  $\left(\mathcal{A}^1_{N_-}\right)$ and $T_0\sim \mathcal{U}([-M-1,-M])$. By Proposition~\ref{prop:conditional:intensity:hawkes1},}
let 
$v_M$ be the solution of System~\eqref{eq:PDEPhi:v}--\eqref{eq:PDEPhi:bound:v} with $\Phi=\Phi_{\P_0^M}^{\mu,h}$ and $v^{in}=v_M^{in},$ (i.e. the survival function of a uniform variable on $[-M-1,-M]$). Let 
also $v_M^\infty$ be the solution {of System \eqref{eq:PDEPhi:v}--\eqref{eq:PDEPhi:bound:v}} with $\Phi=\Phi_{\P_0^M}^{\mu,h}$ and $v^{in}\equiv 1,$ and $v_\infty$ the solution {of \eqref{eq:edp:mass:creation}-\eqref{eq:edp:mass:creation:bound}. Then
$$\Vert v_M - v^\infty \Vert_{L^\infty((0,T)\times (0,S))}\leq \Vert v_M - v_M^\infty \Vert_{L^\infty((0,T)\times (0,S))}+\Vert v_M^\infty - v^\infty \Vert_{L^\infty((0,T)\times (0,S))}.$$
By definition of $v_M^{in},$ it is clear that $v_M^{in}(s)=1$ for $s\leq M,$ so that
 Formula~\eqref{eq:explicit:v} implies that $ v_M{(t,s)}= v_M^\infty {(t,s)}$ as soon as $s-t\leq M$ 
and so $\Vert v_M - v_M^\infty \Vert_{L^\infty((0,T)\times (0,S))}=0$ as soon as $M\geq S$.

To evaluate the distance $\Vert v_M^\infty - v^\infty \Vert_{L^\infty((0,T)\times (0,S))}$, it remains to prove that $e^{-\int_0^t {\Phi_{-,\P_0^M}^{h}}(y,s-t+y)\, dy} \to 1$ uniformly on  $(0,T)\times (0,S)$ for any $T>0,\,S>0.$ For this, it suffices to prove that ${\Phi_{-,\P_0^M}^{h}}(t,s) \to 0$ 
uniformly on  $(0,T)\times (0,S)$.
Since $q$  given by \eqref{Kqdef} takes  values 
in $[\exp(-2 || h ||_{L^1}),1]$, \eqref{eq:Phi:-:1pt} implies
$$
{\Phi_{-,\P_0^M}^{h}}(t,s) \leq 
\frac{\int_{-\infty}^{0\wedge(t-s)} \left( h(t-t_{0}) + K_{s}(t,t_{0}) \right) \mathds{1}_{[-M-1,-M]}(t_{0})  \dd t_{0}}
{\int_{-\infty}^{0\wedge(t-s)} \exp(-2 ||h||_{L^1}) \mathds{1}_{[-M-1,-M]}(t_{0}) \dd t_{0}}.
$$
Since $||G_{s}||_{L^{\infty}}\leq 1$, $L_{s}$ and $h$ are non-negative, it is clear that
$$
K_{s}(t,t_{0})\leq  \int_{0}^{+\infty} \left[ h(t-x) +L_{s}(t-x) \right] h(x-t_{0}) \dd x,
$$
and so
\begin{eqnarray*}
\int_{-M-1}^{-M} K_{s}(t,t_{0}) \dd t_{0} & \leq & \int_{0}^{+\infty} \left[ h(t-x) +L_{s}(t-x) \right] \left( \int_{-M-1}^{-M} h(x-t_{0}) \dd t_{0} \right) \dd x\\
& \leq & \int_{M}^{\infty} h(y)dy \int_{0}^{+\infty} \left[ h(t-x) +L_{s}(t-x) \right] \dd x\\
& \leq & \int_{M}^{\infty} h(y)dy \left[ || h ||_{L^{1}} + || L_{s} ||_{L^1} \right].
\end{eqnarray*}
Hence, for $M$ large enough
$ 
{\Phi_{-,\P_0^M}^{h}}(t,s) \leq \frac{ \int_{M}^{\infty}h(y)dy\left[ || h ||_{L^1} + || L_{s} ||_{L^1} \right] }{\exp(-2 ||h||_{L^1})} \rightarrow 0,
$ 
uniformly in $(t,s)$ since $L_s$ is uniformly bounded in $L^{1}$, which concludes the proof. 

\vspace{-0.2cm}
\subsection{Thinning \label{sec:Thinning}}
The demonstration of Ogata's thinning algorithm uses a generalization of point processes, namely the marked point processes. However, only the basic properties of simple and marked point processes are needed (see \cite{Bre} for a good overview of point processes theory).
Here $(\mathcal{F}_t)_{t>0}$ denotes a general filtration such that $\mathcal{F}_t^N\subset \mathcal{F}_t$ for all $t>0$, and not necessarily the natural one, i.e. $(\mathcal{F}_t^N)_{t>0}$.}
\begin{thm}
\label{thm:Thinning:R2}Let $\Pi$ be a $\left(\mathcal{F}_{t}\right)$-Poisson
process with intensity $1$ on $\mathbb{R}_{+}^{2}$. Let $\lambda(t,\mathcal{F}_{t-})$
be a non-negative $\left(\mathcal{F}_{t}\right)$-predictable process
which is $L_{loc}^{1}$ a.s. and define the point process
$N$ by
$ 
N\left(C\right)=\int_{C\times\mathbb{R}_{+}}\mathbf{1}_{\left[0,\lambda(t,\mathcal{F}_{t-})\right]}\left(z\right)\,\Pi\left(\dd t\times \dd z\right),
$ 
for all $C\in\mathcal{B}\left(\mathbb{R}_{+}\right)$. Then $N$ admits
$\lambda(t,\mathcal{F}_{t-})$ as a $\left(\mathcal{F}_{t}\right)$-predictable intensity.
Moreover, if $\lambda$ is in fact $\left( \mathcal{F}_{t}^{N} \right)$-predictable, i.e. $\lambda(t,\mathcal{F}_{t-}) = \lambda(t,\mathcal{F}_{t-}^{N})$, then $N$ admits
$\lambda(t,\mathcal{F}_{t-}^{N})$ as a $\left(\mathcal{F}_{t}^{N}\right)$-predictable intensity.
\end{thm}
\begin{proof}
The goal is to apply the martingale characterization of the intensity (Chapter II, Theorem 9 in \cite{Bre}). We cannot consider $\Pi$ as a point process on $\mathbb{R}_{+}$ marked in $\mathbb{R}_{+}$ (in particular, the point with the smallest abscissa cannot be defined).
However, for every $k\in\mathbb{N}$, we can define $\Pi^{\left(k\right)}$, the restriction of $\Pi$ to the points with ordinate smaller than $k$, by
$ 
\Pi^{\left(k\right)}\left(C\right)=\int_{C}\,\Pi\left(\dd t\times \dd z\right)
$ 
for all $C\in\mathcal{B}\left(\mathbb{R}_{+}\times\left[0,k\right]\right)$.
Then $\Pi^{\left(k\right)}$ can be seen as a point process on $\mathbb{R}_{+}$
marked in $E_{k}:=\left[0,k\right]$ with intensity kernel $1.\dd z$ with respect
to $\left(\mathcal{F}_{t}\right)$.
In the same way, we define $N^{\left(k\right)}$ by
$$
N^{\left(k\right)}\left(C\right)=\int_{C\times\mathbb{R}_{+}}\mathbf{1}_{z\in\left[0,\lambda(t,\mathcal{F}_{t-})\right]}\,\Pi^{\left(k\right)}\left(\dd t\times \dd z\right)
\quad \mbox{for all} \quad C\in\mathcal{B}\left(\mathbb{R}_{+}\right).
$$
Let $\mathcal{P}(\mathcal{F}_t)$ be the predictable $\sigma$-algebra (see page 8 of \cite{Bre}).

Let us denote $\mathcal{E}_{k}=\mathcal{B}\left(\left[0,k\right]\right)$
and $\tilde{\mathcal{P}}_{k}\left(\mathcal{F}_{t}\right)=\mathcal{P}\left(\mathcal{F}_{t}\right)\otimes\mathcal{E}_{k}$
the associated marked predictable $\sigma$-algebra. 

For any fixed $z$ in $E$, 
$\left\{ \left(u,\omega\right)\in\mathbb{R}_{+}\times\Omega\ \mbox{ such that }\ \lambda(u,\mathcal{F}_{u-})\left(\omega\right)\geq z\right\} \,\in\mathcal{P}\left(\mathcal{F}_{t}\right)$
since $\lambda$
 is predictable. 
If $\Gamma_{k}=\left\{ \left(u,\omega,z\right)\in\mathbb{R}_{+}\times\Omega\times E_{k},\ \lambda(u,\mathcal{F}_{u-})\left(\omega\right)\geq z\right\} $,
then
$$
\Gamma_{k}=\underset{n\in\mathbb{N}^{*}}{\cap}\underset{q\in\mathbb{Q}_{+}}{\cup}\left\{ \left(u,\omega\right)\in\mathbb{R}_{+}\times\Omega,\  \lambda(u,\mathcal{F}_{u-})\left(\omega\right)\geq q\right\} \times\left(\left[0,q+\frac{1}{n}\right]\cap E_{k}\right).
$$
So, $\Gamma_{k}\in\tilde{\mathcal{P}}_{k}\left(\mathcal{F}_{t}\right)$ and $\mathbf{1}_{z\in\left[0,\lambda(u,\mathcal{F}_{u-})\right]\cap E_{k}}$ is $\tilde{\mathcal{P}_{k}}\left(\mathcal{F}_{t}\right)$-measurable.
Hence, one can apply the Integration Theorem (Chapter VIII, Corollary 4 in \cite{Bre}).
So,
$$
(X_{t})_{t\geq 0}:=\left(\int_{0}^{t}\int_{E_{k}}\mathbf{1}_{z\in\left[0,\lambda(u,\mathcal{F}_{u-})\right]}\,\bar{M}^{\left(k\right)}\left(du\times \dd z\right)\right)_{t\geq0}\mbox{ is a \ensuremath{\left(\mathcal{F}_{t}\right)}-local martingale}
$$
where $\bar{M}^{\left(k\right)}\left(du\times \dd z\right)=\Pi^{\left(k\right)}\left(du\times \dd z\right)-\dd z\dd u$.
In fact,
\begin{equation*}
X_{t}=N_{t}^{\left(k\right)}-\int_{0}^{t}\min\left(\lambda(u,\mathcal{F}_{u-}),k\right)\dd u.
\end{equation*}

Yet, $N_{t}^{\left(k\right)}$ (respectively $\int_{0}^{t}\min\left(\lambda(u,\mathcal{F}_{u-}),k\right)\dd u$)
is non-decreasingly converging towards $N_{t}$ (resp. $\int_{0}^{t}\lambda(u,\mathcal{F}_{u-})\dd u$). Both of the limits are finite a.s. thanks to the local integrability of the intensity (see page 27 of \cite{Bre}).
Thanks to monotone convergence we deduce that $\left(N_{t}-\int_{0}^{t}\lambda(u,\mathcal{F}_{u-})\dd u\right)_{t\geq0}$
is a $\left(\mathcal{F}_{t}\right)$-local martingale. Then, thanks
to the martingale characterization of the intensity, $N_{t}$ admits $\lambda(t,\mathcal{F}_{t-})$ as an $\left(\mathcal{F}_{t}\right)$-intensity.
The last point of the Theorem is a reduction of the filtration. Since $\lambda(t,\mathcal{F}_{t-}) = \lambda(t,\mathcal{F}_{t-}^{N})$, it is a fortiori $\left( \mathcal{F}_{t}^{N} \right)$-progressive and the desired result 
{follows} (see page $27$ in \cite{Bre}).
\end{proof}

This final result can be found in \cite{Bremaud_Massoulie_1996}.

\begin{prop}[Inversion Theorem] \label{prop:inversion:theorem}

Let $N=\left\{ T_{n}\right\} _{n>0}$ be a non explosive point process on $\mathbb{R}_{+}$ with $\left(\mathcal{F}_{t}^{N}\right)$-predictable intensity $\lambda_{t}=\lambda(t,\mathcal{F}_{t-}^{N})$. Let $\left\{ U_{n}\right\} _{n>0}$ be a sequence of i.i.d. random variables with uniform distribution on $\left[0,1\right]$. Moreover, suppose that they are independent of $\mathcal{F}^N_{\infty}$. Denote $\mathcal{G}_{t}=\sigma\left(U_{n}, T_{n}\leq t\right)$.
Let $\hat{N}$ be an homogeneous Poisson process with intensity $1$ on $\mathbb{R}_{+}^{2}$ independent of $\mathcal{F}_{\infty}\vee\mathcal{G}_{\infty}$.
Define a point process $\bar{N}$ on $\mathbb{R}_{+}^{2}$ by
\[
\bar{N}\left((a,b]\times A\right)=\sum_{n>0}\mathds{1}_{(a,b]}\left(T_{n}\right)\mathds{1}_{A}\left(U_n\lambda(T_{n},\mathcal{F}_{T_n-}^N) \right)+\int_{(a,b]}\int_{A-[0,\lambda(t,\mathcal{F}_{t-}^N)]}\hat{N}\left(\dd t\times \dd z\right)
\]
for every $0\leq a<b$ and $A\subset\mathbb{R}_{+}$.

Then, $\bar{N}$ is an homogeneous Poisson process on $\mathbb{R}_{+}^{2}$ with intensity $1$ with respect to the filtration $\left(\mathcal{H}_{t}\right)_{t\geq 0}=\left(\mathcal{F}_{t}\vee\mathcal{G}_{t}\vee\mathcal{F}_{t}^{\hat{N}}\right)_{t\geq 0}$.
\end{prop}

\bibliographystyle{plain}       
\bibliography{Tout_Marie}           

\end{document}